\documentclass[a4,12pt]{article}
\usepackage{graphicx}
\usepackage{geometry}
\usepackage{authblk}
\usepackage{amsfonts}
\usepackage{xcolor}
\usepackage{comment}
\usepackage{subcaption}
\usepackage{booktabs}
\usepackage{multirow}
\usepackage{appendix}
\usepackage{amsmath}
\usepackage{amsthm}

\title{Regime-Aware Conditional Neural Processes with Multi-Criteria Decision Support for Operational Electricity Price Forecasting}
\author[1]{Abhinav Das}
\author[2]{Stephan Schlüter}
\affil[1]{\textit{Faculty of Mathematics and Economics, Ulm University, Helmholtzstrasse 20, Ulm 89081, Germany}}
\affil[2]{\textit{Department of Mathematics, Natural and Economic Sciences, Ulm University of Applied Sciences, Prittwitzstrasse 10, Ulm, 89075, Germany}}
\date{}
\usepackage{hyperref}
\begin{document}
\newtheorem{theorem}{Theorem}
\newtheorem{remark}{Remark}

\maketitle

\begin{abstract}
\hfill\newline
\hrule
\hfill\newline\newline
This work integrates Bayesian regime detection with conditional neural processes for 24-hour electricity price prediction in the German market. Our methodology integrates regime detection using a disentangled sticky hierarchical Dirichlet process hidden Markov model (DS-HDP-HMM) applied to daily electricity prices. Each identified regime is subsequently modeled by an independent conditional neural process (CNP), trained to learn localized mappings from input contexts to 24-dimensional hourly price trajectories, with final predictions computed as regime-weighted mixtures of these CNP outputs. We rigorously evaluate R-NP against deep neural networks (DNN) and Lasso estimated auto-regressive (LEAR) models by integrating their forecasts into diverse battery storage optimization frameworks, including price arbitrage, risk management, grid services, and cost minimization. This operational utility assessment revealed complex performance trade-offs: LEAR often yielded superior absolute profits or lower costs, while DNN showed exceptional optimality in specific cost-minimization contexts. Recognizing that raw prediction accuracy doesn't always translate to optimal operational outcomes, we employed TOPSIS as a comprehensive multi-criteria evaluation layer. Our TOPSIS analysis identified LEAR as the top-ranked model for 2021, but crucially, our proposed R-NP model emerged as the most balanced and preferred solution for 2021, 2022 and 2023.
\hfill\newline
\hrule
\hfill\newline\newline
\hfill\newline\newline
\hfill\newline
\thanks{*Preprint Submitted to Energy Economics}
\end{abstract}

\section{Introduction}\label{sec:intro}
The energy market has faced a significant structural change in the past decade. The global strife for decarbonization is encouraging the use of renewable energy sources, thus affecting the traditional supply-demand pattern, which were historically dominated by fossil fuels like coal, oil, and natural gas \cite{FLORES2024123464}. The growing integration of renewable energy sources into the power supply increases uncertainties in the electricity market due to intermittent nature of the sources such as wind or sunshine  \cite{WESTPHAL2024114415}. The volatility of the generation sources causes high price shocks and regime changes that is compromising to financial stability as well as investment strategies in the power market \cite{xu2024resilience}. Particularly for countries such as Germany, where the larger percentage of electricity is produced by renewable energy sources \cite{LOIZIDIS2024123058}, levels of sunlight and wind impact electricity generation and thus prices. This introduces, in addition to the physical problem of balancing the grid, non-stationarity to most price models, which further adds unreliability to the predictions. Accurate electricity price forecasting is crucial for efficient resource planning, financial risk management, and stabilization of the market, especially with increasing renewable energy penetration, which enables utilities, businesses, and governments to optimize planning and policy maximization while matching demand and supply. The building of an adequate prediction model, which is relatively straightforward and understandable but at the same time can reflect the market complexity and all influence factors engaged in it is not straightforward, and authors have utilized quite broadly three types of model for prediction: statistical/(probability-based) models \cite{CORNELL20241421}, machine learning/deep learning models \cite{appliedmath3020018}, and mixed models \cite{JIANG2023106471}. Precise forecasting allows the players in the market to make sound monetary policy.

The regime changes and the non-stationarity and non-linearity of electricity prices, due to the impact of exogenous variables like economic activity, weather patterns, and regulatory policy, demand robust forecast techniques capable of capturing the subtle relationship between these variables \cite{su16072691}. Moreover, this model will also enhance the accuracy of forecasts and enable risk-based financial decisions when market conditions are changing. Accurate price forecasts also help stakeholder to make comparatively better decisions regarding dispatch schedules, for example. Policymakers also gain from delivered accurate predictions,a s this facilitates regulatory structure design towards more general application of clean energy sources\cite{alshater2025early}. 

Besides direct price prediction, the real value of accurate predictions lies in their ability to inform and optimize advanced operation plans, particularly for flexible resources like Battery Energy Storage Systems (BESS) \cite{KOLLER2015128, en16176344}. BESS applications are at the heart of modern grids, offering a range of services from energy arbitrage (selling high, buying low) and peak shaving to providing critical ancillary services and soaking up intermittent renewables \cite{YANG2022112671}. The profitable participation of BESS in wholesale electricity markets requires sophisticated operating strategies that transform price predictions into best-selling and buying schedules \cite{MERCIER2023106721}. They might range from basic profit maximization through energy arbitrage in day-ahead and intra-day markets to risk-averse operations for limiting exposure to price spikes, and even more to providing essential grid services such as frequency regulation or renewable energy firming \cite{10761266}. For some actors, the objective might even be cost-minimizing for a given load, with BESS maximizing grid usage over self-discharge \cite{7977466}. All these different operation goals impose unique demands on the price forecasts underlying them as well as the consequent optimization scheme, since the robustness of decisions to prediction mistakes may significantly differ based on alternative strategic objectives. Above all, the relationship between the basic predictive accuracy of a prediction model and its capacity to guide actually optimal operating decisions is not necessarily straight or proportional \cite{AHMED20199}.

A slightly worse raw forecast errors model can actually have better operational performance if its errors occur during times of low market impact, or at correctly identifying key price turning points (e.g., highs and lows) most beneficial for profitable dispatch. A very accurate forecast that fails to identify these key decision points, however, can still result in poor operational judgments. This intrinsic detachment, wherein prediction capability and real-world applicability branch apart, creates a multi-dimensional problem of evaluation where no single measure is in a position to best describe a model's real-world utility in a BESS implementation. An extensive evaluation therefore demands an systematic process capable of converting several dimensions of performance into valuable insight. This calls for the application of Multi-Criteria Decision-Making (MCDM) techniques, which are designed to rank alternatives in terms of several, often competing, criteria \cite{Ching_Kwangsun, Brans2016}.

For the evaluation of forecasting models in the complex setting of BESS operating policies, MCDM provides a solid framework to consolidate not only prediction performance (e.g., various accuracy measures for price forecasts) but also heterogeneous indicators of operating performance (such as realized profit, cost minimized, or optimized regret measured). Among various MCDM techniques, the Technique for Order Preference by Similarity to Ideal Solution (TOPSIS) is best suited for this research \cite{CHEN20001}. TOPSIS offers a straightforward distance-based method, ordering alternatives in terms of proximity to an ideal best solution and remoteness from an ideal worst solution. The approach can offer a clear, transparent, and understandable composite score able to accommodate the diversified quantitative measures of performance that result from both the estimation accuracy assessment and the comprehensive BESS optimization runs across different operational settings and appraisal horizons. Through the use of TOPSIS, our research extends beyond the common univariate analyses to provide a superior, more holistic, and actionable forecast of forecasting model performance in actual energy storage market participation, responding to the question of which model actually performs best when all the elements of market participation are accounted for. The conventional time series models fail to capture the non-stationary regime changes, such as demand changes, supply changes, or regulatory regimes changes. These may lead to eras of unique market regimes where the price formation process dynamics, volatility, and demand shapes significantly change.

In order to learn this regimes-based behavior, the regime-switching models are used that allow us to identify and model the different phases of market behavior.

However, traditional regime-switching models that deduce from a pre-specified number of regimes or Markovian state dependencies might not be suitable for the electricity market, in which regime switching might be unexpected and depends on outside influences. To bridge this gap, we apply the disentangled sticky hierarchical Dirichlet process hidden Markov model (DS-HDP-HMM). It is a Bayesian nonparametric model that can learn automatically an unseen and dynamically evolving number of regimes with persistence over discovered states \cite{zhou2020disentangled}. This is significant, as energy markets do not adhere to a specific set of behaviors and new market conditions can emerge based on changes in technology, regulation, or geopolitics. The disentanglement option ensures that many latent factors that influence price regimes remain separable so that an orderly and interpretable regime-switching is feasible. Having determined the regimes as latent, forecasting must now be adapted to each individual statistical nature of each regime.

Conventional forecasting methodologies are not applicable under such scenarios due to their requirements of strict stationarity and model structure. While GARCH models are capable of representing time-varying volatility and heteroskedasticity, they are still not able to explain regime switching, where the underlying market behavior can change radically due to some external factors such as changes in regulation or technology. \cite{CHAN2016182}. In order to circumvent this, we utilize conditional neural process (CNP), which is meta-learning-based probabilistic model, for every regimes identified, that offers regime-specific, distributional predictions with inherent uncertainty quantification. As opposed to standard neural networks, the NP model learns a generalizable stochastic mapping from historical market conditions to future prices over market regimes. This not only can capture the subtleties of functional relationships within electricity price data but, more critically, also allow for principled uncertainty estimation, as is essential for risk-sensitive decision-making in financial and energy markets. With the incorporation of S-HDP-HMM for regime adjustment and CNPs for regime-dependent probabilistic prediction, our system offers a scalable and high-quality solution to the problem of uncertainty in electricity prices in the scenario of energy transition and climate risk. The results in this paper demonstrate that the suggested training structure of the conditional neural process for each regime not only achieves better performance than the baseline models, i.e., deep neural network model and lasso estimated autoregressive model presented by \cite{LAGO2021116983} but also gives a balanced model for the battery operational strategies.

The rest of the paper is as follows: In Section \ref{sec:related work} we give a overview of related works in this field. In Section \ref{sec:data_set} we discuss the dataset and variables affecting the electricity prices. In Section \ref{sec:prelim} we give a basic mathematical overview of the hidden Markov models and hierarchical Dirichlet process hidden Markov model and its extensions. Section \ref{sec:DS-HDP-HMM} explains the model formation for the regime detection in the electricity price time series. In Section \ref{sec:neural_process} we briefly explain the architecture of conditional neural process where as Section \ref{sec:aggreate_mechanism}. In Section \ref{sec:operational_strategies} we formulate different operational strategies for predicted prices. In Section \ref{sec:numerical_result} we analyze the proposed frame work in terms of prediction accuracy of the model as well the operational strategies. the  explains the framework for the combination of prediction done by each conditional neural processes trained per regimes. Section \ref{sec:conslusion} concludes the paper. The benchmark models, some mathematical explanation and supplementary results are shown in the Appendix.

\section{Literature Review}\label{sec:related work}
 
Electricity can not be treated as other financial commodities and securities: Due to the lack of sufficient storage capabilities in most markets, it is not storable in considerable quantities yet, i.e. once it is produced it must be consumed. This immediate physical delivery property adds complexity to price forecasting \cite{o2025conformal_1, ForecastingElectricityPrices_3}. Besides, the combination of non-storability and its dependency on weather conditions sometimes causes considerable imbalance between supply and demand \cite{inproceedings_6, article_7} which leads to price spikes or also negative prices. Further factors such as demand uncertainty add to the uncertainty. Given the complexity, the development of accurate forecasting model is challenging and there is a vast amount of literature concerning with this topic.
%\cite{o2025conformal_1, forecast6010007_8}. 
As a comprehensive overview is not in the scope of this paper we refrain from this attempt. Instead we solely focus on citing examples for the most widely used methods and especially methods related to our approach. In general, literature suggests that  electricity time series often exhbit a changing behavior, reflecting discrete shifts in the market structure and participant’s strategies \cite{GABRIELLI2022123107_9}. %This is the primary reason regime shift models gained popularity from the late 1990s onwards  as they were capable of incorporating the fast and considerable shifts in the market \textcolor{red}{(REFERENCE FOR THAT)}. 

\subsection{Linear Models and Parametric Heteroskedasticity Models}
Time series models such as the autoregressive integrated moving average (ARIMA) or the generalized autoregressive conditional heteroskedasticity (GARCH) model show greater edge on capturing linear dependencies and are comparably simple models, which have been used to forecast the electricity prices \cite{en14237952, KETTERER2014270}. %The over simplistic formulation of these fails to capture the non-linearity in the electricity price time-series.
In order to better capture non-linear effects, a combination of both model and various modifications has been proposed. For example, in \cite{rubio2023forecasting} authors are building a model based on wavelet ARIMA and GARCH to capture the non-linear patterns in the price time-series. Similarly, a modified version of ARIMA which has a exogenous component, namely, ARIMAX is combined with the GARCH model to include the effect of weather variables on the price and capture the non-linearity in \cite{ HUURMAN20123793}. Several other modifications such as seasonal ARIMA (SARIMA), seasonal ARIMAX  (SARIMAX) and different combinations are explored in \cite{ etep_1734, DIONGUE2009505} and the references therein. Even though these combined models are capable of capturing the non-linearities within the time-series, the shift in the regimes or even the regime detections remains unexplored. In \cite{article_schwartz} two-factor model provides a foundational framework by decomposing commodity prices into a mean-reverting short-term factor and a geometric Brownian motion long-term factor, it assumes constant parameters. This inherent rigidity fails to capture the empirically observed non-stationary and episodic changes in commodity market dynamics, such as shifts between periods of high and low volatility or distinct mean-reversion characteristics. Consequently, a single set of parameters may lead to suboptimal forecasting and risk management, especially during periods of structural change or crisis. While methods like those explored in \cite{Schluter_Deuschle_2014}, which leverage wavelet transforms for multi-resolution signal decomposition and noise reduction, can improve forecasting accuracy by isolating distinct frequency components, they primarily focus on the statistical characteristics of the time series itself. However, these wavelet-based approaches do not inherently model or provide direct economic interpretation for the underlying fundamental shifts or distinct market states that drive significant changes in the energy market's behavior. This problem is addressed in \cite{le2022variance} by formulating the regime-switching models, allowing parameters  to adapt to distinct market states. The model demonstrate its applicability in crude oil market by developing a tractable regime-switching stochastic volatility model where only the market price of volatility risk transitions between two distinct regimes (normal and crisis), while the underlying risk-neutral dynamics remain consistent.

Although the aforementioned models are effective in modeling linear temporal dependencies and volatility clustering, they have per construction limitations regarding the parametric structure and when it comes to nonlinear dynamics or abrupt shifts in market conditions. To address this, hybrid models have emerged, which often pair statistical approaches that capture linear dynamics with machine learning methods designed to learn nonlinear patterns. For instance, an ARIMA or GARCH model might be combined with a neural network to jointly model trend and nonlinear deviations. 

While this approach extends the expressive power of traditional models, its performance is highly sensitive to the choice and integration of components. Moreover, it still assumes a globally consistent structure and often overlooks latent temporal regimes, especially when the transitions are not explicitly modeled.

\subsection{Machine Learning Models}
Machine learning/deep learning models have gained popularity in almost every field research and %of study spanning from social sciences, technology to arts and energy economics is not an exception.
several neural network architectures have been developed for price prediction tasks; \cite{SINGHAL2011550,POURDARYAEI2024121207, GHIMIRE2024122059} and the references therein are just a few to mention. Particularly, the long short-term memory (LSTM) and deep neural networks (DNN) -- due to the capabilities of handling non-linearity within the data -- have been able to outperform various benchmark models such as the aforementioned ARIMA model \cite{LAGO2021116983}. Ensemble DNN models, as per the claim by the authors in \cite{LAGO2021116983}, are increasing performance even further. Roughly speaking ensemble models combine two or more models from the same or a different class in order to improve foreasting accuracy. For example, the ensemble model proposed by \cite{LAGO2021116983} is the combination of four DNN models trained with different training batch sizes. However, their DNN has only been tested against the ensemble Lasso estimated autoregressive (LEAR) method and its own individual models. For a better comparison, the ensemble DNN should be compared with other ensembles e.g. \cite{das2025electricitypricepredictionusing}. Also the DNN models in \cite{LAGO2021116983} does not quantify the uncertainty which is highly important in price prediction. Global deep learning models such as LSTMs or transformer models, known for their attention mechanisms that allow them to weigh the importance of different parts of input sequences for capturing long-range dependencies, though being more flexible, exhibit related limitations \cite{vaswani2017attention}. When trained on data spanning multiple regimes, the learned function approximator effectively minimizes loss over a multi-modal conditional distribution blending behavior across regimes. This leads to mode averaging, where the model fails to accurately represent regime-specific dynamics, potentially resulting in biased or over-smoothed forecasts particularly near regime transitions. Ensemble approaches that average outputs from models trained on different temporal windows attempt to mitigate this but do not solve the core issue: the absence of regime-awareness in the learning objective. To address this issue, a model that can capture the abrupt market shift and can account for the time-varying price dynamics, is required.

\subsection{Regime Switching Models}
The work by Hamilton \cite{Hamilton_1989}, in which he uses Markov regime switching (MRS) to forecast the electricity price of New Zealand, is fundamental. In early 2000s regime-based models were used to forecast the German, UK and Dutch electricity market \cite{HUISMAN2003425}. The ability to explicitly model different market states is a significant advantage of regime-switching models, as it allows for a more nuanced understanding of price dynamics compared to models that assume a single underlying process \cite{DECASTROMATIAS2025108341,demiralay2024uncertainty,GABRIELLI2022123107_9}. Key aspect is determining the number of regimes. This choice often reflects the different market conditions that are believed to influence electricity prices. For instance, a two-regime model might distinguish between a base regime characterized by normal price levels and volatility, and a spike regime with significantly higher prices and volatility \cite{article_7}. %More complex models can include additional regimes to represent other market states, such as periods of low demand or high renewable energy generation.
Thereby the transition probabilities between the individual regimes are a crucial specification, as they determine the likelihood and persistence of each regime \cite{article_7}. Various types of regime-switching models exist, allowing for regime-dependent means, variances (modeling heteroscedasticity), or even the entire dynamic structure of the price process \cite{inproceedings_6}. The flexibility of regime-switching models allows them to be tailored to the specific characteristics of the electricity market under study by adjusting the number of regimes, the transition probabilities, and the within-regime dynamics \cite{Ahlawat2025}.
 \cite{MR4784780} formulate an optimal trading strategy for energy storage devices in balancing markets by modeling price dynamics with a MRS framework combined with stochastic differential equations (SDE).  While the hybrid SDE-MRS approach captures short-term stochasticity and structural transitions, it assumes a fixed number of regimes and lacks mechanisms for adaptive regime expansion or robust long-term memory. Whereas in \cite{MR4667848}, the authors evaluate MRS and Extreme Value Theory (EVT) models on New Zealand electricity price data. The MRS framework models up to five regimes with time-varying transitions and exogenous market covariates, demonstrating strong performance under stable market conditions. EVT performs well in capturing tail behavior. However, both approaches are limited: MRS models require predefined regime counts and may overfit, while EVT lacks dynamic adaptability. Similarly, the contribution in \cite{MR3451179} introduces regime-switching Lévy semistationary (LSS) processes to account for the impact of wind penetration on electricity prices. The regime switches are driven by the wind penetration index, capturing the structural dependence between renewables and prices. While capturing forward-looking impacts, this framework embeds the regimes in deterministic external signals, and lacks latent stochastic transitions or hierarchical structure.  The work in \cite{MR3184351} surveys and advances MRS models for electricity spot prices, emphasizing their success in capturing stylized features such as spikes, mean-reversion, and transition persistence. However, classical MRS frameworks remain limited due to the requirement of fixed regime counts and the assumption of memoryless transitions. Moreover, joint inference of parameters and states becomes computationally burdensome.  In \cite{MR3088175}, a self-excited threshold regime model is developed, allowing electricity price spikes to emerge as stochastic orbits in a nonlinear phase space, without explicit jump processes. The model achieves spike realism through endogenous dynamics but introduces significant complexity and lacks modularity for multi-covariate learning. The study in \cite{MR2507759} explores heavy-tailed MRS models for electricity prices, challenging the norm of log-price modeling and instead advocating for raw-price modeling with t-distributed errors. This improves tail fit but does not resolve core MRS limitations: fixed regime counts and limited memory. Finally, in \cite{MR2328405}, the authors introduce a regime-switching model that accommodates long-memory processes via fractional integration. This is particularly relevant for Nordic markets, where transmission congestion and integration policies yield regimes with persistent memory. While theoretically appealing, the model faces computational and identifiability challenges. 

Taken together, these prior works have significantly advanced electricity price modeling, particularly through regime-switching mechanisms. Yet, they often fall short in terms of adaptability, memory modeling, uncertainty quantification,  and scalability. This is why, in this article, we propose a new model which uses a modified version of the hierarchical Dirichlet process hidden Markov model (HDP-HMM) which is disentangled sticky HDP-HMM. For the ease of readers, we briefly introduce the hidden Markov models and its variants in the next section.

\section{Data Exploration}\label{sec:data_set}
In 2023, renewable energy sources supplied 57\% of Germany’s total electricity demand, with onshore wind being the most prominent, followed by solar and hydropower \cite{Fraunhofer}. This substantial share, coupled with the inherent variability of renewable sources, makes them a key factor influencing electricity prices in the country. Our study utilizes data from the Federal Network Agency of Germany (Bundesnetzagentur\footnote{{\href{https://www.smard.de/home}{Link to Data: Accessed on March 12, 2025}}}), which includes historical electricity prices, residual load forecasts, and projections of renewable energy generation.

In electricity markets, renewable energy is typically traded with priority. Therefore, the residual load—defined as gross electricity demand minus the renewable energy output—is especially relevant, as it indicates the amount of power that must be generated from conventional (non-renewable) sources. A higher residual load signals lower renewable output and a greater reliance on non-renewable generation to meet demand. Figure \eqref{price_load_scatter} illustrates how electricity prices relate to the residual load forecast, while Figure \eqref{price_ren_scatter} depicts their relationship with total forecasted renewable production.

Incorporating both residual load and total renewable output into the model allows it to better capture underlying trends in price formation.

\begin{figure}[ht!]
    \centering
    \begin{subfigure}{0.5\linewidth}
        \includegraphics[height = 5cm, width = 6.5cm]{ 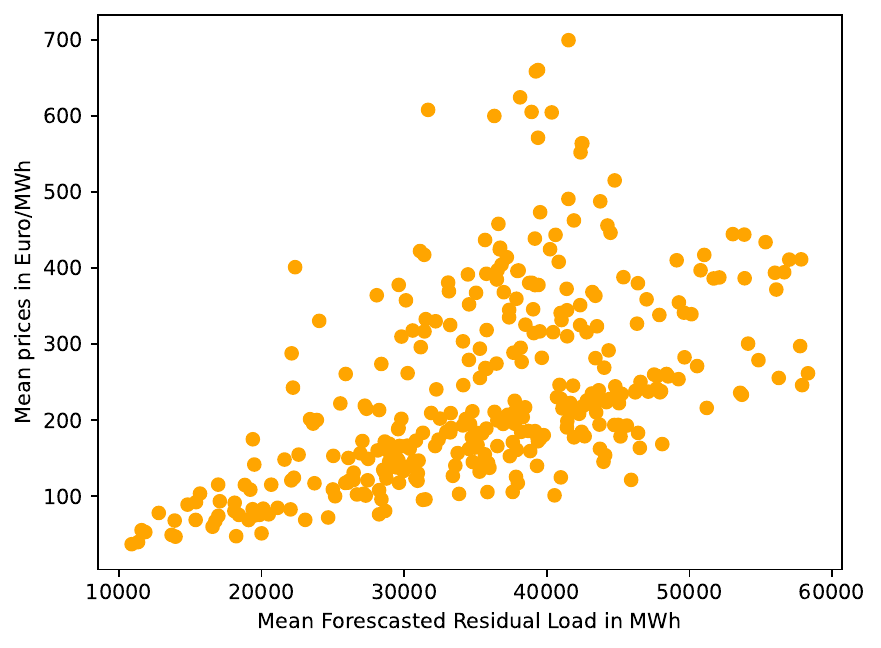}
        \caption{}
        \label{price_load_scatter}
    \end{subfigure}
    \begin{subfigure}{0.45\linewidth}
        \includegraphics[height = 5cm, width = 6.5cm]{ 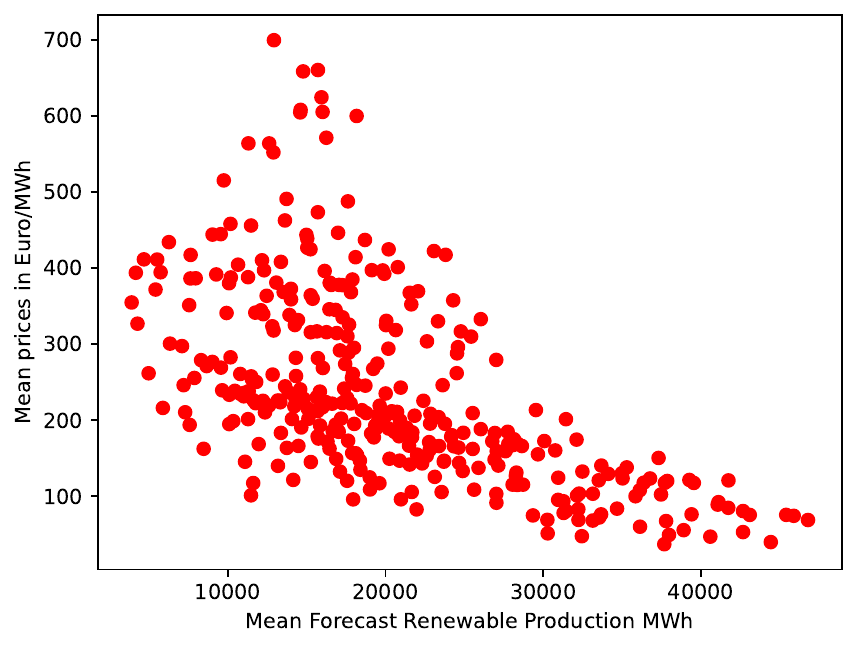}
        \caption{}
        \label{price_ren_scatter}
    \end{subfigure}
    \caption{\centering{Scatter Plot of Historic Data for (a) Daily Average of Forecast Residual Load vs Daily Average of Price and (b) Daily Average of Forecast Total Renewable Production vs Daily Average of Price, for One Year}}
    \label{scatter}
\end{figure}

Despite the known relevance of meteorological data, we exclude weather variables from the model due to their inherently local, station-specific nature, which is inconsistent with our use of nationally aggregated datasets. Although weather data are available, their integration poses challenges in terms of dimensionality, spatial aggregation, and data quality. Additionally, weather parameters do not directly influence electricity prices; their impact is mediated through renewable energy generation, which is explicitly modeled via forecasted production from solar, wind, biomass, and hydropower sources.

The dataset is structured as in \cite{das2025electricitypricepredictionusing}:

$$\text{Price data} = \begin{bmatrix}
    P_{1}^{(1)}&P_{2}^{(1)}, \cdots, &P_{24}^{(1)}\\
    P_{1}^{(2)}&P_{2}^{(2)}, \cdots, &P_{24}^{(2)}\\
    \vdots & \vdots  &\vdots\\
    P_{1}^{(n)}&P_{2}^{(n)}, \cdots, &P_{24}^{(n)}
\end{bmatrix}_{n\times 24}$$ 

Similarly, we have the load data and total energy production (TRP) data as follows:

$$\text{Load data} = \begin{bmatrix}
    L_{1}^{(1)}&L_{2}^{(1)}, \cdots, &L_{24}^{(1)}\\
    L_{1}^{(2)}&L_{2}^{(2)}, \cdots, &L_{24}^{(2)}\\
    \vdots & \vdots  &\vdots\\
    L_{1}^{(n)}&L_{2}^{(n)}, \cdots, &L_{24}^{(n)}
\end{bmatrix}_{n\times 24}\text{ and\hspace{0.5cm}} \text{TRP data} = \begin{bmatrix}
    R_{1}^{(1)}&R_{2}^{(1)}, \cdots, &R_{24}^{(1)}\\
    R_{1}^{(2)}&R_{2}^{(2)}, \cdots, &R_{24}^{(2)}\\
    \vdots & \vdots  &\vdots\\
    R_{1}^{(n)}&R_{2}^{(n)}, \cdots, &R_{24}^{(n)}
\end{bmatrix}_{n\times 24}$$

In this work the prediction models are trained using the a 248-dimensional input, $\mathbf{x}_{t}$ and 24 -dimensional output data $P_{\mathbf{x}_{t}}$ as follows: 
\begin{equation}
\label{train_set}
\mathbf{P}_{\text{train}} 
= \{P_{\mathbf{x}_{t}},: \mathbf{x}_{t} \in T \subset \mathbb{R}^{248}, \forall i=1,\cdots n\}.
\end{equation}
where $$\mathbf{x}_{t} = [i, P^{(i-1)}, P^{(i-2)}, P^{(i-3)}, P^{(i-7)}, L^{(i)}, L^{(i-1)}, L^{(i-7)}, R^{(i)}, R^{(i-1)}, R^{(i-7)}, BD^{(i)}].$$ Unless explicitly specified, we consider $P_{\mathbf{x}_{t}}$ and $P_{t}$ as same for notational simplicity.

\section{Fundamentals of Hidden Markov Models}\label{sec:prelim}
In this section we briefly introduce the concept of Bayesian hidden Markov model (HMM) and hierarchical Dirichlet process (HDP-HMM) and their application to electricity price time series.
\subsection{Hidden Markov Model}

The Hidden Markov Model is a statistical model that describes a system that is assumed to be a Markov process with unobserved (hidden) states.  For electricity prices, let $P_t$ denote the hourly mean price at day $t$, and let $z_t$ represent the hidden state at day $t$, where $z_t \in \{1, 2, ..., K\}$, and $K$ is the number of hidden states. The HMM is formally defined by a set of parameters $\Theta = \{\boldsymbol{\pi}, \mathbf{A}, \boldsymbol{\theta}\}$, where:
\begin{itemize}
    \item \textbf{Initial State Probabilities:} The vector $\boldsymbol{\pi} = (\pi_1, \pi_2, \dots, \pi_K)$ specifies the probability of starting in each hidden state at time $t=1$. Specifically, $\pi_i = \mathbb{P}(z_1 = i)$, with the constraint that $\sum_{i=1}^{K} \pi_i = 1$.
    \item \textbf{Transition Probabilities:} The transition matrix $\mathbf{A} = \{a_{ij}\}$ of size $K \times K$ defines the probabilities of switching from one hidden state to another. An element $a_{ij}$ represents the probability of transitioning from state $i$ at time $t-1$ to state $j$ at time $t$, i.e., $a_{ij} = \mathbb{P}(z_t = j \mid z_{t-1} = i)$. For each row, the probabilities must sum to one: $\sum_{j=1}^{K} a_{ij} = 1$ for all $i \in \{1, \dots, K\}$. This matrix governs the dynamics of switching between different price regimes.
    \item \textbf{Emission Probabilities:} For each hidden state $i$, there is an associated emission probability distribution $f(P_t \mid \theta_i)$. This distribution models the probability of observing the electricity price $P_t$ given that the system is in hidden state $i$, parameterized by $\theta_i$. Common choices for $f$ include Gaussian distributions for continuous data, where $\theta_i = (\mu_i, \sigma_i^2)$ represents the mean and variance specific to regime $i$. This component captures the distinct statistical properties of electricity prices within each regime.
\end{itemize}
The likelihood of an observed mean electricity price sequence $P = \{P_1, P_2, \dots, P_T\}$ given model parameters $\Theta$ can be efficiently calculated using the forward algorithm. The estimation of model parameters $\Theta$ from observed data is typically performed using the Expectation-Maximization (EM) algorithm, specifically the Baum-Welch algorithm, which iteratively refines the parameters to maximize the likelihood of the observations. For a more detailed exposition of the HMM and its associated algorithms, we refer the reader to \cite{1165342}.

Despite its utility, the standard HMM possesses several limitations when directly applied to complex time series like electricity price data. A primary constraint is the requirement to pre-determine the number of hidden states $K$. In dynamic electricity markets, the true number of underlying price regimes may not be known a priori and, critically, can change over time due to evolving market conditions. Furthermore, estimating the HMM parameters with absolute certainty can be challenging, especially with limited data or when regimes are not clearly separable \cite{zhou2020disentangled}. This inherent uncertainty in parameter estimation suggests that a Bayesian approach, which naturally quantifies uncertainty through posterior distributions, would be more desirable for robust modeling.

\subsection{Hierarchical Dirichlet Process Hidden Markov Model}

The Hierarchical Dirichlet Process Hidden Markov Model (HDP-HMM) addresses the limitation of a fixed number of states by using a Bayesian nonparametric approach \cite{ teh2004sharing}. It allows the number of states to grow as needed based on the data. The HDP-HMM leverages the Dirichlet Process (DP), a stochastic process that defines a probability distribution over probability distributions.

The generative process of of the HDP-HMM is defined as follows:
Let $T$ denote the number of time steps. The generative process of the HDP-HMM is defined as follows:

\begin{align}
\beta &\sim \text{GEM}(\gamma), \quad \text{where } \beta_k = v_k \prod_{l=1}^{k-1} (1 - v_l),\quad v_k \sim \text{Beta}(1, \gamma), \label{eq:stickbreaking} \\
\boldsymbol{\pi}_j &\sim \text{DP}(\alpha, \beta), \quad \forall j \in \mathbb{N}, \label{eq:transition_dists} \\
\theta_k &\sim H, \quad \forall k \in \mathbb{N}, \label{eq:emission_params} \\
z_0 &\sim \pi_0, \quad \text{(initial state)} \label{eq:initial_state} \\
z_t \mid z_{t-1} &\sim \boldsymbol{\pi}_{z_{t-1}}, \quad t = 1, \dots, T, \label{eq:latent_states} \\
y_t \mid z_t &\sim f(P \mid \theta_{z_t}), \quad t = 1, \dots, T. \label{eq:observations}
\end{align}

Here $\gamma > 0$ is the concentration parameter controlling the spread of the global weights $\beta$. where $\beta \sim \text{GEM}(\gamma)$. The GEM distribution is derived from a stick-breaking process, which effectively defines a discrete probability distribution over an infinite number of potential states. For each potential state $j \in \{1, 2, \dots\}$, an emission parameter $\theta_j$ is drawn from a base measure $H$, i.e., $\theta_j \sim H$. This $H$ is typically a conjugate prior for the emission distribution $f(P \mid \theta_j)$. In this work we have considered $f$ is a Gaussian distribution, hence $H$ should be a Normal-Inverse-Gamma distribution. Similarly, for each state $j \in \{1, 2, \dots\}$, the outgoing transition distribution $\boldsymbol{\pi}_j = (\pi_{j1}, \pi_{j2}, \dots)$ is drawn from a Dirichlet Process with base distribution $\beta$ and concentration parameter $\alpha$, i.e., $\boldsymbol{\pi}_j \sim \text{DP}(\alpha, \beta)$ and $\pi_0$ is the initial state distribution, often taken as uniform over the first $K$ components in practical truncation. This means that each row of the (potentially infinite) transition matrix $\mathbf{A}$ is a draw from a DP centered around the global $\beta$. The concentration parameter $\alpha > 0$ controls the similarity of the transition distributions across different states; a larger $\alpha$ implies that the $\boldsymbol{\pi}_j$ distributions are more similar to each other and to the global $\beta$. The sequence of latent states $z_t$ evolves according to the Markov property. At each time step $t$, the current state $z_t$ is drawn from the transition distribution corresponding to the previous state $z_{t-1}$, i.e., $z_t \sim \boldsymbol{\pi}_{z_{t-1}}$ for $t = 1, \dots, T$.
The observed data $P_t$ at time $t$ is generated from the emission probability distribution associated with the current latent state $z_t$, i.e., $P_t \sim f(P \mid \theta_{z_t})$ for $t = 1, \dots, T$.

The complete joint distribution over the model’s random variables, including the infinite sequences $(\beta_k)_{k=1}^\infty$, $(\pi_j)_{j=1}^\infty$, $(\theta_k)_{k=1}^\infty$, the latent states $(z_t)_{t=0}^T$, and the observations $(P_t)_{t=1}^T$, is given by:

\begin{align}
&\mathbb{P}(\beta, \{\pi_j\}, \{\theta_k\}, z_{0:T}, y_{1:T}) = \mathbb{P}(\beta) \prod_{j=1}^\infty \mathbb{P}(\pi_j \mid \beta) \prod_{k=1}^\infty \mathbb{P}(\theta_k) \cdot \mathbb{P}(z_0) \nonumber \\
&\quad \times \prod_{t=1}^T \mathbb{P}(z_t \mid \pi_{z_{t-1}}) \cdot \mathbb{P}(P_t \mid \theta_{z_t}) \label{eq:joint}.
\end{align}

This expression encapsulates the entire HDP-HMM. In practice, inference is performed by truncating the infinite-dimensional objects (e.g., using a finite $K$ for $\beta$ and $\pi_j$) and applying variational inference or Gibbs sampling techniques. Despite its flexibility, the HDP-HMM exhibits a limitation in practice: it tends to rapidly switch between states even when the true underlying regime remains constant. This is due to the lack of an inductive bias favoring state persistence. Mathematically, the Dirichlet Process prior over the transition distribution $\boldsymbol{\pi}_j \sim \text{DP}(\alpha, \beta)$ treats all transitions, including self-transitions, as exchangeable. Consequently, the expected self-transition probability $\mathbb{E}[\pi_{jj}]$ is just $\beta_j$, the same as the expected probability of transitioning to any other state $k \neq j$. We have shown this in Appendix \ref{thm:hdp-exchangeable}

This results in an unreasonably high rate of switching, especially when the number of effective states is large, as $\beta_j$ becomes small. In time-series applications such as electricity pricing, where regimes often persist for contiguous durations, this behavior is undesirable.

\subsection{Sticky Hierarchical Dirichlet Process Hidden Markov Model}

The Sticky HDP-HMM, further refines the HDP-HMM by encouraging self-transition \cite{10.1145/1390156.1390196}, which improves upon the standard HDP-HMM by introducing a state-specific self-transition bias. This is particularly important in time series where regimes (latent states) tend to persist over extended durations. We show it mathematically in the Appendix \ref{rem:shdp-self}. The Sticky HDP-HMM introduces a mechanism to capture the observed persistence of the regimes, which can be crucial for accurately modeling periods of prolonged high or low volatility. Mathematically, this is achieved by modifying the base distribution of the Dirichlet Process (DP) prior over transition distributions:

\begin{equation}
\boldsymbol{\pi}_j \sim \mathrm{DP}(\alpha \beta + \kappa \delta_j), \quad j = 1, 2, \dots
\end{equation}

where $\alpha \beta$ is the usual HDP prior with global base measure $\beta \sim \mathrm{GEM}(\gamma)$, $\kappa > 0$ is a "stickiness" parameter that controls the weight placed on self-transitions and $\delta_j$ is a Dirac measure concentrated at state $j$, i.e., $\delta_j(k) = \mathbb{I}\{k = j\}$.

Despite its benefits, a notable limitation of sticky HDP-HMM couples two independent modeling aspects: regime persistence (via $\kappa$) and transition diversity (via $\alpha$). The additive nature of the Dirichlet Process base measure, $\alpha \beta + \kappa \delta_j$, inherently entangles these, making it challenging to tune them independently. To resolve the entanglement in the sticky HDP-HMM, the Disentangled Sticky HDP-HMM (DS-HDP-HMM) is introduced in \cite{zhou2020disentangled} which we briefly explain in the next section.

\section{A New Disentangled Sticky Sticky Hierarchical Dirichlet Process Hidden Markov Model Approach for Electricity Prices}\label{sec:DS-HDP-HMM}

%\textcolor{blue}{reference is not updated in this section and the model formulation section}\hfill\newline\newline
The disentangled sticky HDP-HMM (DS-HDP-HMM) addresses the limitations of the sticky HDP-HMM by providing separate control over the transition probabilities and the self-persistence tendency. We briefly show this mathematically in Appendix \ref{rem::limitsHDPHMM} and \ref{thm:dshdp-expected}. This added flexibility is particularly beneficial for modeling electricity price dynamics, where the persistence of price regimes can vary significantly and independently of the overall similarity of transitions \cite{zhou2020disentangled}.

\subsection{Model Formulation}

The DS-HDP-HMM modifies the transition matrix prior as follows:
        \begin{align*}
            \kappa_j &\sim Beta(\rho_1, \rho_2) \\
            \overline{\pi}_j &\sim DP(\alpha \beta) \\
            \pi_j &= \kappa_j \delta_j + (1 - \kappa_j) \overline{\pi}_j, \quad j = 1, 2, ...
        \end{align*}

For each state $j$, a specific self-persistence probability $\kappa_j$ is drawn from a beta distribution with parameters $\rho_1$ and $\rho_2$. The beta distribution, defined on the interval $[0,1]$, is a natural choice for modeling probabilities. In the context of time series, $\kappa_j$ quantifies the probability that the series regime at time $t$ will remain the same at time $t+1$, given that the current regime is $j$. A high value of $\kappa_j$ indicates a highly persistent regime (e.g., a prolonged period of stable prices or sustained high volatility), whereas a low $\kappa_j$ suggests a short-lived regime. {Transition distribution to other states $\overline{\boldsymbol{\pi}}_j$:} This component, $\overline{\boldsymbol{\pi}}_j$, represents the probability distribution over all other states (excluding the current state $j$) to which the system can transition. It is drawn from a Dirichlet Process with a global base distribution $\beta$ and a concentration parameter $\alpha$, i.e., $\overline{\boldsymbol{\pi}}_j \sim \text{DP}(\alpha \beta)$. As in the HDP-HMM, $\beta$ is a global "menu" of possible states drawn from $\text{GEM}(\gamma)$, and $\alpha$ controls the similarity of these transition distributions across different states. {Mixture transition distribution $\boldsymbol{\pi}_j$:} The final transition distribution $\boldsymbol{\pi}_j$ for state $j$ is a convex combination of the self-persistence component and the transition component to other states. With probability $\kappa_j$, the next state is deterministically the same as the current state (represented by $\delta_j$, the Dirac delta function centered at $j$). With probability $(1 - \kappa_j)$, the next state is drawn from $\overline{\boldsymbol{\pi}}_j$, allowing for transitions to any other regime.

An equivalent formulation of the latent state dynamics using auxiliary variables provides further clarity:
\begin{align*}
    w_t &\sim \text{Bernoulli}(\kappa_{z_{t-1}}) \\
    z_t &\sim w_t \delta_{z_{t-1}} + (1 - w_t) \overline{\boldsymbol{\pi}}_{z_{t-1}}, \quad t = 1, \dots, T
\end{align*}
Here, $w_t$ is a binary auxiliary variable. If $w_t = 1$, it indicates that the system persists in the same price regime as $z_{t-1}$. If $w_t = 0$, it signifies a transition to a different regime, with the new regime $z_t$ being drawn from $\overline{\boldsymbol{\pi}}_{z_{t-1}}$. This formulation makes the disentanglement explicit: the decision to persist or switch is made first (governed by $\kappa_{z_{t-1}}$), and then, if a switch occurs, the destination regime is chosen from $\overline{\boldsymbol{\pi}}_{z_{t-1}}$.

\subsection{Gibbs Sampling Inference}
To estimate the parameters of the DS-HDP-HMM and infer the hidden electricity price regimes, we employ a Markov Chain Monte Carlo (MCMC) inference scheme, specifically a Gibbs sampling approach. Among the various Gibbs samplers for HDP-HMMs, the Direct Assignment Sampler is particularly well-suited as it directly samples the latent states, making it efficient for inference. This sampler targets the true posterior distribution of the DS-HDP-HMM, providing samples from which we can estimate the parameters and the sequence of hidden states. The direct assignment sampler for electricity price data proceeds iteratively through the following steps:
\begin{enumerate}
    \item \textbf{Sequential Sampling of Latent States and Auxiliary Variables $\{z_t, w_t\}$:} For each time step $t \in \{1, \dots, T\}$, we sequentially sample the current price regime $z_t$ and the self-persistence indicator $w_t$, conditional on all other latent variables $\{z_{s \neq t}, w_{s \neq t}\}$, the observed mean prices $P = \{P_1, \dots, P_T\}$, and the model hyperparameters $(\alpha, \gamma, \rho_1, \rho_2)$. This step involves calculating the full conditional probability $\mathbb{P}(z_t, w_t \mid \dots)$ for each possible combination of $z_t$ and $w_t$. The probability of assigning $z_t$ to a particular state $k$ and $w_t$ to 0 or 1 depends on:
    \begin{itemize}
        \item The emission probability $\mathbb{P}(P_t \mid z_t = k, \theta_k)$.
        \item The transition probability from the previous state $z_{t-1}$ to $z_t$, which is $\mathbb{P}(z_t \mid z_{t-1}, \kappa_{z_{t-1}}, \overline{\boldsymbol{\pi}}_{z_{t-1}})$.
        \item The transition probability from $z_t$ to the next state $z_{t+1}$, which is $\mathbb{P}(z_{t+1} \mid z_t, \kappa_{z_t}, \overline{\boldsymbol{\pi}}_{z_t})$.
    \end{itemize}
    If a new price regime (i.e., a state $z_t$ that has not been previously observed) is encountered during this sampling process, the number of active regimes $K$ is effectively incremented, and a new self-persistence probability $\kappa_{K+1}$ is sampled for this newly discovered regime from its Beta prior.
    \item \textbf{Sampling Self-Persistence Probabilities $\{\kappa_j\}_{j=1}^{K}$:} Given the sampled self-persistence indicators $w_t$ for all $t$, we can update our belief about the self-persistence probabilities $\kappa_j$ for each active regime $j$. This step typically leverages the Beta-Binomial conjugacy. For each state $j$, if $N_j^{\text{persist}}$ is the number of times the system transitioned from state $j$ to itself (i.e., $w_t=1$ when $z_{t-1}=j$), and $N_j^{\text{switch}}$ is the number of times it transitioned from state $j$ to another state (i.e., $w_t=0$ when $z_{t-1}=j$), then the posterior distribution for $\kappa_j$ is:
    $$\kappa_j \mid \dots \sim \text{Beta}(\rho_1 + N_j^{\text{persist}}, \rho_2 + N_j^{\text{switch}})$$
    \item \textbf{Sampling Global Base Distribution $\beta$:} The global transition base distribution $\beta$ is sampled using Dirichlet-multinomial conjugacy, similar to the standard HDP-HMM. This involves introducing auxiliary variables (often referred to as``Chinese Restaurant Franchise" variables) that track the number of times each state has been chosen as a destination state across all transitions. The posterior for $\beta$ becomes a Dirichlet distribution whose parameters are updated based on these counts.
    \item \textbf{Sampling Emission Parameters $\{\theta_j\}_{j=1}^{K}$:} For each active regime $j$, the emission parameters $\theta_j$ (e.g., mean and variance for a Gaussian emission) are sampled from their posterior distribution, conditional on all observations $y_t$ assigned to regime $j$. This step also utilizes conjugacy; for example, if the emission is Gaussian with a Normal-Inverse-Gamma prior, the posterior will also be Normal-Inverse-Gamma, updated by the sufficient statistics of the observed data points in regime $j$.
    \item \textbf{Sampling Hyperparameters:} The global hyperparameters $\alpha$, $\gamma$, $\rho_1$, and $\rho_2$ are also sampled from their respective posterior distributions.
    \begin{itemize}
        \item \textbf{$\alpha$ (Concentration for $\overline{\boldsymbol{\pi}}_j$):} The concentration parameter $\alpha$ is typically sampled using a Gamma-conjugate prior, given auxiliary variables derived from the empirical transition counts between states. This involves a slice sampling or Metropolis-Hastings step.
        \item \textbf{$\gamma$ (Concentration for $\beta$):} Similarly, $\gamma$ is sampled from its posterior, often using a Gamma prior and auxiliary variables related to the number of distinct states observed.
        \item \textbf{$\rho_1$ and $\rho_2$ (Beta prior for $\kappa_j$):} The parameters $\rho_1$ and $\rho_2$ of the Beta distribution for the self-persistence probabilities $\kappa_j$ are often sampled using a grid approximation of their joint posterior distribution or a Metropolis-Hastings step. This involves evaluating the posterior probability at various candidate values of $\rho_1$ and $\rho_2$ and sampling from this discretized distribution.
    \end{itemize}
\end{enumerate}
By iteratively performing these sampling steps, the Gibbs sampler generates a sequence of samples from the joint posterior distribution of all model parameters and latent states. After a sufficient burn-in period, these samples can be used to obtain point estimates (e.g., posterior means) and credible intervals for the parameters, as well as to infer the most probable sequence of hidden price regimes.

\begin{figure}[ht!]
    \centering
    \begin{subfigure}{0.58\linewidth}
        \includegraphics[height = 5cm, width = 12cm]{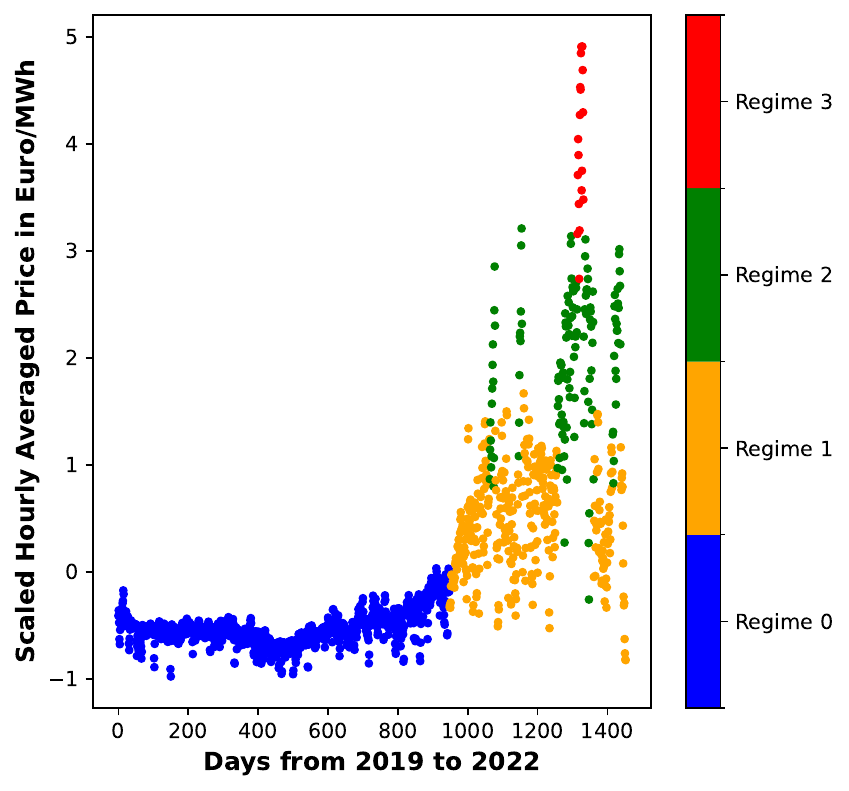}
        \caption{}
        \label{4_regimes}
    \end{subfigure}
    \begin{subfigure}{0.55\linewidth}
        \includegraphics[height = 5cm, width = 10cm]{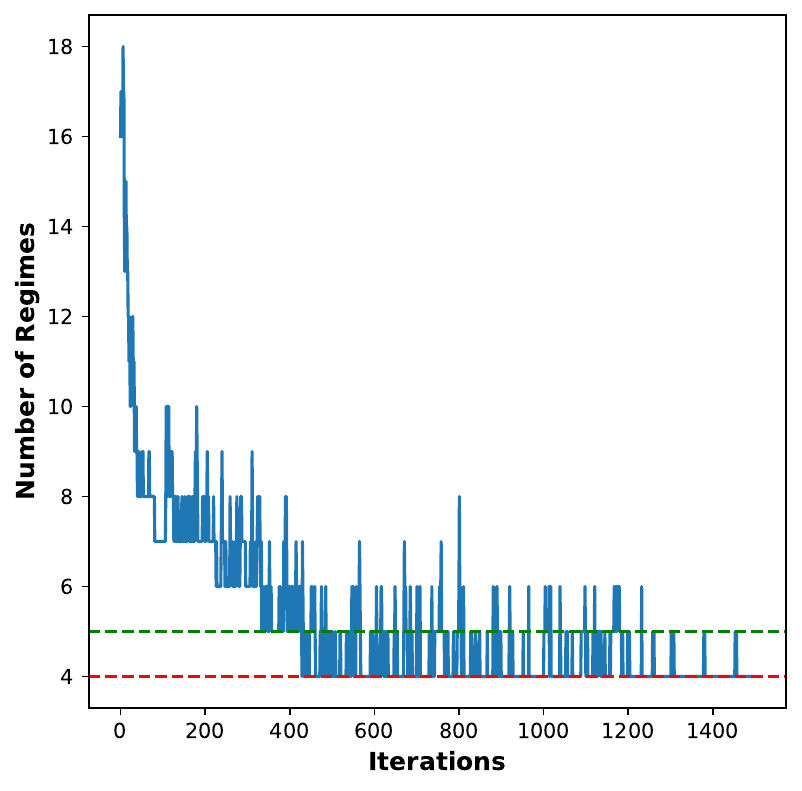}
        \caption{}
        \label{Regime_Convergence_4regimes}
    \end{subfigure}
    \caption{\centering{Regime Detection and Convergence}}
    \label{DS_2019_2022_4_regimes}
\end{figure} 

\begin{figure}[ht!]
    \centering
    \begin{subfigure}{0.8999\linewidth}
        \includegraphics[height = 4.5cm, width = 14cm]{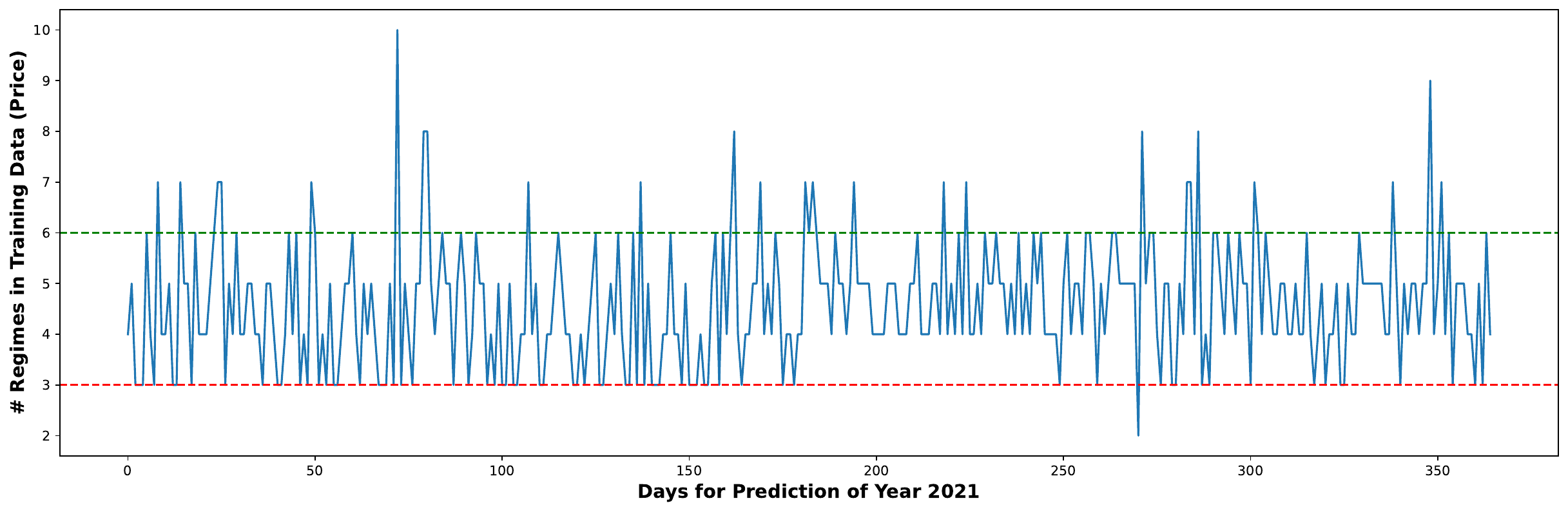}
        \caption{}
        \label{num_regime_2021}
    \end{subfigure}
    \begin{subfigure}{0.8999\linewidth}
        \includegraphics[height = 4.5cm, width = 14cm]{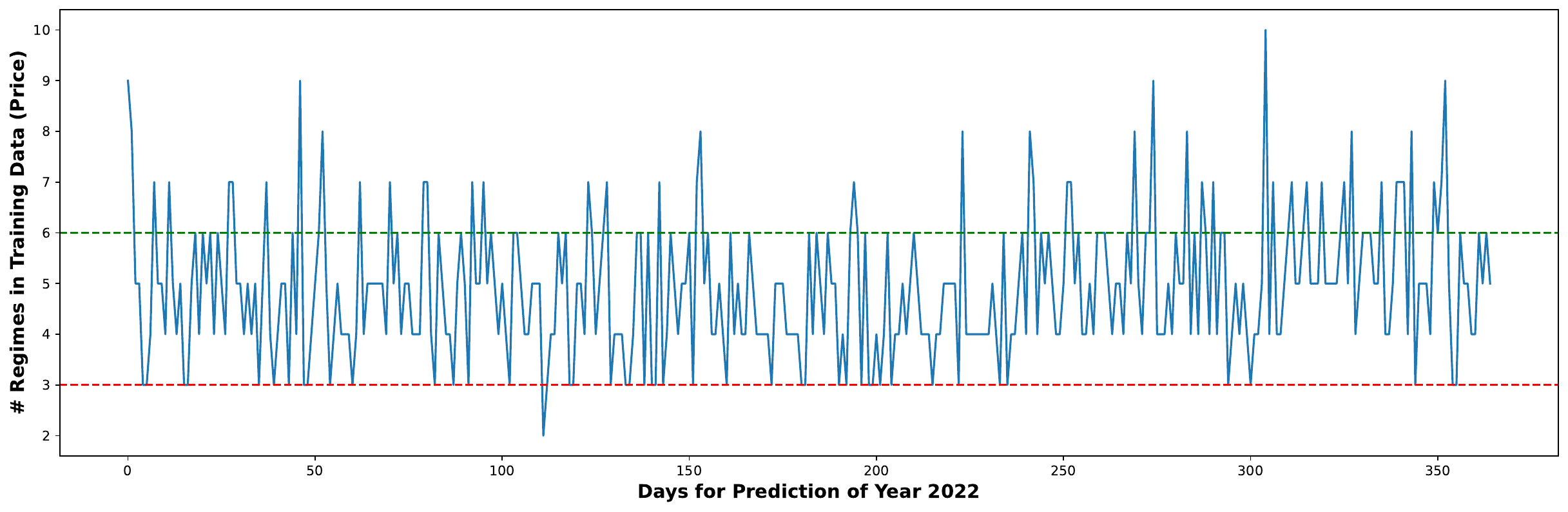}
        \caption{}
        \label{num_regime_2022}
    \end{subfigure}
    \begin{subfigure}{0.8999\linewidth}
        \includegraphics[height = 4.5cm, width = 14cm]{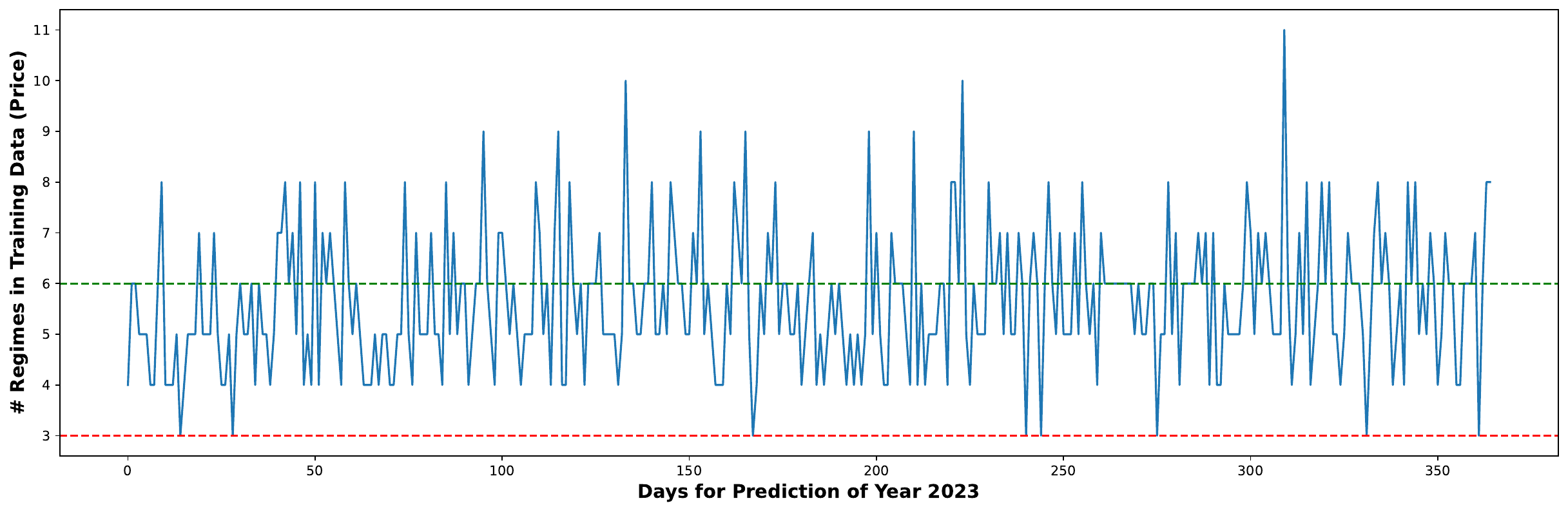}
        \caption{}
        \label{num_regime_2023}
    \end{subfigure}
    \caption{\centering{Number of regimes detected in the training data for prediction of prices from (a) 1 Jan, 2021 to 31 Dec, 2021, (b) 1 Jan, 2022 to 31 Dec, 2022 and (c) 1 Jan, 2023 to 31 Dec, 2023}}
    \label{num_regimes}
\end{figure}
The disentangled structure of the DS-HDP-HMM offers key advantages for modeling electricity price dynamics. It flexibly captures varying regime persistence, distinguishing between long-lasting regimes—such as extended periods of high demand or low prices—and short-lived events like sudden price spikes. The model is also adaptable to evolving market conditions, dynamically learning new regime behaviors in response to structural changes, such as increased renewable energy integration or regulatory shifts. Moreover, by separating parameters for self-persistence and transitions, the DS-HDP-HMM improves inference accuracy and enhances interpretability, providing clearer insights into regime characteristics and transitions—critical for informed decision-making. The Figure \ref{4_regimes} illustrates an example of detected regimes in the training data for electricity price prediction. As depicted, the DS-HDP-HMM successfully identifies distinct periods characterized by different price behaviors, such as periods of high volatility (e.g., spikes), stable base prices, or negative prices. Each color in the plot represents a unique, inferred regime, demonstrating the model's ability to segment the time series into statistically homogeneous segments without prior knowledge of the number of such segments. In Figure \ref{Regime_Convergence_4regimes}, we show the convergence of number of regimes for particular price data and similarly in the Figure \ref{num_regime_2023} and \ref{num_regime_2022} we show the number of regimes detected in the training data for the price prediction from 1 Jan, 2023 to 31 Dec, 2023 and from 1 Jan, 2022 to 31 Dec, 2022 respectively. From Figure \ref{num_regime_2023} and \ref{num_regime_2022} we see that most of the training data has number of regimes ranging from 4 to 6. 
Our model infers the number of regimes nonparametrically by maximizing the marginal likelihood under the HDP-HMM posterior, which adapts the number of regimes to the statistical complexity of each dataset. While the model detects pratically feasible regimes in most cases, some exhibit increased variance or transition frequency, leading to a higher number of latent regimes (e.g., 7) that maximize the marginal likelihood. However, we observe that certain regimes, particularly those with transient shifts or weak evidence separation, may be spurious or statistically similar. To enhance interpretability and ensure robustness, we apply a post hoc compression step that merges regimes whose emission distributions are statistically indistinguishable (low KL divergence) or whose posterior mass is negligible. This compression does not alter the underlying model inference but provides parsimonious summaries suitable for downstream prediction tasks, addressing the challenges of detecting semantically meaningful regimes.

The key advantage of the DS-HDP-HMM lies in its ability to disentangle the parameters influencing electricity price regime transitions, namely, {Concentration parameter $\alpha$:} This parameter controls the overall similarity of the transition distributions $\overline{\boldsymbol{\pi}}_j$ across different price regimes. A large $\alpha$ implies that the transition patterns between regimes are broadly similar, suggesting a more homogeneous market behavior in terms of regime switching. Conversely, a small $\alpha$ allows for more distinct and heterogeneous transition patterns, which might be appropriate in electricity markets where some regimes (e.g., periods of extreme price spikes) are very difficult to escape from, while others (e.g., periods of stable base load) might transition more readily to different states. 

Secondly, {Beta distribution parameters $\rho_1$ and $\rho_2$:} These parameters of the Beta distribution govern the self-persistence probabilities $\kappa_j$. The mean of the Beta distribution, given by $\frac{\rho_1}{\rho_1 + \rho_2}$, directly controls the average self-persistence probability across all regimes. A value close to 1 indicates a strong general tendency for price regimes to persist. The variance of the Beta distribution, calculated as $\frac{\rho_1 \rho_2}{(\rho_1 + \rho_2)^2 (\rho_1 + \rho_2 + 1)}$, controls the variability of the self-persistence probabilities across different regimes. A small variance suggests that all regimes exhibit similar persistence characteristics, while a large variance allows for significant differences in persistence. For example, in electricity markets, some regimes (e.g., extremely volatile periods driven by sudden outages) might be very short-lived (low $\kappa_j$), whereas others (e.g., periods of stable demand and ample supply) might last for extended durations (high $\kappa_j$). The DS-HDP-HMM can capture this heterogeneity.

Next, we train regime-aware regression models to predict future prices for wcich we have training input and output data. The training input data contains the past price data and exogenous variables, where as the output data is the 24 dimensional price of the current day. The figures above shows the detection of the regimes and number of regimes using the training output data. The regimes are intended to represent structural changes in the price generation mechanism — therefore, they must reflect discontinuities or non-stationarities in the conditional distribution $\mathbb{P}(P_t \mid x_t)$. From a statistical modeling standpoint, the mapping $f: x_t \mapsto P_t$, where $x_t\text{ and } P_t$ are the input and output, may change across time, and such regime shifts are best identified by analyzing the marginal behavior of $P_t$, especially when $\{x_t\}$ is high-dimensional, noisy, or semantically heterogeneous (e.g., mixing lagged prices, demand, production, calendar effects). From an information-theoretic perspective, $P_t$ serves as a low-dimensional sufficient statistic of the system’s observable behavior. In contrast, the inputs $x_t$ include multiple exogenous sources with varying signal-to-noise characteristics, making direct clustering of inputs less robust and less interpretable. As a result, we apply an unsupervised Bayesian segmentation algorithm (DS-HDP-HMM) to the sequence $\{P_t\}_{t=1}^{T}$, where $P_t$ is hourly mean price at dat $t$,  to identify latent regimes $\{z_t\}$. This inference is completely independent of the input data $\{x_t\}$. In the next sub-sections we explain the methodology for regime detection and the mathematically rigorous procedure used to associate these regimes to input features.

In our framework, regimes are inferred using a Disentangled Sticky Hierarchical Dirichlet Process Hidden Markov Model (DS-HDP-HMM), which segments the observed univariate time series $\{P_t\}_{t=1}^T$ into a sequence of latent states $\{z_t\}_{t=1}^T$, where $z_t \in \{1, \dots, K\}$ and $K$ is inferred nonparametrically. Each regime $k$ is associated with an emission distribution $\mathbb{P}(P_t \mid z_t = k)$, typically modeled as a Gaussian with unknown mean and known variance. The HDP-HMM further imposes a Markovian transition structure $\mathbb{P}(z_t \mid z_{t-1})$ with state-specific self-transition bias through a sticky parameter $\kappa_k$, encouraging temporal persistence.

While the inferred regimes originate from a dynamic probabilistic process, in practice, the posterior segmentation $\{z_t\}$ often partitions the time series into contiguous blocks of statistically homogeneous observations. These partitions behave similarly to clusters when the posterior transition matrix exhibits high self-transition probabilities and when emission distributions across regimes are well-separated. Specifically, if $\mathbb{E}[\kappa_k]$ is large for most $k$ and the Kullback-Leibler divergence between emission distributions,
\[
D_{\mathrm{KL}}\left( \mathcal{N}(\mu_i, \sigma^2) \, \| \, \mathcal{N}(\mu_j, \sigma^2) \right) = \frac{(\mu_i - \mu_j)^2}{2\sigma^2},
\]
is large for most pairs $(i, j)$, then the resulting regime assignments exhibit cluster-like behavior.

Despite this resemblance, it is important to emphasize that regimes are not equivalent to clusters in a classical i.i.d. mixture model. Clusters are defined purely based on feature similarity, typically via marginal likelihood maximization in a mixture model $\mathbb{P}(y_t) = \sum_{k=1}^K \pi_k \mathbb{P}(y_t \mid \theta_k)$. In contrast, regimes arise from a joint distribution over state sequences and transitions:
\[
\mathbb{P}(y_{1:T}, z_{1:T}) = \prod_{t=1}^T \mathbb{P}(y_t \mid z_t, \theta_{z_t}) \cdot \mathbb{P}(z_t \mid z_{t-1}, \pi),
\]
where $\pi$ denotes the transition matrix drawn from a hierarchical Dirichlet process. This temporal structure introduces dependencies not present in clustering models and can capture subtle nonstationarities that static clustering cannot.

In the context of regression, we leverage regime assignments $\{z_t\}$ as conditionally informative covariates: that is, we train a distinct regression model $f_k(x_t)$ for each regime $k$, assuming that the mapping from inputs $x_t$ to outputs $P_t$ varies with latent state. Although this effectively treats regime assignments as discrete condition labels, they remain grounded in the temporal generative structure of the HDP-HMM. Hence, even when the regression task does not exploit transition dynamics explicitly, the segmentation provided by the DS-HDP-HMM provides a principled, data-driven method of partitioning the data based on latent temporal regimes rather than arbitrary clusters.

\subsection{Input–Regime Association}\label{ss:input_regime_association}

Having inferred regime labels $\{z_t\}$ from the output, we aim to train a classifier that maps input vectors $x_t$ to their associated regime labels $z_t$. Formally, we construct a labeled dataset:

\[
\mathcal{D}_{\text{label}} = \{(x_t, z_t)\}_{t=1}^T
\]

This association is performed index-wise: the regime label $z_t$ inferred from $y_t$ is assigned to the corresponding $x_t$. Since $\{x_t\}$ was not used during the inference of $\{z_t\}$ — the labeling of inputs is performed post hoc. In real-world prediction tasks, the test output $P_t^{\text{test}}$ is unavailable at inference time. Therefore, the ability to infer $z_t$ from $x_t$ alone is not just valid but \emph{essential} for out-of-sample generalization. The regime label $z_t$ can be interpreted as a hidden state influencing the generative process $\mathbb{P}(P_t \mid x_t, z_t)$. Hence, learning the mapping $x_t \mapsto z_t$ constitutes a supervised approximation to the posterior $\mathbb{P}(z_t \mid x_t)$, consistent with the principles of latent-variable modeling.

To predict the regime labels from input data, we train a Multi-Layer Perceptron (MLP) classifier on the labeled dataset $\mathcal{D}_{\text{label}}$. The classifier is evaluated on a held-out test set. The classification performance for one randomly selected day (1 Jan, 2023), summarized in Table~\ref{tab:clf}, demonstrates strong discriminative capacity despite class imbalance.

\begin{table}[h!]
\centering
\caption{Classification Report on Test Set}
\label{tab:clf}
\begin{tabular}{lcccc}
\toprule
Regime & Precision & Recall & F1-score & Support \\
\midrule
0 & 0.99 & 0.89 & 0.99 & 200 \\
1 & 0.91 & 0.82 & 0.86 & 62 \\
2 & 0.69 & 0.75 & 0.72 & 24 \\
3 & 0.83 & 0.92 & 0.91 & 5 \\
\bottomrule
\end{tabular}
\end{table}

%\subsection{Statistical Dependence Between Inputs and Regimes}\label{ss:Statistical_Dependence_Between_Inputs_and_Regimes}

To independently verify that the regimes are reflected in the input space, we compute both Mutual Information (MI) and ANOVA F-statistics between each input feature $x_t^{(j)}$ and the regime label $z_t$ for the same date 1 Jan, 2023. The top three features by mutual information are $I(x^{(0)}; z) = 0.8837, I(x^{(19)}; z) = 0.6788,I(x^{(43)}; z) = 0.6750$. Similarly the top three features by ANOVA F-statistic are $F_{23} = 2628.67,\quad p < 10^{-300}$, $F_{22} = 2582.84,\quad p < 10^{-300}$ and $F_{21} = 2566.55,\quad p < 10^{-300}$. The high scores across both metrics confirm that regime structure is statistically embedded in the input space and justifies the supervised learning approach.

\subsection{Geometric Structure of Regimes in Input Space}

To visualize the regime separability in $\mathbb{R}^d$, we perform dimensionality reduction on the standardized inputs $\{\mathbf{x}_{t}\}$ using t-distributed stochastic neighboring embedding (t-SNE) which a nonlinear embedding that reveals manifold structure and priciple component analysis (PCA) which is a linear projection onto the top three principal components.

\begin{figure}[ht!]
     \centering
    \begin{subfigure}{0.5\linewidth}
        \includegraphics[height = 5cm, width = 8cm]{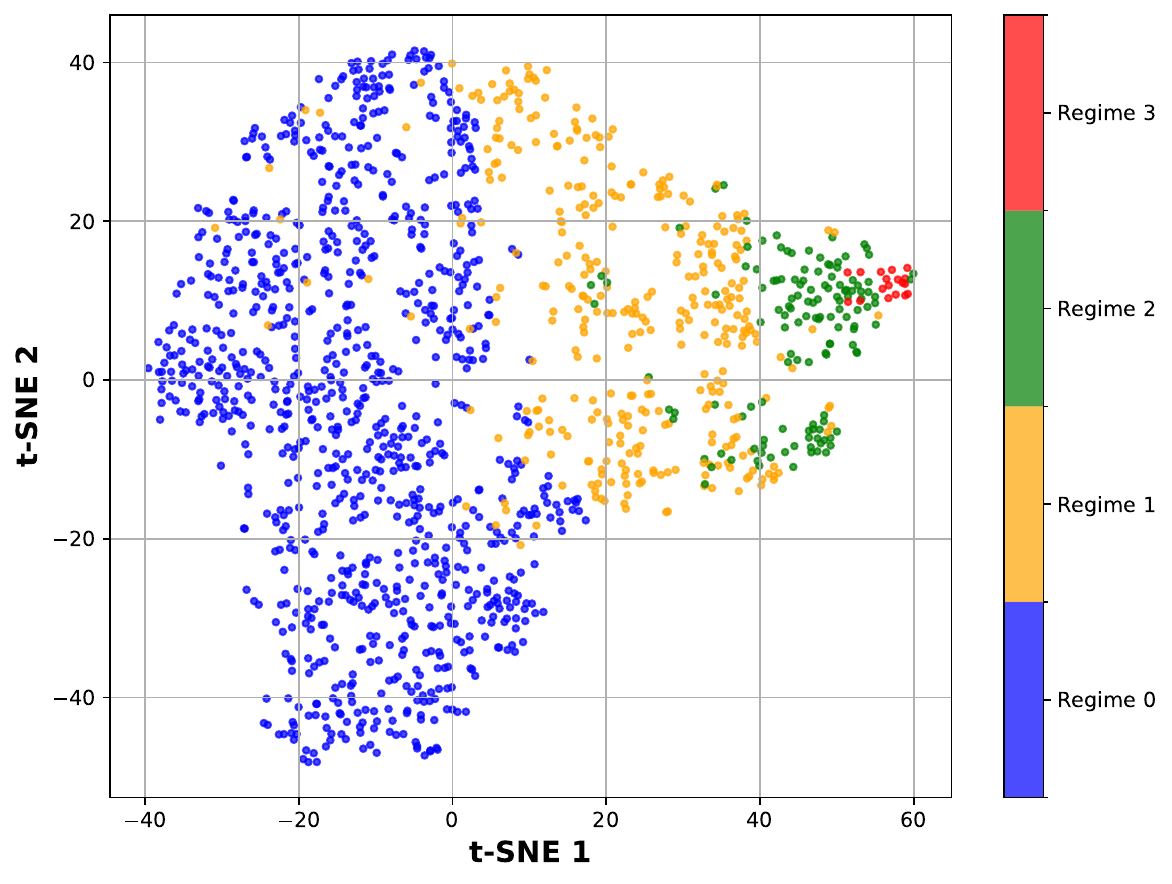}
        \caption{}
        \label{fig:tsne}
    \end{subfigure}
    \begin{subfigure}{0.5\linewidth}
        \includegraphics[height = 5cm, width = 8cm]{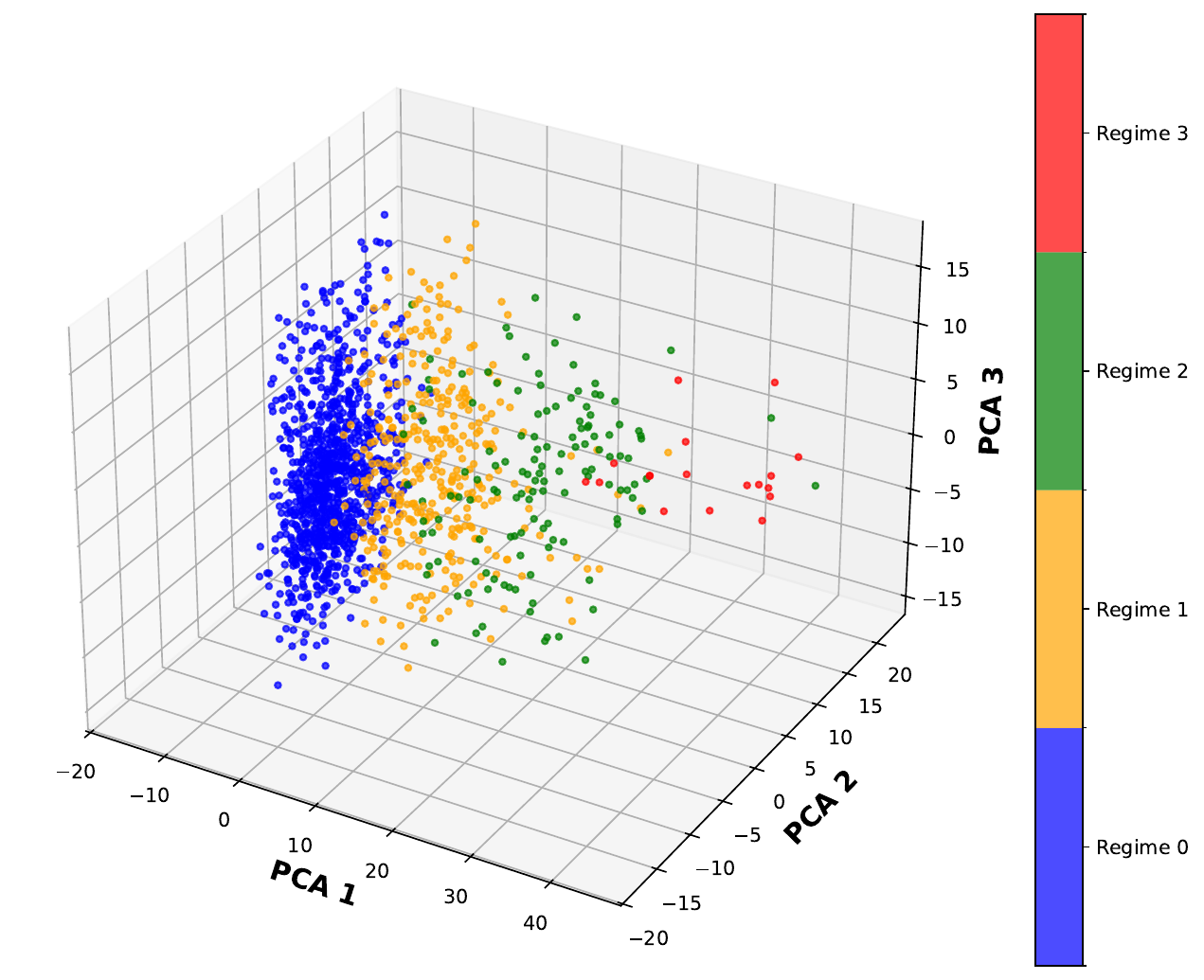}
        \caption{}
        \label{fig:pca}
    \end{subfigure}
    \caption{\centering{Visualization of Regime Separability for 1 Jan, 2023 (a) 2D t-SNE embedding of training input 
    $\{\mathbf{x}_{t}\}$ colored by regime labels $z_t$ and (b) 3D PCA projection of standardized input vectors colored by regime labels.}}
    \label{fig:regime_separability}
\end{figure}

Both visualizations in the Figure \ref{fig:tsne} and \ref{fig:pca} reveal that the regimes inferred from output-only segmentation manifest as coherent clusters in the input space. This provides further empirical evidence that the discovered regimes correspond to meaningful, distinguishable patterns in the system's state as captured by input features. We have demonstrated a mathematically sound pipeline for (i) inferring regimes from output data, (ii) associating these regimes to inputs without leakage, and (iii) training a regime classifier with strong empirical performance. Statistical dependence analysis and geometric visualizations validate that the inferred regimes are embedded in the input space, justifying regime-based modeling for predictive tasks.

To evaluate the consistency and quality of regime separability across datasets, we computed four key aggregate metrics: classification accuracy, weighted F1 score, ANOVA F-statistic, and mean mutual information between features and regimes. The distribution of these metrics across datasets is illustrated in Figure \ref{fig:input_regimes_assignment_analysis}. Notably, the classification accuracy and F1 scores shown in the Figure \ref{subfig:accuracy_box_plot} and \ref{subfig:f1_Weighted_box_plot} respectively, are consistently high across the majority of datasets (median $\approx 0.9$), indicating strong discriminative performance of the learned models.

\begin{figure}[ht!]
    \centering
    \begin{subfigure}{0.5\linewidth}
        \includegraphics[height = 5.5cm, width = 7cm]{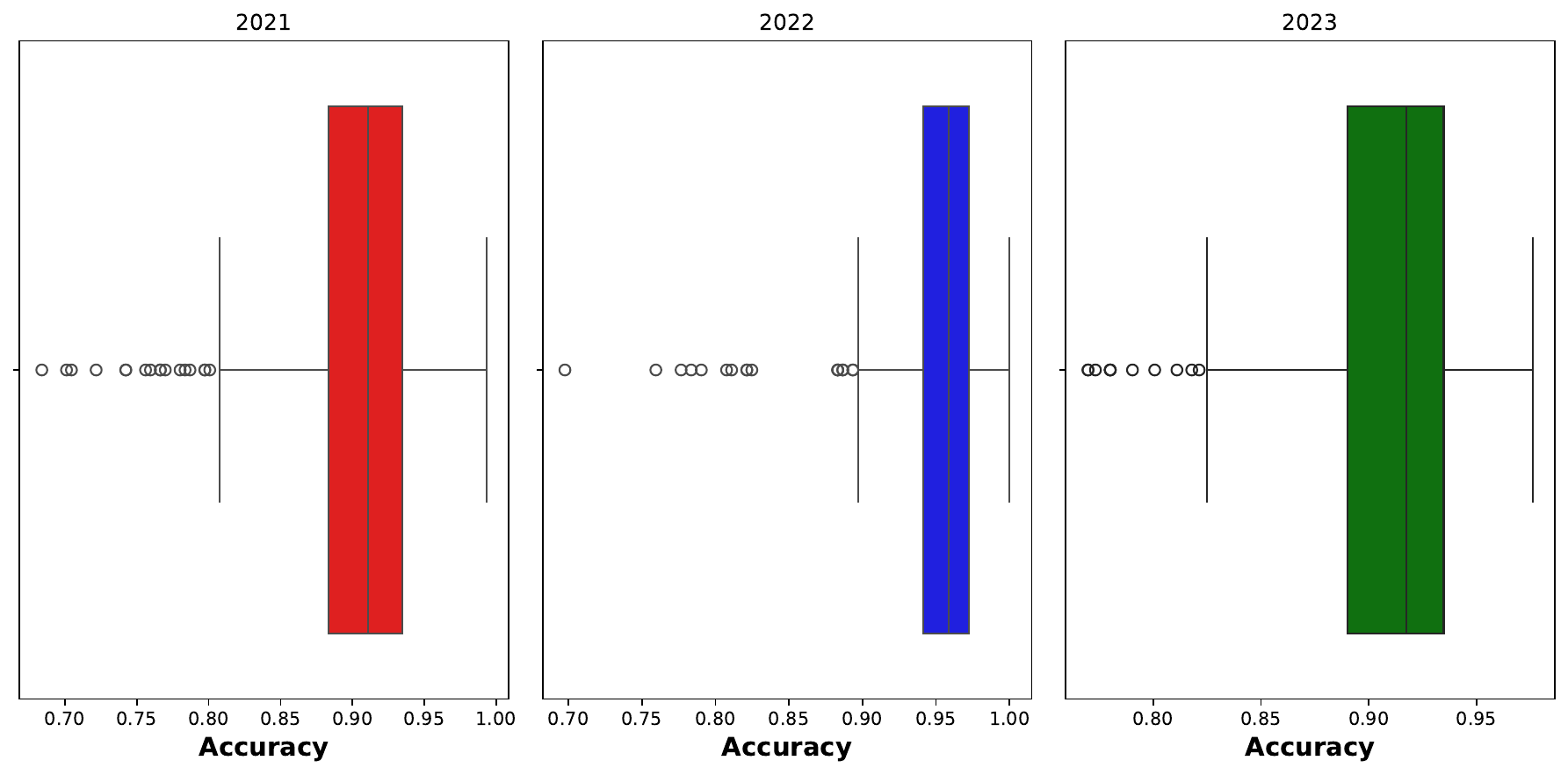}
        \caption{}
        \label{subfig:accuracy_box_plot}
    \end{subfigure}
    \begin{subfigure}{0.45\linewidth}
        \includegraphics[height = 5.5cm, width = 7cm]{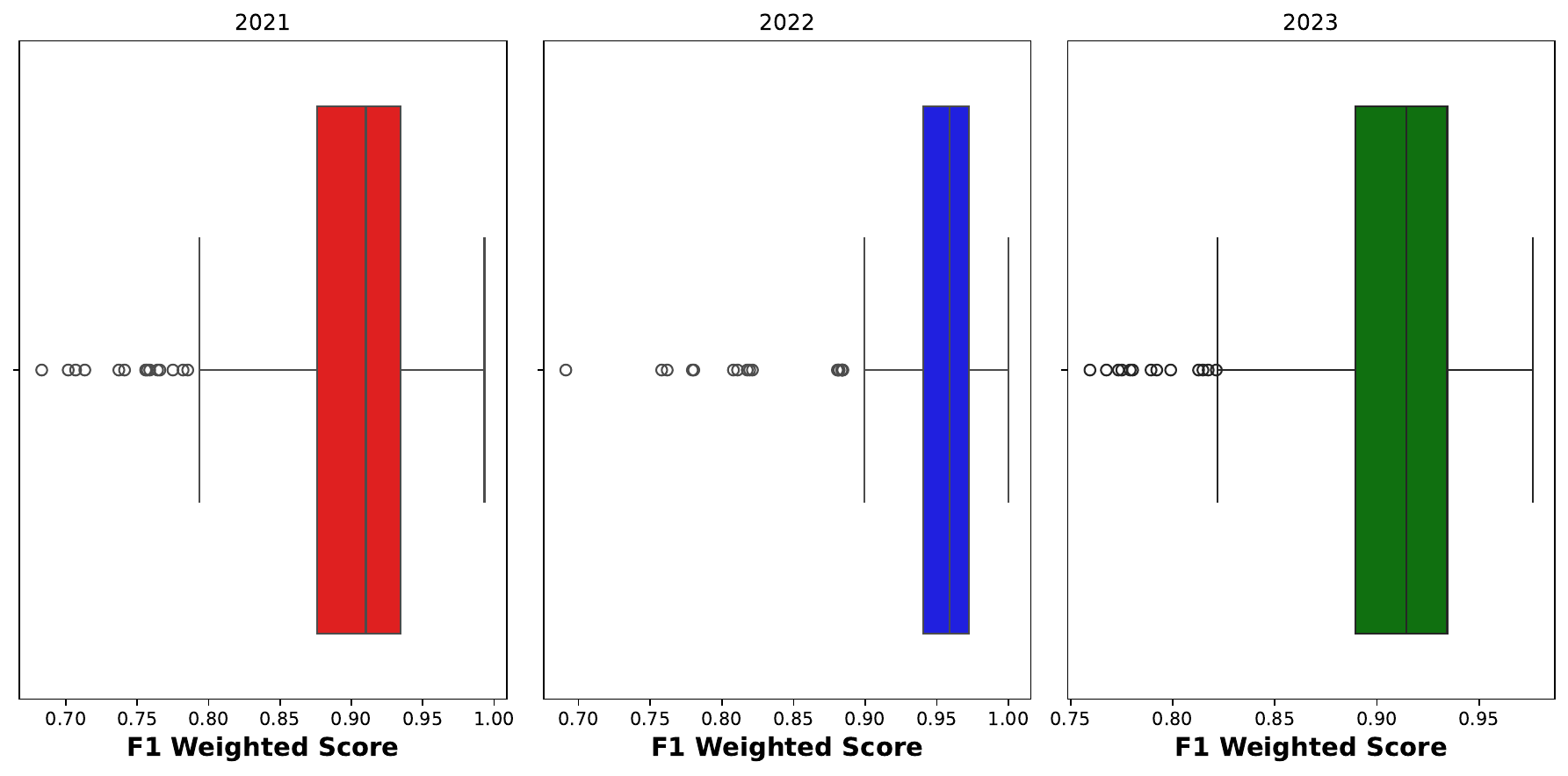}
        \caption{}
        \label{subfig:f1_Weighted_box_plot}
    \end{subfigure}
    \begin{subfigure}{0.5\linewidth}
        \includegraphics[height = 5.5cm, width = 7cm]{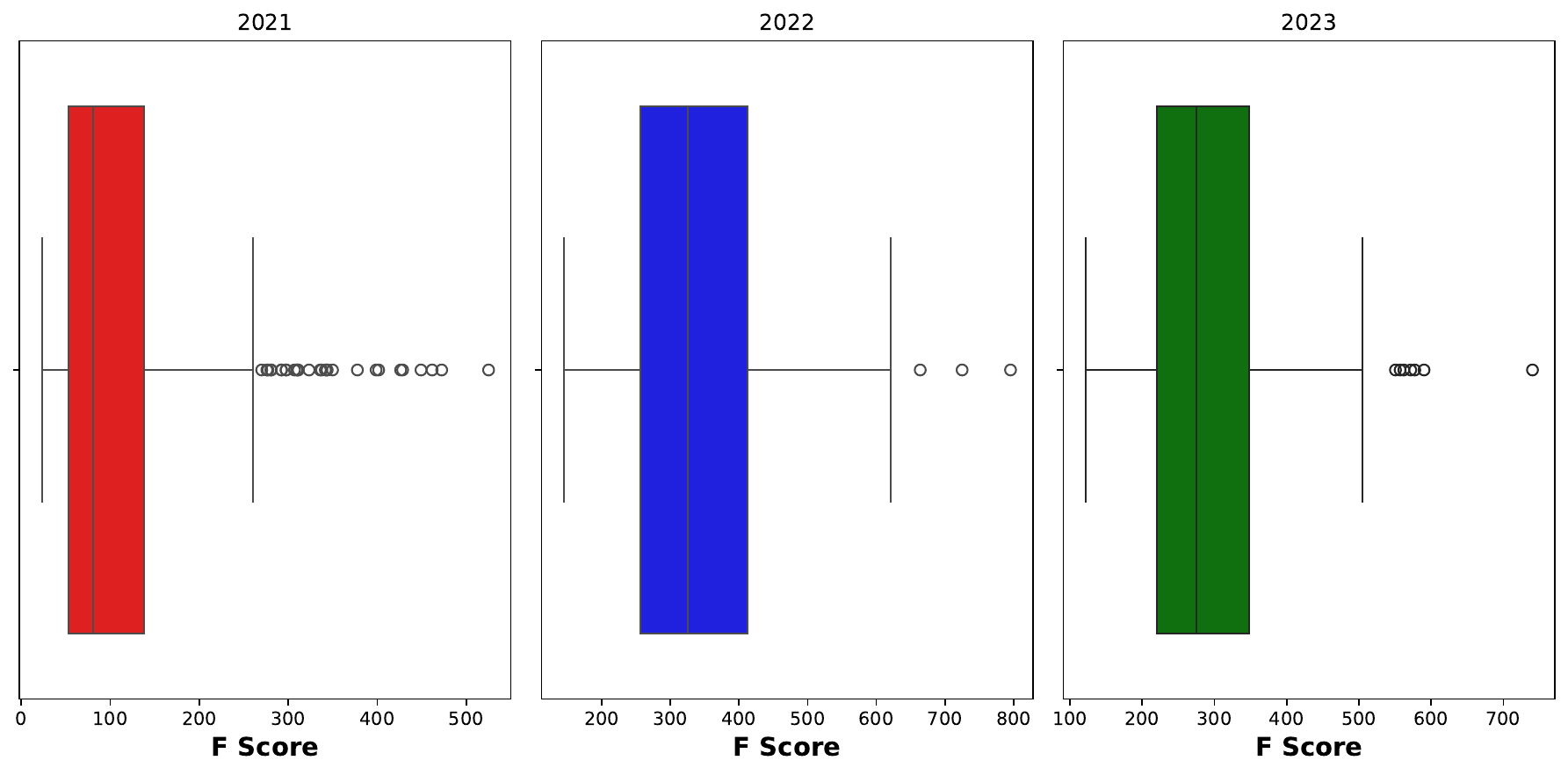}
        \caption{}
        \label{subfig:f_Score_box_plot}
    \end{subfigure}
    \begin{subfigure}{0.45\linewidth}
        \includegraphics[height = 5.5cm, width = 7cm]{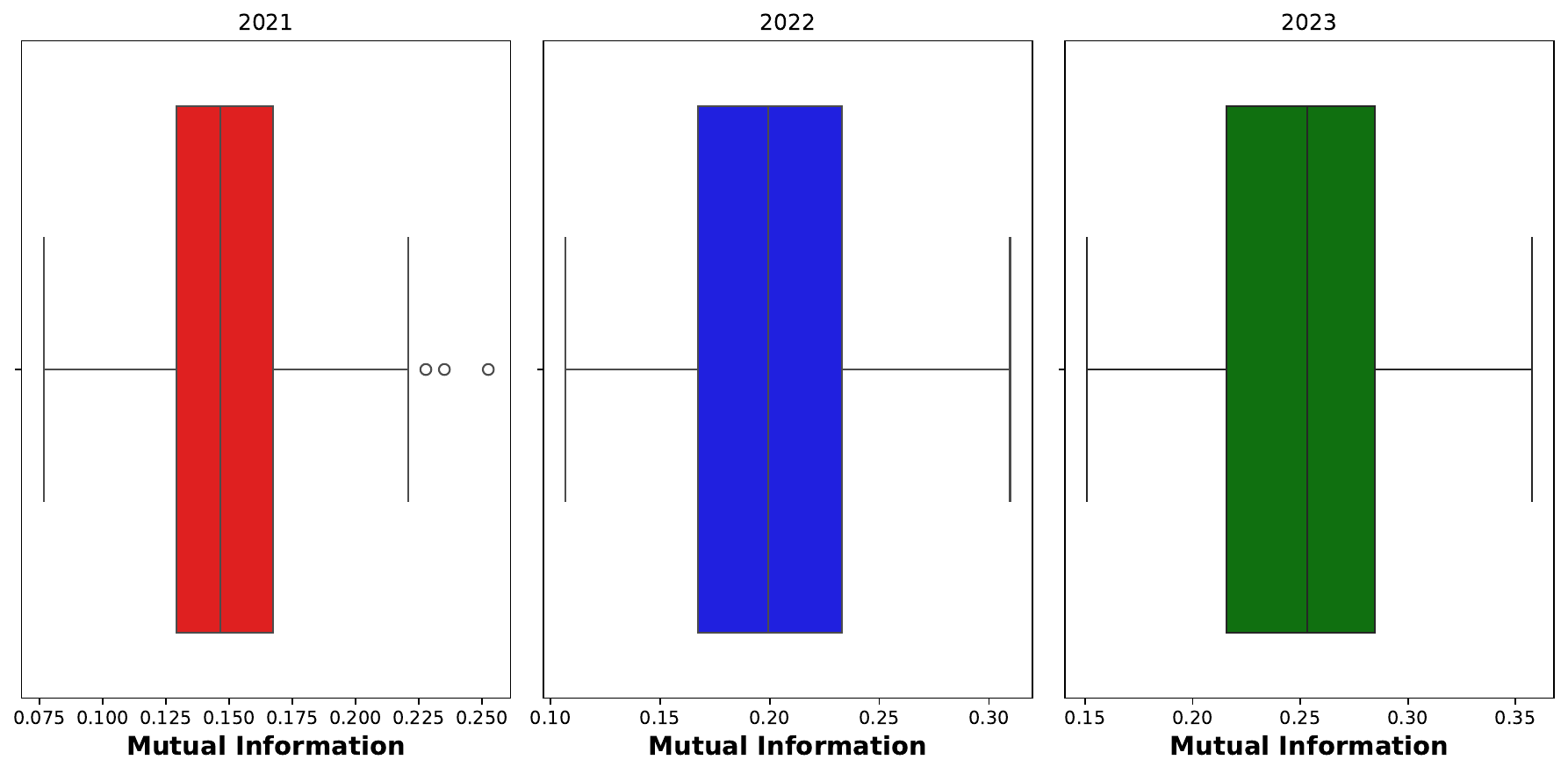}
        \caption{}
        \label{subfig:mutual_information_box_plot}
    \end{subfigure}
    \caption{\centering{Analysis of Assignment of Regimes to Training Input Data for Prediction of Prices of Year 2021,2022 and 2023  Via (a) Accuracy Across Dataset (b) Weighted Average F1 Score (c) ANOVA F-Statistic Mean and (d) Mutual Information}}
    \label{fig:input_regimes_assignment_analysis}
\end{figure}
From a statistical dependency perspective, the ANOVA F-statistic, shown in the Figure \ref{subfig:f_Score_box_plot} exhibits large values (mean $\approx 300$), suggesting strong linear separability between regimes along multiple input dimensions. In contrast, the mean mutual information (MI) per feature, shown in the Figure \ref{subfig:mutual_information_box_plot} is comparatively low (median $\approx 0.2$). This discrepancy is not contradictory but rather expected: mutual information is a marginal measure and does not capture interactions among features. That is, even if individual features exhibit weak dependence on regime labels, their joint distribution can still encode highly separable structure. This is a classical manifestation of distributed or synergistic representations, which are effectively exploited by multivariate models such as neural networks but not captured by univariate statistics. Therefore, the combination of low per-feature MI and high classification accuracy highlights that regime identity is embedded in complex, possibly nonlinear feature interactions—further justifying the use of expressive classifiers in our analysis.

\section{Neural Processes for Regime-Specific Forecasting}\label{sec:neural_process}
Once the latent market regimes are robustly identified by the DS-HDP-HMM, the next crucial step is to perform accurate and uncertainty-aware forecasting within each detected regime. Traditional forecasting models often struggle to adapt to the distinct statistical properties and functional relationships that characterize different regimes. To address this, we employ conditional neural processes (CNPs), a class of meta-learning models that learn a distribution over functions. This allows them to generalize across different tasks or, in our case, different market regimes, providing probabilistic predictions with inherent uncertainty quantification \cite{DBLP:journals/corr/abs-1807-01622}.

\subsection{Conditional Neural Processes}
While standard Neural Processes (NPs) provide a powerful framework for distributional meta-learning, Conditional Neural Processes (CNPs) represent a more specific instantiation in which the global latent variable is omitted, yielding a deterministic yet probabilistic model that directly learns the conditional distribution $\mathbb{P}(P_{\mathbf{x}_{t}}^* \mid \mathbf{x}_{t}^*, \mathcal{C})$ via learned function approximators. The CNP framework is particularly suitable for structured prediction in regime-specific financial forecasting due to its ability to condition explicitly on observed context data while preserving permutation invariance and scalability.

Let $\mathcal{C} = \{(\mathbf{x}_{t}, P_{\mathbf{x}_{t}})\}_{i=1}^{N_C}$ denote the context set and let $\mathcal{T} = \{\mathbf{x}_j^*\}_{j=1}^{N_T}$ be a set of target inputs for which predictions are required. The goal of a CNP is to model the conditional predictive distribution $\mathbb{P}(\{P_{\mathbf{x}_j}^*\}_{j=1}^{N_T} \mid \mathcal{T}, \mathcal{C})$.

The architecture of a Conditional Neural Process can be formalized as follows. Define a function $h_\theta : \mathcal{X} \times \mathcal{Y} \to \mathbb{R}^d$ parameterized by a neural network with weights $\theta$, which maps each input-output pair $(\mathbf{x}_i, P_{\mathbf{x}_{i}}) \in \mathcal{C}$ into a representation vector $r_i = h_\theta(\mathbf{x}_i, P_{\mathbf{x}_{i}})$. Permutation invariance across the context set is ensured via an aggregation function $\rho: (\mathbb{R}^d)^{N_C} \to \mathbb{R}^d$ such that the aggregated representation is $r = \rho(r_1, \ldots, r_{N_C})$. A standard choice for $\rho$ is the element-wise mean: $r = \frac{1}{N_C} \sum_{i=1}^{N_C} h_\theta(\mathbf{x}_i, P_{\mathbf{x}_{i}})$.

Once the global context representation $r$ is obtained, each target input $\mathbf{x}_{j}^* \in \mathcal{T}$ is paired with $r$ and passed through a decoder network $g_\phi : \mathcal{X} \times \mathbb{R}^d \to \mathcal{Y}$ parameterized by weights $\phi$, to produce the predictive distribution parameters. That is,
\[
(\mu_j^*, \sigma_j^{*2}) = g_\phi(x_j^*, r), \quad \text{for each } j \in \{1, \ldots, N_T\},
\]
which implies a predictive distribution of the form $\mathbb{P}(P_{\mathbf{x}_{j}}^* \mid x_j^*, \mathcal{C}) = \mathcal{N}(P_{\mathbf{x}_{j}}^* \mid \mu_j^*, \sigma_j^{*2})$. The conditional independence assumption is enforced across target points given the shared representation $r$.

The learning objective is to minimize the expected negative log-likelihood of the targets under the predictive distribution. For a dataset of tasks $\mathcal{D} = \{(\mathcal{C}_k, \mathcal{T}_k, \{P_{\mathbf{x}_{j}}^{*(k)}\})\}_{k=1}^{K}$ drawn from a distribution over tasks (e.g., regime-specific financial dynamics), the loss function is given by
\[
\mathcal{L}(\theta, \phi) = - \frac{1}{K} \sum_{k=1}^{K} \sum_{j=1}^{N_T^{(k)}} \log \mathbb{P}(P_{\mathbf{x}_{j}}^{*(k)} \mid \mathbf{x}_j^{*(k)}, \mathcal{C}_k),
\]
where each $\mathbb{P}(P_{\mathbf{x}_{j}}^{*(k)} \mid \mathbf{x}_j^{*(k)}, \mathcal{C}_k)$ is computed as a Gaussian likelihood with parameters given by $g_\phi(x_j^{*(k)}, r^{(k)})$ and $r^{(k)} = \rho\left( \{h_\theta(\mathbf{x}_i^{(k)}, P_{\mathbf{x_{i}}}^{(k)})\}_{i=1}^{N_C^{(k)}} \right)$.

This deterministic conditioning structure, as opposed to the latent variable modeling in full NPs or ANPs (Attentive Neural Processes), makes CNPs more efficient and interpretable in scenarios where the functional variation can be sufficiently captured by the context set alone. This is particularly advantageous in financial time series forecasting, where contextual data within a market regime can effectively guide the predictive behavior without requiring additional stochastic global priors.

Moreover, due to the Gaussian output assumption, the model inherently provides a quantification of predictive uncertainty through the variance $\sigma_j^{*2}$, which increases in extrapolative or low-context regions. This capability is essential for risk-sensitive financial decision-making, as it allows for probabilistic reasoning about future market states rather than relying on deterministic point forecasts.

Conditional neural processes offer a theoretically grounded, data-efficient, and uncertainty-aware approach to modeling conditional distributions over functions. Their capacity to encode permutation-invariant representations of context data and their computational tractability make them ideal candidates for regime-specific function approximation in non-stationary financial environments.

\subsection{Integration with DS-HDP-HMM}
Our framework integrates neural processes with the DS-HDP-HMM in a sequential manner. After the DS-HDP-HMM has identified the latent regimes from the historical electricity price data, we train a separate neural process for each detected regime.
\begin{itemize}
    \item \textbf{Regime-Specific Data Partitioning:} The historical electricity price time series is partitioned based on the inferred regime assignments from the DS-HDP-HMM. For each regime $r$, we collect all data points $(x_t, y_t)$ where the DS-HDP-HMM assigned $z_t = r$. This forms the regime-specific dataset $\mathcal{D}_r$.
    \item \textbf{Regime-Specific NP Training:} For each regime $r$, a dedicated neural process is trained exclusively on $\mathcal{D}_r$. During training, context sets $\mathcal{C}_r$ are sampled from $\mathcal{D}_r$, and the NP learns to infer the conditional distribution $\mathbb{P}(P_{\mathbf{x}}^* \mid \mathbf{x}^*, \mathcal{C}_r)$ that characterizes the price dynamics within that specific regime. This ensures that each NP specializes in the unique statistical properties and functional relationships of its assigned regime, avoiding the "mode averaging" problem encountered by global models.
    \item \textbf{Probabilistic Forecasting:} When forecasting future electricity prices, the DS-HDP-HMM first predicts the most probable regime for the upcoming period. Then, the corresponding regime-specific neural process is queried with relevant historical context points from that regime to generate a probabilistic forecast (e.g., a mean prediction and an associated variance) for the future price. This approach ensures that forecasts are tailored to the prevailing market conditions, providing more accurate and reliable predictions with quantified uncertainty.
\end{itemize}
This integration leverages the strengths of both models: the DS-HDP-HMM provides a flexible and data-driven mechanism for discovering the underlying market states, while the neural processes offer adaptive, probabilistic forecasting capabilities within each identified state, leading to a comprehensive and robust electricity price prediction framework.

\section{Probabilistic Aggregation Mechanism}\label{sec:aggreate_mechanism}

In a dynamic environment where multiple regimes can influence future electricity prices, a single regime-specific forecast, while valuable, may not capture the full spectrum of possibilities. To provide a more robust and comprehensive prediction, we employ a probabilistic aggregation mechanism that combines the forecasts from each regime-specific neural process, weighted by the likelihood of each regime occurring. This approach accounts for the inherent uncertainty in regime assignment and provides a more holistic predictive distribution. Let $K$ be the total number of distinct regimes identified by the DS-HDP-HMM. For a given future time step $t^*$, each regime-specific neural process $NP_k$ (trained on data from regime $k$) provides a conditional predictive distribution $\mathbb{P}(P_{\mathbf{x}^*} \mid \mathbf{x}_{t^*}, \mathcal{C}_k, \text{Regime } k)$, which we can approximate as a Gaussian distribution $\mathcal{N}(\mu_{t^*}^{(k)}, (\sigma_{t^*}^{(k)})^2)$, where $\mu_{t^*}^{(k)}$ is the mean forecast and $(\sigma_{t^*}^{(k)})^2$ is the predictive variance from $NP_k$ for time $t^*$, given context $\mathcal{C}_k$. The DS-HDP-HMM also provides the posterior probability of being in each regime $k$ at time $t^*$, denoted as $\mathbb{P}(\text{Regime } k \mid \text{Data})$. This probability can be derived from the transition probabilities and the observed data history. Let $w_k = \mathbb{P}(\text{Regime } k \mid \text{Data})$ be the weight associated with regime $k$, such that $\sum_{k=1}^K w_k = 1$. The compution of the weights is explained in the next sub-section.These weights effectively represent the confidence in each regime being active at time $t^*$. The aggregated predictive distribution for $y_{t^*}$ is then a mixture of Gaussians, given by:
$$\mathbb{P}(P_{\mathbf{x}^*} \mid \mathbf{x}_{t^*}, \text{Data}) = \sum_{k=1}^K w_k \mathcal{N}(\mu_{t^*}^{(k)}, (\sigma_{t^*}^{(k)})^2)$$
From this mixture distribution, we can derive the overall aggregated mean forecast $\hat{y}_{t^*}$ and the total predictive variance $(\hat{\sigma}_{t^*})^2$. The aggregated mean forecast is simply the weighted average of the regime-specific means:
$$\hat{P}_{\mathbf{x}^*} = \sum_{k=1}^K w_k \mu_{t^*}^{(k)}$$
The total predictive variance, incorporating both the within-regime uncertainty and the uncertainty across regimes, is given by the law of total variance:
$$(\hat{\sigma}_{t^*})^2 = \sum_{k=1}^K w_k \left( (\sigma_{t^*}^{(k)})^2 + (\mu_{t^*}^{(k)} - \hat{P}_{\mathbf{x}^*})^2 \right)$$
This formulation ensures that the final forecast not only provides a point estimate but also a comprehensive measure of uncertainty that accounts for both the intrinsic variability within each regime and the uncertainty associated with which regime will be active. This robust, uncertainty-aware prediction is critical for downstream applications such as risk management and optimal decision-making in electricity markets.

\subsection{Weight Computation}\label{ss:weight_comp}
%\subsection*{Soft Regime Association via Neural Classification}

To determine the degree to which a new observation $\mathbf{x}_{\text{new}} \in \mathbb{R}^d$ is associated with each of the $R$ learned regimes $\{r_1, \ldots, r_R\}$, we utilize a probabilistic classifier trained to map input features $\mathbf{x}_t$ to their corresponding regimes $z_t \in \{1, \ldots, R\}$. Specifically, we employ a multilayer perceptron (MLP) classifier trained on the regime-labeled training set $\{(\mathbf{x}_t, z_t)\}_{t=1}^T$.

Once trained, the MLP provides a posterior probability distribution over the regime labels for any input $\mathbf{x}_{\text{new}}$ via the \texttt{predict\_proba} function. Formally, for each regime $r$, we obtain:
\[
    w_r(\mathbf{x}_{\text{new}}) = \mathbb{P}(z = r \mid \mathbf{x}_{\text{new}}),
\]
where the weights $\{w_r\}_{r=1}^R$ satisfy $w_r \in [0,1]$ and $\sum_{r=1}^R w_r = 1$. These probabilities are directly interpretable as soft membership scores indicating the classifier's confidence in the assignment of $\mathbf{x}_{\text{new}}$ to each regime.

This approach is mathematically robust as it avoids comparing likelihoods from separately estimated generative models—which are not directly comparable due to differing normalization constants, data scales, and priors. Instead, the discriminative training of the MLP ensures that the posterior probabilities reflect meaningful regime distinctions learned from data. As such, this method provides a stable, interpretable, and computationally efficient mechanism for regime attribution that integrates seamlessly into a probabilistic decision-making pipeline.

\section{Operational Strategy with Battery and Multi-Criteria Optimization}\label{sec:operational_strategies}
Given the probabilistic forecasts of electricity prices, specifically the predicted price levels $\hat{P}_{\text{new}} \in \mathbb{R}^{24}$, and their associated uncertainty estimates, a crucial application involves using these predictions into real-world operational strategies. For a battery storage system operated over a planning horizon (e.g., the next 24 hours), the decision on charging/discharging schedules must balance several objectives. These objectives include maximizing forecasted price level for profit, minimizing risk by accounting for price volatility, ensuring operational feasibility within storage limits and efficiency, and considering market cost asymmetry related to over/under prediction penalties. The optimization problem involves the following variables and parameters: $\hat{p}_t$ represents the predicted electricity price at hour $t$, obtained from our regime-switching neural process framework. $\sigma_t$ is the uncertainty estimate (e.g., standard deviation) of the electricity price forecast at hour $t$, also derived from the neural process or from the benchmark models. $x_t \in \mathbb{R}$ denotes the net charging/discharging decision at hour $t$, where a positive $x_t$ indicates charging, and a negative $x_t$ indicates discharging. $x_t^+ \geq 0$ is the amount of electricity charged into the battery at hour $t$, while $x_t^- \geq 0$ is the amount of electricity discharged from the battery at hour $t$. $s_t$ refers to the state of charge (SoC) of the battery at the end of hour $t$. $\eta_c \in (0, 1]$ signifies the charging efficiency of the battery, representing the fraction of energy input successfully stored. $\eta_d \in (0, 1]$ is the discharging efficiency, representing the fraction of stored energy that can be retrieved. $C_{\text{max}}$ denotes the maximum storage capacity of the battery in energy units (e.g., MWh). $x_{\min}$ represents the minimum allowable net charging rate (e.g., maximum discharge rate), and $x_{\max}$ is the maximum allowable net charging rate. Finally, $\lambda \geq 0$ is a risk aversion parameter, which quantifies the decision-maker's trade-off between maximizing expected profit and minimizing exposure to forecast uncertainty; a higher $\lambda$ implies greater risk aversion. Below we show the application of predicted price for different trading strategies, i.e. we formulate different optimization problems with different objective functions and constraints representing different market conditions. 

\paragraph{Case I: Trading Strategy with Predicted Price Only}
The operational strategy focuses purely on maximizing economic profit through price arbitrage, is formulated as a linear programming problem. The objective function is defined in the optimization problem \ref{no_uncer_no_exogebous_var}  aims to maximize the difference between revenue from selling discharged energy and cost from buying charged energy, adjusted by battery efficiencies. The constraints ensure physical feasibility: the initial state of charge $s_1$ is fixed at $s_{\text{init}}$; the state of charge evolves according to $s_{t+1} = s_t + \eta_c x_t^+ - \frac{1}{\eta_d} x_t^-$ for all $t \in \{1, \dots, T\}$; the state of charge must remain between $0$ and $C_{\text{max}}$ for all $t \in \{1, \dots, T+1\}$; a minimum final state of charge, $s_{T+1} \ge s_{\text{final}}$, is enforced for battery health; individual charging and discharging power limits are respected, $0 \le x_t^+ \le P_{\text{charge,max}}$ and $0 \le x_t^- \le P_{\text{discharge,max}}$; and finally, the net power exchanged with the grid must comply with $x_{\min} \le x_t^- - x_t^+ \le x_{\max}$ for all $t$. This model dictates the optimal hourly charging and discharging schedule to capitalize on predicted price differentials, thereby demonstrating the direct economic value of accurate price forecasts for maximizing BESS profitability.

\begin{align}\label{no_uncer_no_exogebous_var}
\max_{\{x_t^+, x_t^-, s_t\}_{t=1}^{24}} \quad & \sum_{t=1}^{24} \left( \hat{p}_t x_t^-\eta_{d} - \hat{p}_t \frac{x_t^+}{\eta_{c}} \right) \\
\text{subject to} \quad & s_1 = s_{\text{init}}, \nonumber\\
& s_{t+1} = s_t + \eta_c x_t^+ - \frac{1}{\eta_d} x_t^-, \quad \forall t = 1, \dots, 24, \nonumber\\
& 0 \leq s_t \leq C_{\text{max}}, \quad \forall t=1, \dots, 25,\nonumber \\
& s_{25} \geq s_{\text{final}}, \nonumber\\
& 0 \leq x_t^+ \leq P_{\text{charge,max}}, \quad \forall t, \nonumber\\
& 0 \leq x_t^- \leq P_{\text{discharge,max}}, \quad \forall t, \nonumber\\
& x_{\min} \leq x_t^- - x_t^+ \leq x_{\max}, \quad \forall t. \nonumber
\end{align}
(Note: The constraint $x_{\min} \leq x_t^- - x_t^+ \leq x_{\max}$ implies that the net power injected into the grid (discharge minus charge) must be within the defined limits. If $x_{\min}$ is negative, it indicates maximum power drawn from the grid. If $x_{\max}$ is positive, it indicates maximum power injected to the grid.)
\hfill\newline
The corresponding perfect foresight model calculates the absolute maximum profit achievable by knowing the true future electricity prices $p_t^{\text{real}}$. For the purpose of comparison, keeping the constraints same, the objective function for perfect foresight  is as follows: the  It shares identical physical and operational constraints with the model-based strategy.

\begin{equation}\label{no_uncer_no_exogebous_var_perfect_foresight}
\max_{\{x_t^+, x_t^-, s_t\}_{t=1}^{24}} \quad  \sum_{t=1}^{24} \left( p_t^{\text{real}} x_t^-\eta_{d} - p_t^{\text{real}} \frac{x_t^+}{\eta_{c}} \right)
\end{equation}

When evaluating the performance of the model for Case I, the actual profit obtained must be calculated using the optimal decisions $(x_t^{*+}, x_t^{*-})$ derived from solving \eqref{no_uncer_no_exogebous_var} and substituting them into the objective function of \eqref{no_uncer_no_exogebous_var_perfect_foresight} (i.e., using $p_t^{\text{real}}$). This actual profit will always be less than or equal to the profit obtained from solving \eqref{no_uncer_no_exogebous_var_perfect_foresight} (the Perfect Foresight profit).

This formulation serves as an upper bound on achievable profit, representing the ideal scenario where future electricity prices are known with absolute certainty. This perfect foresight problem is identical in its structure to the model-based strategy, but critically, it substitutes the predicted electricity prices $\hat{p}_t$ with the true realized (actual) prices $p_t^{\text{real}}$ in its objective function. Formally, the objective becomes $\max_{\{x_t^+, x_t^-, s_t\}_{t=1}^{T}} \sum_{t=1}^{T} \left( p_t^{\text{real}} x_t^-\eta_{d} - p_t^{\text{real}} \frac{x_t^+}{\eta_{c}} \right)$. All operational and physical constraints remain unchanged, including the battery dynamics, SoC limits, power limits, and net power flow constraints. The solution to this perfect foresight problem yields the maximum theoretical profit obtainable under the given physical system and market rules. By comparing the profit achieved by the model-based strategy (using its optimal decisions but evaluated against $p_t^{\text{real}}$) with this perfect foresight profit, the economic value of the price forecasting model, i.e., its ability to approximate optimal decisions under perfect information, can be precisely quantified.

\paragraph{Case II: Trading Strategy with predicted price and the uncertainty associated with the predicted price}

Case II extends the price arbitrage strategy by explicitly incorporating a risk aversion mechanism that penalizes trading actions based on the uncertainty of price forecasts. This strategy, implemented as a linear programming problem, aims to balance profit maximization with a reduction in exposure during periods of high price volatility. The additional terms, $-\lambda_1 \sigma_t x_t^-$ and $-\lambda_2 \sigma_t x_t^+$, in the objective function in problem \ref{with_uncer_no_exogebous_var} introduce a cost for both discharging and charging activities that is proportional to the forecast uncertainty and the volume of energy traded. This encourages the battery to adopt less aggressive trading patterns when price predictions are less reliable. All physical and operational constraints, including battery dynamics, SoC limits, power limits, and net power flow limits, remain identical to Case I, ensuring consistency in the battery's physical operation while its economic behavior is modulated by risk.

\begin{align}\label{with_uncer_no_exogebous_var}
\max_{\{x_t^+, x_t^-, s_t\}_{t=1}^{24}} \quad & \sum_{t=1}^{24} \left( \hat{p}_t x_t^-\eta_{d} - \hat{p}_t \frac{x_t^+}{\eta_{c}} - \lambda_1 \sigma_t x_t^- - \lambda_2 \sigma_t x_t^+ \right) \\
\text{subject to} \quad & s_1 = s_{\text{init}}, \nonumber\\
& s_{t+1} = s_t + \eta_c x_t^+ - \frac{1}{\eta_d} x_t^-, \quad \forall t = 1, \dots, 24, \nonumber\\
& 0 \leq s_t \leq C_{\text{max}}, \quad \forall t=1, \dots, 25,\nonumber \\
& s_{25} \geq s_{\text{final}}, \nonumber\\
& 0 \leq x_t^+ \leq P_{\text{charge,max}}, \quad \forall t, \nonumber\\
& 0 \leq x_t^- \leq P_{\text{discharge,max}}, \quad \forall t, \nonumber\\
& x_{\min} \leq x_t^- - x_t^+ \leq x_{\max}, \quad \forall t. \nonumber
\end{align}
The penalty term $-\lambda_1 \sigma_t x_t^- - \lambda_2 \sigma_t x_t^+$ explicitly penalizes higher volumes of selling ($x_t^-$) and buying ($x_t^+$) when forecast uncertainty $\sigma_t$ is high, adjusted by risk aversion parameters $\lambda_1$ and $\lambda_2$. This encourages less aggressive trading in uncertain periods.

\hfill\newline
The corresponding foresight for Case II is identical to that of Case I, as the system's inherent maximum profit capability is independent of our predictive uncertainty or risk aversion with th objective function as follows:

\begin{equation}\label{with_uncer_no_exogebous_var_perfect_foresight}
\max_{\{x_t^+, x_t^-, s_t\}_{t=1}^{24}} \quad  \sum_{t=1}^{24} \left( p_t^{\text{real}} x_t^-\eta_{d} - p_t^{\text{real}} \frac{x_t^+}{\eta_{c}} \right)
\end{equation}
When evaluating the performance of the model for Case II, the actual profit must be calculated using the optimal decisions $(x_t^{*+}, x_t^{*-})$ derived from solving \eqref{with_uncer_no_exogebous_var} and substituting them into the objective function of \eqref{with_uncer_no_exogebous_var_perfect_foresight} (i.e., using $p_t^{\text{real}}$ and {omitting the penalty terms}). This actual profit will always be less than or equal to the Perfect Foresight profit.

This foresight strategy for Case II, in contrast to its model-based counterpart, fundamentally disregards any notion of forecast uncertainty or risk aversion, as it operates under the assumption of complete knowledge of future realized prices. Therefore, its objective function is precisely the same as that for Case I's perfect foresight: $\max_{\{x_t^+, x_t^-, s_t\}_{t=1}^{T}} \sum_{t=1}^{T} \left( p_t^{\text{real}} x_t^-\eta_{d} - p_t^{\text{real}} \frac{x_t^+}{\eta_{c}} \right)$. The terms involving $\sigma_t$ are entirely omitted from the perfect foresight objective, as risk aversion is a policy choice made by an agent facing uncertainty, not an inherent property of the maximal profit achievable with perfect information. The constraints, including battery dynamics, SoC bounds, and power limits, remain unchanged. This perfect foresight formulation provides the absolute best achievable profit benchmark, allowing for an evaluation of how much profit is sacrificed by the model-based strategy due to its consideration of (and aversion to) forecast uncertainty, thus highlighting the trade-off between risk mitigation and maximal profit.

\paragraph{Case III: Trading Strategy with predicted price, the uncertainty associated with the predicted price, the solar production and the load}

Case III presents a more sophisticated operational strategy where the battery not only seeks profit from price arbitrage and manages risk from price uncertainty, but also incorporates objectives related to grid support and renewable energy integration. In the objective function of problem \ref{with_uncer_with_exogebous_var}, $\hat{r}_t \in R$ represents the simulated hourly load and $\hat{g}_t \in R_{\ge 0}$ represents the simulated solar generation at hour $t$. For the simulation of load and solar energy production we have used the Python Package PVLIB \cite{article_pvlib}. The parameter $\lambda \ge 0$ is a unified risk aversion coefficient for price uncertainty, while $\mu, \gamma \ge 0$ are weighting parameters for the grid support policies. Specifically, the term $-\mu \hat{r}_t x_t^+$ penalizes charging during periods of high predicted residual load, aiming to reduce strain on conventional generators, while $+\gamma \hat{g}_t x_t^+$ incentivizes charging when renewable generation is abundant, promoting better utilization of clean energy. All physical and operational constraints on battery dynamics, state of charge, power limits, and net power flow are identical to the previous cases, ensuring realistic operation within the combined objective of profit, risk, and grid services.

\begin{align}\label{with_uncer_with_exogebous_var}
\max_{\{x_t^+, x_t^-, s_t\}_{t=1}^{24}} \quad & \sum_{t=1}^{24} \left[ \hat{p}_t x_t^- \eta_d - \hat{p}_t \frac{x_t^+}{\eta_c} - \lambda \sigma_t (x_t^+ + x_t^-) - \mu \hat{r}_t x_t^+ + \gamma \hat{g}_t x_t^+ \right] \\
\text{subject to} \quad & s_1 = s_{\text{init}}, \nonumber\\
& s_{t+1} = s_t + \eta_c x_t^+ - \frac{1}{\eta_d} x_t^-, \quad \forall t = 1, \dots, 24, \nonumber\\
& 0 \leq s_t \leq C_{\text{max}}, \quad \forall t=1, \dots, 25,\nonumber \\
& s_{25} \geq s_{\text{final}}, \nonumber\\
& 0 \leq x_t^+ \leq P_{\text{charge,max}}, \quad \forall t, \nonumber\\
& 0 \leq x_t^- \leq P_{\text{discharge,max}}, \quad \forall t, \nonumber\\
& x_{\min} \leq x_t^- - x_t^+ \leq x_{\max}, \quad \forall t. \nonumber
\end{align}
The term $-\mu \hat{r}_t x_t^+$ discourages charging during periods of high forecasted residual load $\hat{r}_t$, aiming to reduce grid stress. The term $+\gamma \hat{g}_t x_t^+$ incentivizes charging during periods of high forecasted renewable availability $\hat{g}_t$, supporting renewable integration. These are additional objectives influencing the battery's operational policy.
\hfill\newline
The corresponding perfect foresight model must account for the actual grid conditions ($r_t^{\text{real}}$, $g_t^{\text{real}}$) when calculating the maximum achievable benefit from grid-supportive actions. Constraints remaining same as the \ref{with_uncer_with_exogebous_var}, the objective function is defined as follows:

\begin{equation}\label{with_uncer_with_exogebous_var_pefect_foresight}
\max_{\{x_t^+, x_t^-, s_t\}_{t=1}^{24}} \quad  \sum_{t=1}^{24} \left[ p_{t}^{\text{real}} x_t^- \eta_d - p_{t}^{\text{real}} \frac{x_t^+}{\eta_c} - \mu r_t^{\text{real}} x_t^+ + \gamma g_t^{\text{real}} x_t^+ \right]
\end{equation}
When evaluating the performance of the model for Case III, the actual profit must be calculated using the optimal decisions $(x_t^{*+}, x_t^{*-})$ from solving \eqref{with_uncer_with_exogebous_var}, but using $p_t^{\text{real}}$, $r_t^{\text{real}}$, $g_t^{\text{real}}$ and {omitting the uncertainty penalty terms}. This actual profit will always be less than or equal to the Perfect Foresight profit obtained from solving \eqref{with_uncer_with_exogebous_var_pefect_foresight}.

The perfect foresight strategy for Case III, as a benchmark, reveals the maximum theoretical benefit achievable if all relevant market and grid conditions were known precisely. Its objective function utilizes the actual realized electricity prices $p_t^{\text{real}}$, realized residual load $r_t^{\text{real}}$, and realized renewable availability $g_t^{\text{real}}$. Crucially, the terms related to forecast uncertainty ($\lambda \sigma_t (x_t^+ + x_t^-)$) are absent, as perfect foresight implies no uncertainty to mitigate. The remaining terms are re-evaluated with their realized values, reflecting the true impact of the battery's actions on economic profit and grid stability if perfect information were available. All physical and operational constraints remain consistent with the model-based problem. By comparing the actual profit (calculated using realized data and omitting uncertainty penalties) from the model's decisions against this perfect foresight benchmark, the economic and operational value of the forecasting system and policy choices in integrating grid support objectives can be accurately assessed.

\paragraph{Case IV:Trading Strategy with only predicted price and fixed total load}

Case IV fundamentally shifts the objective from profit maximization to cost minimization, specifically for a battery system serving a predetermined, fixed hourly electricity demand, denoted by $L_t \in R_{\ge 0}$, for each hour $t$. The decision variables $x_t^{\text{gb}} \in \mathbb{R}_{\ge 0}$ represent energy purchased directly from the grid to meet demand, $x_t^{\text{bl}} \in \mathbb{R}_{\ge 0}$ is energy discharged from the battery to serve the load, and $x_t^{\text{bg}} \in \mathbb{R}_{\ge 0}$ is energy purchased from the grid to charge the battery. The cost in the objective function of the problem \ref{with_uncer_no_exogebous_var_fix_load} terms include direct grid purchases for load, cost of charging the battery from the grid, and a penalty for total transaction volume (buying for load, charging, discharging to load) weighted by price uncertainty $\sigma_t$ and risk aversion $\lambda$. The core new constraint is the load satisfaction constraint: $x_t^{\text{grid\_buy}} + x_t^{\text{bl}}\eta_d = L_t$ for all $t$, ensuring that the hourly demand is met either by direct grid purchase or by discharging the battery (adjusted for efficiency). Battery dynamics ($s_{t+1} = s_t + \eta_c x_t^{\text{bg}} - \frac{1}{\eta_d} x_t^{\text{bl}}$), SoC limits ($0 \le s_t \le C_{\text{max}}$), initial and final SoC requirements ($s_1 = s_{\text{init}}$, $s_{T+1} \ge s_{\text{final}}$), and individual battery power limits ($0 \le x_t^{\text{bg}} \le P_{\text{charge,max}}$, $0 \le x_t^{\text{bl}} \le P_{\text{discharge,max}}$) are also enforced. This formulation determines the most cost-effective way to utilize the battery to manage electricity purchases for a given demand profile under predicted prices and uncertainty.

This formulation assumes a known hourly load profile $L_t$ that sums to $L_{\text{total}}$.

\begin{align}\label{with_uncer_no_exogebous_var_fix_load}
\min_{\{x_t^{\text{gb}}, x_t^{\text{bl}}, x_t^{\text{bg}}, s_t\}_{t=1}^{24}} \quad & \sum_{t=1}^{24} \left[ \hat{p}_t x_t^{\text{gb}} + \hat{p}_t x_t^{\text{bg}} + \lambda \sigma_t (x_t^{\text{gb}} + x_t^{\text{bg}} + x_t^{\text{bl}}) \right] \\
\text{subject to} \quad & s_1 = s_{\text{init}}, \nonumber\\
& s_{t+1} = s_t + \eta_c x_t^{\text{bg}} - \frac{1}{\eta_d} x_t^{\text{bl}}, \quad \forall t = 1, \dots, 24, \nonumber\\
& 0 \leq s_t \leq C_{\text{max}}, \quad \forall t=1, \dots, 25,\nonumber \\
& s_{25} \geq s_{\text{final}}, \nonumber\\
& x_t^{\text{gb}} \geq 0, \quad \forall t, \nonumber\\
& 0 \leq x_t^{\text{bg}} \leq P_{\text{charge,max}}, \quad \forall t, \nonumber\\
& 0 \leq x_t^{\text{bl}} \leq P_{\text{discharge,max}}, \quad \forall t, \nonumber\\
& x_t^{\text{gb}} + x_t^{\text{bl}}\eta_d = L_t, \quad \forall t. \nonumber
\end{align}
Here, $L_t$ represents the fixed hourly demand for the household. The term $\hat{p}_t x_t^{\text{gb}}$ is the cost of buying electricity directly from the grid to meet demand. $\hat{p}_t x_t^{\text{bl}}$ is the cost of buying electricity to charge the battery. The penalty term for uncertainty discourages high transaction volumes (buying for load, charging battery, discharging battery to load) when price forecasts are uncertain.

The perfect foresight model calculates the absolute minimum cost to serve the fixed hourly load $L_t$ using actual electricity prices $p_t^{\text{real}}$ . Under the same constraints, the perfect foresight has the following objective function:

\begin{equation}\label{with_uncer_no_exogebous_var_fix_load_perfect_foresight}
\min_{\{x_t^{\text{gb}}, x_t^{\text{bl}}, x_t^{\text{bg}}, s_t\}_{t=1}^{24}} \quad  \sum_{t=1}^{24} \left[ p_t^{\text{real}} x_t^{\text{gb}} + p_t^{\text{real}} x_t^{\text{bg}} \right]
\end{equation}
When evaluating the performance of the model for Case IV, the actual cost incurred must be calculated using the optimal decisions from solving \eqref{with_uncer_no_exogebous_var_fix_load} and substituting them into the objective function of \eqref{with_uncer_no_exogebous_var_fix_load_perfect_foresight} (i.e., using $p_t^{\text{real}}$ and {omitting the uncertainty penalty term}). This actual cost will always be greater than or equal to the Perfect Foresight cost obtained from solving \eqref{with_uncer_no_exogebous_var_fix_load_perfect_foresight}.

The perfect foresight strategy for Case IV serves as the theoretical minimum cost achievable for serving the fixed hourly load $L_t$, assuming perfect knowledge of future electricity prices. The uncertainty penalty terms are omitted, as there is no uncertainty in prices to hedge against in this ideal scenario. The constraints are identical to the model-based problem: $s_1 = s_{\text{init}}$, the battery dynamics $s_{t+1} = s_t + \eta_c x_t^{\text{bg}} - \frac{1}{\eta_d} x_t^{\text{bl}}$, SoC bounds $0 \le s_t \le C_{\text{max}}$ for all $t$, the terminal SoC requirement $s_{T+1} \ge s_{\text{final}}$, battery power limits $0 \le x_t^{\text{bg}} \le P_{\text{charge,max}}$ and $0 \le x_t^{\text{bl}} \le P_{\text{discharge,max}}$, and the load satisfaction constraint $x_t^{\text{gb}} + x_t^{\text{bl}}\eta_d = L_t$. By comparing the actual cost incurred by the model-based strategy (using its optimal decisions but evaluated against $p_t^{\text{real}}$ and without the penalty terms) against this perfect foresight minimum cost, the efficacy of the price forecasting system in reducing operational costs for fixed-load scenarios can be quantifiably assessed.

The numerical evaluation of all aforementioned is explained in the next section.

\section{Numerical Simulation and Empirical Validation}\label{sec:numerical_result}

\subsection{Case Study Setup and Calibration Issues}

To rigorously evaluate the performance of our proposed framework, we conduct a series of comprehensive numerical simulations and empirical validations on real-world electricity market data. We train the conditional neural process for each regime using the python based PYTORCH library. For the prediction of one day in future 4 years of past data is used. Each training input and output are of dimension $248 \times 1$ and $24 \times 1$ which results to the training input of size $1460 \times 248$ and $1460 \times24$. However, the training data is passed through the regime detection which segregates the training data into $r$ number of so called sub-training data where $r$ is the number of regimes detected. Then the conditional neural process models are trained individually using $r$ sub-training data. The prediction for new input is obtained using all $r$ trained models and finally the prediction via each trained model is linearly comined using the weights which are computed using the method mentioned in the Section \ref{ss:weight_comp}. The parameters of the conditional neural process is not fixed a priori but there are tuned using the hyperopt library based in python \cite{pmlr-v28-bergstra13}. 
The evaluation methodology is structured into several steps to assess different facets of the model's capabilities:
The first step involves a quantitative assessment of the overall electricity price prediction accuracy of our integrated RS-NP model. We employ standard forecasting error metrics, including the root mean squared error (RMSE), mean absolute error (MAE), and mean absolute percentage error (MAPE), to compare our model's point forecasts against those of several established benchmark models. These benchmarks include deep neural network (DNN) model and lasso-estimated autoregressive (LEAR) model.
The second and crucial step evaluates the quality of the uncertainty quantification provided by our RS-NP framework. Since neural processes inherently provide probabilistic forecasts (e.g., predictive distributions), we assess the calibration and sharpness of these distributions. We use the coverage probability to check if true values fall within predicted confidence intervals at the expected rate, and continuous ranked probability score (CRPS). Furthermore, we integrate the predicted prices and their associated uncertainty estimates into the proposed convex optimization problem for battery storage operation. We simulate the performance of the battery storage system using actual realized prices and compare the profitability and risk exposure of strategies derived from our RS-NP forecasts against those derived from benchmark models that may not provide robust uncertainty estimates. This allows us to demonstrate how the enhanced risk quantification capabilities of our model translate into tangible benefits for financial decision-making and risk management in the electricity market. We will analyze the trade-off between expected profit and risk (controlled by $\lambda$) to illustrate the flexibility of our optimization framework.
\begin{figure}[ht!]
    \centering
    \begin{subfigure}{0.4\linewidth}
        \includegraphics[height = 5.5cm, width = 7cm]{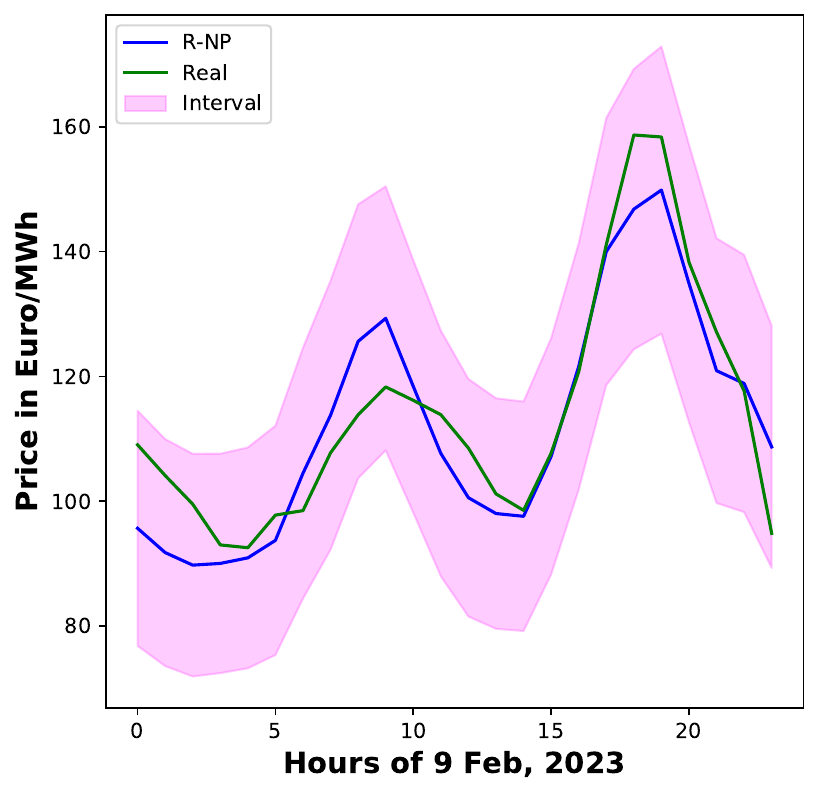}
        \caption{}
        \label{subfig:RNP_19_Feb}
    \end{subfigure}
    \begin{subfigure}{0.9\linewidth}
        \includegraphics[height = 5.5cm, width = 14cm]{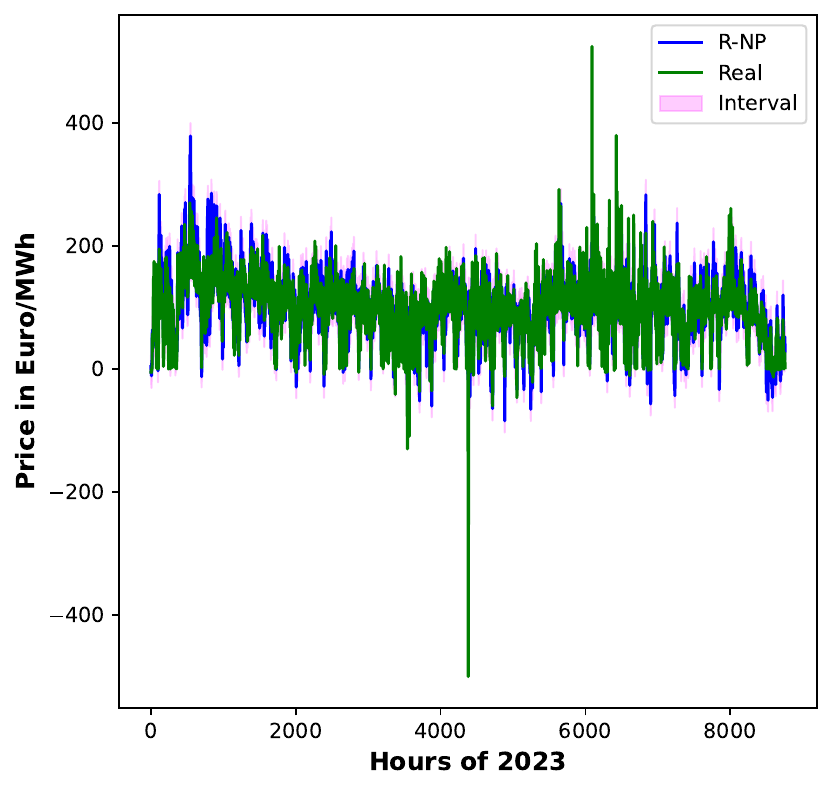}
        \caption{}
        \label{subfig:RNP_2023}
    \end{subfigure}
    \caption{Price Prediction via Regime-Aware Conditional Neural Process (a) For 9 Feb, 2023 and (b) For whole year 2023}
    \label{fig:RNP_2023_Interval}
\end{figure}
To further investigate the practical implications of our forecasting framework on operational strategies, we compare trading strategies under two price observation setups for a full week. The setup considers electricity prices for all 24 hours of the day across the entire week, representing a continuous trading environment. In this scenario, the optimization problem for the battery storage system will be solved over a 24-hour horizon, with $t \in \{1, \ldots, 24\}$, utilizing the full range of predicted prices and their associated uncertainties. 
The simulations are conducted using historical data from the German electricity market, covering multiple years to capture diverse market conditions, including periods of high renewable penetration, price spikes, and regulatory changes. This rigorous validation process provides compelling evidence of the superior predictive performance and improved risk quantification capabilities of our RS-NP framework, offering valuable insights for market participants, policymakers, and investors.

\subsection{Weight distributions to each Regime}

\begin{figure}[ht!]
    \centering
    \begin{subfigure}{0.45\linewidth}
        \includegraphics[height = 4.5cm, width = 7cm]{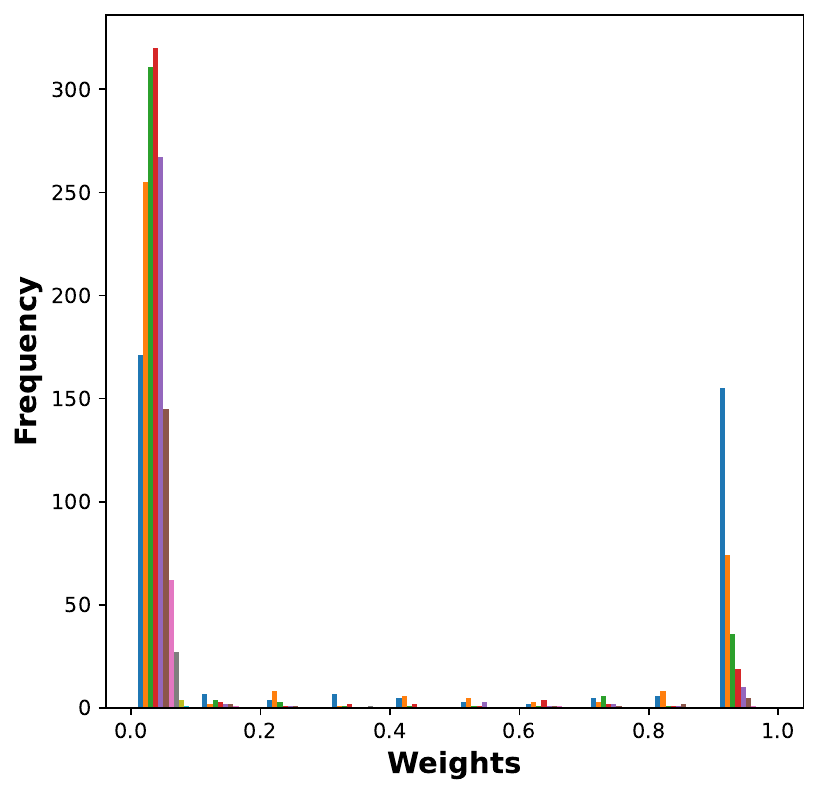}
        \caption{}
        \label{subfig:Weight_hist_2021}
    \end{subfigure}
    \begin{subfigure}{0.4\linewidth}
        \includegraphics[height = 4.5cm, width = 7cm]{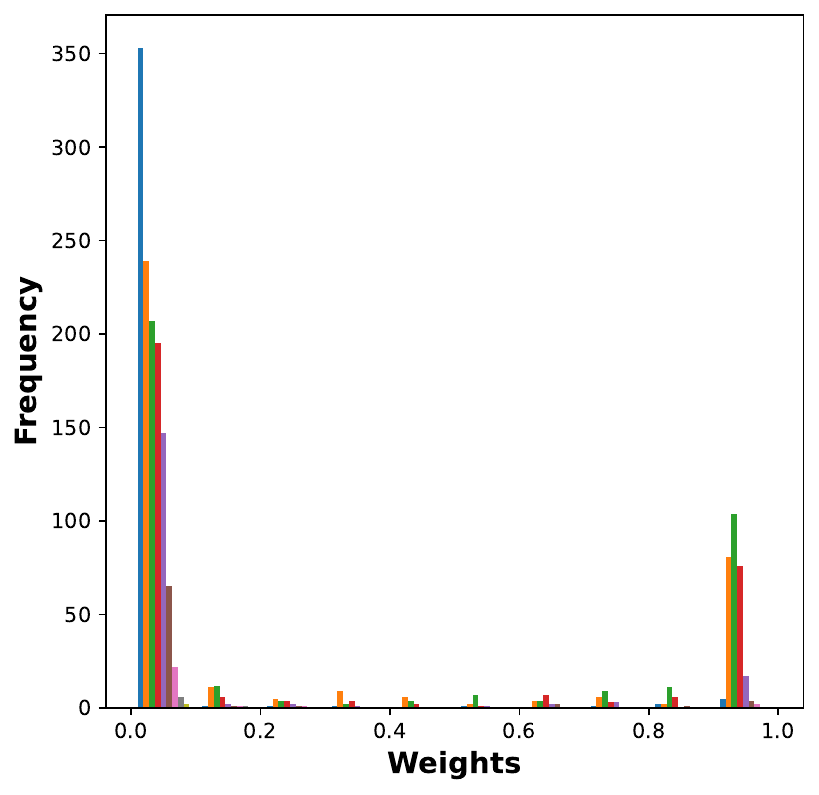}
        \caption{}
        \label{subfig:Weight_hist_2022}
    \end{subfigure}
    \begin{subfigure}{0.45\linewidth}
        \includegraphics[height = 4.5cm, width = 7cm]{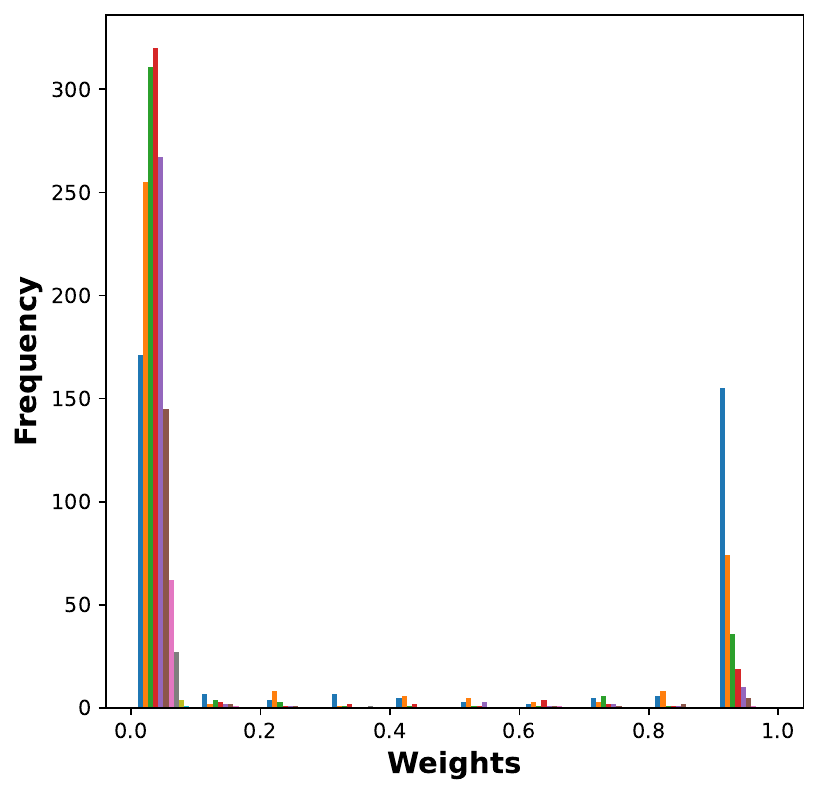}
        \caption{}
        \label{subfig:Weight_hist_2023}
    \end{subfigure}
    \caption{\centering{Weight Associated to the Different regimes Detected for the Prediction of Year (a) 2021, (b) 2022 and (c) 2023}}
    \label{fig:weight_distribution}
\end{figure}

%\subsection*{Distribution of Regime Association Weights Across Years}

To analyze how the newly constructed input features $\mathbf{x}_t$ for each prediction date in years 2021, 2022, and 2023 align with previously discovered regimes, we compute the regime association weights using the trained MLP classifier. For each prediction day $t \in \{1, 2, \dots, 365\}$ within a given year $Y$, a historical training set of size 1453 is used to determine the soft regime association weights $W_Y^{(t)} = [w_{r_1}^{(t)}, \ldots, w_{r_{n_t}}^{(t)}]$, where $w_{r_i}^{(t)}$ denotes the posterior probability that the input $\mathbf{x}_t$ belongs to regime $r_i$ among $n_t$ regimes inferred from the corresponding training window.

Figure~\ref{fig:weight_distribution} displays the histogram of all such weights across the 365 prediction days for each of the years 2021 (a), 2022 (b), and 2023 (c). Each histogram aggregates the distributions of weights across regimes and days, allowing us to inspect how sharply or diffusely the predictive regime assignment behaves over time. From the visual distribution, we observe that:
 A large proportion of weights are clustered near zero, suggesting that for most days, the majority of regimes contribute negligibly to the predictive inference.
 In each year, a sharp secondary concentration near weight $1.0$ is present. This indicates that the classifier is highly confident in assigning most of the prediction days almost exclusively to a single regime. In other words, $\exists\, r^*$ such that $w_{r^*}^{(t)} \approx 1$ and $w_{r \neq r^*}^{(t)} \approx 0$.
The multimodal nature of the histogram—with spikes near 0 and near 1—reveals a regime attribution pattern that is highly sparse in practice. Despite using a soft probabilistic model, the effective regime responsibility is nearly hard (i.e., close to 0–1) for the majority of predictions. Mathematically, this implies that the posterior over regimes,
\[
    w_r^{(t)} = \mathbb{P}(z = r \mid \mathbf{x}_t),
\]
is highly peaked for most $\mathbf{x}_t$, despite the use of a discriminative probabilistic classifier (MLP). In practice, such sparsity is beneficial for interpretability and regime-specific forecasting, as it allows one to attribute the prediction almost entirely to a single dominant regime. However, it also limits the potential benefit of ensembling across multiple regimes, as the non-dominant weights are almost entirely suppressed.

\subsection{Analysis of Prediction Accuracy}
%\section{Comprehensive Statistical Evaluation of the Regime-Aware Neural Process Model}

This section presents a mathematically integrated evaluation of the Regime-aware Neural Process (R-NP) model against benchmark models—Deep Neural Network (DNN) and Lasso Estimated Autoregressive Regression (LEAR)—based on point prediction accuracy, uncertainty quantification, and statistical significance. Importantly, these dimensions are not orthogonal: precise forecasting in nonstationary, regime-shifting time series environments must simultaneously account for predictive accuracy and the model's epistemic and aleatoric uncertainty. We therefore analyze the metrics in an interdependent fashion to rigorously assess the R-NP's suitability as a probabilistic forecast model under complex dynamics.

\begin{figure}[ht!]
    \centering
    \begin{subfigure}{0.45\linewidth}
        \includegraphics[height = 5.5cm, width = 7cm]{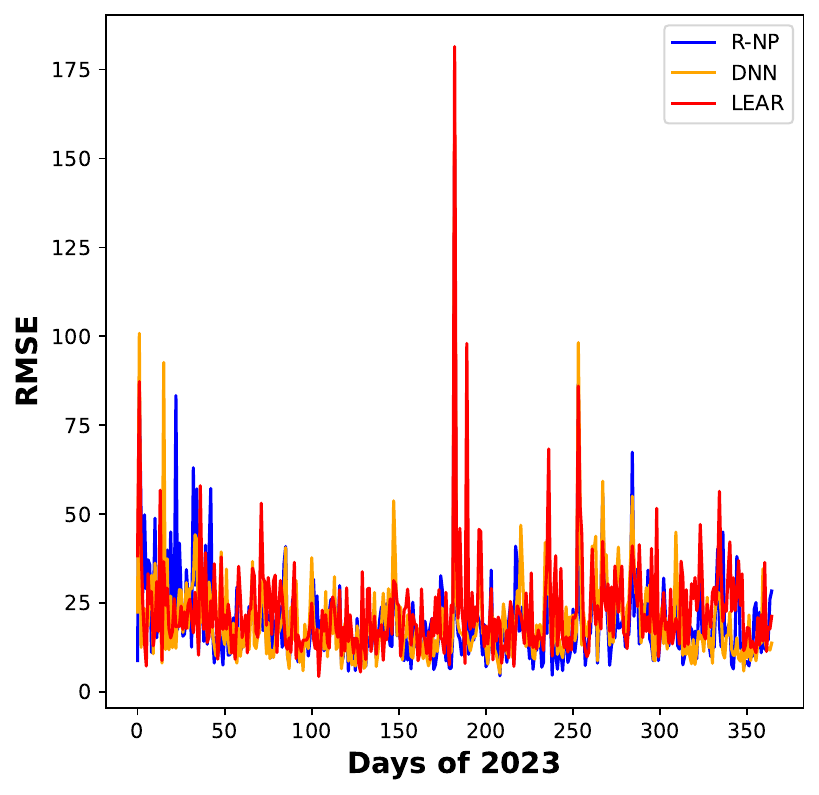}
        \caption{}
        \label{subfig:RMSE_2023_accuracy}
    \end{subfigure}
    \begin{subfigure}{0.4\linewidth}
        \includegraphics[height = 5.5cm, width = 7cm]{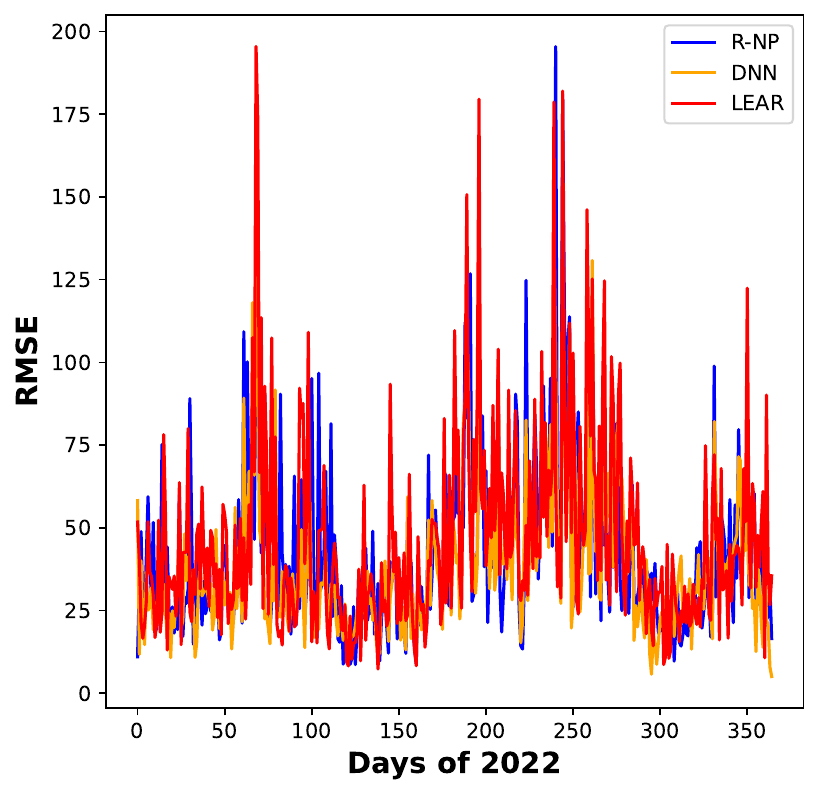}
        \caption{}
        \label{subfig:RMSE_2022_accuracy}
    \end{subfigure}
    \begin{subfigure}{0.45\linewidth}
        \includegraphics[height = 5.5cm, width = 7cm]{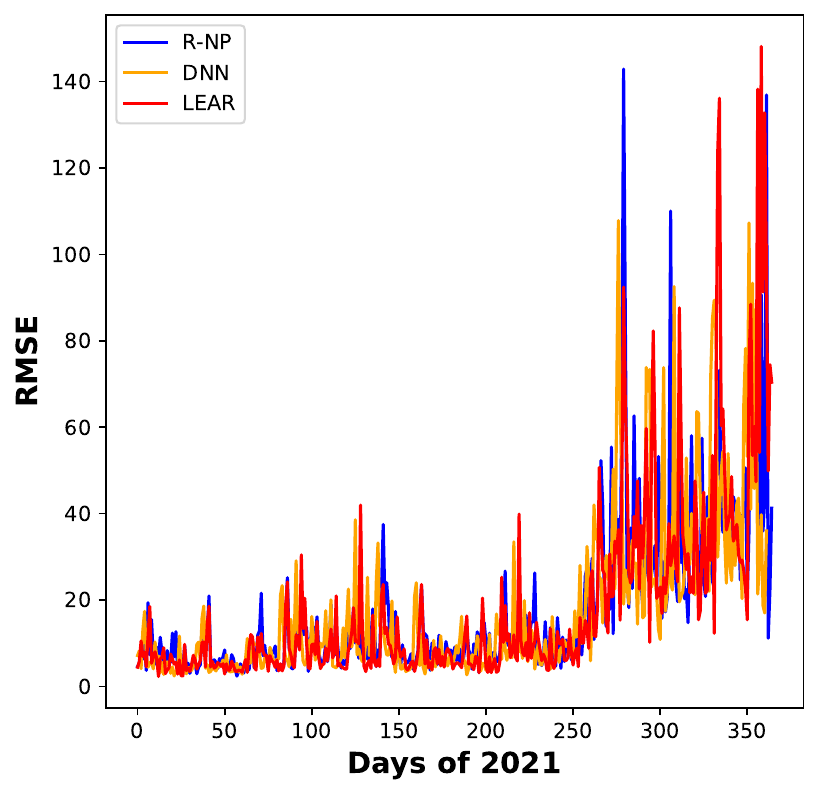}
        \caption{}
        \label{subfig:RMSE_2021_accuracy}
    \end{subfigure}
    \caption{\centering{Prediction Error Comparison for Year (a) 2021, (b) 2022 and (c) 2023}}
    \label{fig:RMSE_accuracy}
\end{figure}

The poin-twise predictive performance is measured using standard loss metrics:  Mean Absolute Error (MAE), Root Mean Squared Error (RMSE), Mean Absolute Percentage Error (MAPE), and Symmetric Mean Absolute Percentage Error (SMAPE). Each of these metrics offers complementary perspectives on predictive fidelity. MAE evaluates the average magnitude of errors without disproportionately penalizing outliers, thus serving as a robust estimator under asymmetric noise. RMSE, by contrast, amplifies the effect of larger deviations and is more sensitive to error variance, effectively capturing both bias and dispersion. MAPE and SMAPE provide scale-invariant measures crucial for economic variables like prices, where proportional error matters more than absolute deviation. 
In the year 2023, which empirically corresponds to a high-variance year, the R-NP model achieves the lowest MAE (16.045) and RMSE (19.690) as shown in Table \ref{all_errors_merged}, indicating strong bias and variance control. These values outperform LEAR significantly and narrowly surpass DNN, suggesting that the regime-conditioned structure of R-NP enables superior local adaptation.
Similarly we see that the MAPE for DNN and LEAR (13.509 and 14.123 receptively) is bigger, yet comparable, to that of R-NP (12.466) which implies superior ability to minimize both the bias and variance of the predictive mean.  However, in 2021, a relatively low-variance year, DNN slightly edges out R-NP in terms of MAPE and SMAPE, despite the latter achieving a lower MAE. This suggests that while R-NP is statistically well-calibrated, its predictive mean in stable time series may be slightly over-regularized, likely due to the model’s tendency to preserve uncertainty even in low-variance settings. In 2022, all error metrics for R-NP degrade marginally, indicating a transitional regime where predictive difficulty increases due to shifts in underlying dynamics. Here, the RMSE and MAPE rise simultaneously, implying both a failure in modeling the conditional mean and in capturing regime-induced aleatoric variance. This dip is instructive—it suggests that regime boundary estimation or the posterior variance in the encoder-decoder pipeline may have underfit local transitions, thus amplifying out-of-sample errors.
\begin{table}[ht!]
\centering
\caption{Model Performance Across Different Error Metrics (2021-2023)}
\begin{tabular}{l c c c c}
    \toprule
    Metric & Model & {Error 2021} & {Error 2022} & {Error 2023} \\
    \midrule
    \multirow{3}{*}{MAE} & R-NP     & 14.917 & 35.740 & 16.045 \\
                         & DNN     & 14.333 & 31.637  & 16.285 \\
                         & LEAR    & 15.271  & 39.130  & 19.241 \\
    \bottomrule
    \multirow{3}{*}{RMSE} & R-NP     & 17.714 & 42.239  & 19.690  \\
                          & DNN     & 17.205 & 38.171   & 19.993  \\
                          & LEAR    & 18.069 & 46.100  & 22.741  \\
    \bottomrule
    \multirow{3}{*}{MAPE} & R-NP     & 6.020 & 7.062  & 12.466  \\
                          & DNN     & 5.773 & 7.236  & 13.509  \\
                          & LEAR    & 7.035  & 13.651 & 14.123  \\
    \bottomrule
    \multirow{3}{*}{SMAPE} & R-NP     & 0.194 & 0.226 & 0.284 \\
                           & DNN     & 0.183 & 0.206 & 0.292 \\
                           & LEAR    & 0.189 & 0.243 & 0.30  \\
    \bottomrule
\end{tabular}
\label{all_errors_merged}
\end{table}

\begin{table}[ht!]
\centering
\caption{Diebold-Mariano Test Results}
\begin{tabular}{l c c }
    \toprule
    \textbf{Model} & \textbf{DM Statistic} & \textbf{P-Value} \\
    \midrule
    \multicolumn{3}{c}{\textbf{2021}} \\
    R-NP Vs LEAR    & -0.999 & 0.317 \\
    R-NP Vs DNN     & -4.639 & $3.493 \times 10^{-6}$ \\
    \midrule
    \multicolumn{3}{c}{\textbf{2022}} \\
    R-NP Vs LEAR    & -1.569 & 0.116 \\
    R-NP Vs DNN     & 2.257  & 0.024 \\
    \midrule
    \multicolumn{3}{c}{\textbf{2023}} \\
    R-NP Vs LEAR    & -1.594 & 0.110 \\
    R-NP Vs DNN     & -0.281 & 0.778 \\
    \bottomrule
\end{tabular}
\label{tab:dm_combined}
\end{table}

\begin{table}[ht!]
\centering
\caption{Prediction Interval Evaluation: PICP and MPIW (2021–2023)}
\begin{tabular}{lcccccc}
\toprule
\multirow{2}{*}{Model} & \multicolumn{2}{c}{\textbf{2023}} & \multicolumn{2}{c}{\textbf{2022}} & \multicolumn{2}{c}{\textbf{2021}} \\
\cmidrule(r){2-3} \cmidrule(r){4-5} \cmidrule(r){6-7}
 & PICP & MPIW & PICP & MPIW & PICP & MPIW \\
\midrule
R-NP       & 0.8260 & 52.35  & 0.5758 & 52.34 & 0.7410 & 29.07\\
\bottomrule
\end{tabular}
\label{tab:picp_mpiw}
\end{table}

\begin{table}[ht!]
\centering
\caption{Pairwise Nemenyi Post-hoc Test p-values (Friedman Test Ranks across Forecast Years and Metrics)}
\begin{tabular}{lccccc}
\toprule
 & \textbf{R-NP}  & \textbf{LEAR} & \textbf{DNN}  \\
\midrule
\textbf{R-NP} & 1.000 & $1.71\times 10^{-10}$ & 0.440  \\
\textbf{LEAR} & $1.71\times 10^{-10}$  &   1.00 &  $2.947\times 10^{-7}$    \\
\textbf{DNN} &  0.440  &   $2.947\times 10^{-7}$    &  1.00      \\
\bottomrule
\end{tabular}
\label{tab:nemenyi}
\end{table}

To evaluate the quality of the predictive distributions, we consider the model’s ability to quantify uncertainty through the dual metrics of Prediction Interval Coverage Probability (PICP) and Mean Prediction Interval Width (MPIW). These metrics are theoretically connected to calibration and sharpness. PICP approximates the empirical coverage of the model’s predictive intervals, and is ideally close to a nominal threshold (e.g., 80\% or 90\%), while MPIW captures the average width of these intervals. Statistically, a well-calibrated model satisfies the inequality $P(y \in \hat{I}_\alpha(x)) \approx 1 - \alpha$ with minimal $\text{width}(\hat{I}_\alpha(x))$, balancing informativeness and reliability. As shown in Table \ref{tab:picp_mpiw}, in 2023, R-NP achieves a PICP of 0.8260 with MPIW of 52.35, suggesting high calibration under volatile conditions without unnecessarily inflated intervals. This is a hallmark of high-quality uncertainty quantification: confidence intervals are both meaningful and precise. In contrast, in 2022, PICP drops to 0.5758 despite a similar MPIW, revealing undercoverage. This discrepancy cannot be attributed to a trade-off; rather, it indicates miscalibration, as the intervals fail to encapsulate the true targets. This reinforces earlier findings from RMSE and MAPE—namely, that 2022 is a regime where the R-NP model struggles both in mean prediction and variance estimation, likely due to transitional or ambiguous dynamics that violate strong separation assumptions in the regime encoding. In 2021, a lower MPIW of 29.07 coupled with PICP of 0.7410 suggests a reasonable calibration under conditions of reduced epistemic and aleatoric uncertainty, which aligns with the relatively benign regime of that year. Crucially, benchmark models such as DNN and LEAR offer no native support for uncertainty quantification—an omission that makes them fundamentally inadequate for applications requiring interval forecasts or risk-sensitive decision-making.

These findings are not merely descriptive but statistically testable. The Diebold-Mariano (DM) test results presented in Table~\ref{tab:dm_combined} further solidify these observations by statistically testing the null hypothesis of equal predictive accuracy between R-NP and each competing model. The DM test examines whether the mean loss differential between two forecasts is statistically distinguishable from zero under consistent loss functions. In 2021, R-NP significantly outperforms DNN (DM statistic of -4.639, $p = 3.49 \times 10^{-6}$), affirming that its forecast errors are not only lower in expectation but also statistically distinguishable. In contrast, for 2022, the DNN significantly outperforms R-NP ($p = 0.024$), aligning with the earlier observation of poor PICP and higher RMSE, likely due to regime transition ambiguity that hinders effective posterior inference in R-NP. For 2023, no significant difference is detected between R-NP and DNN ($p = 0.778$), although R-NP still exhibits superior uncertainty quantification. LEAR consistently trails both models and is never statistically superior to R-NP, indicating its ineffectiveness in both linear approximation and generalization across regimes.

The importance of aggregating these multiple metrics into a single statistical framework is further supported by the pairwise Nemenyi post-hoc test (Table~\ref{tab:nemenyi}), applied to Friedman test ranks across years and metrics. The results show that R-NP significantly outperforms LEAR with $p = 1.71 \times 10^{-10}$, while the R-NP and DNN comparison yields a non-significant $p = 0.44$. These findings echo the nuanced interpretation: while R-NP does not dominate DNN uniformly in point prediction, it exhibits more statistically consistent performance when interval coverage and risk quantification are accounted for. DNN, though competitive in average prediction error, is a deterministic model incapable of producing predictive intervals without post hoc ensembling or bootstrapping—methods that are not only computationally intensive but also theoretically less principled than variational inference in the neural process framework.

The absence of uncertainty quantification in both LEAR and DNN models represents a fundamental limitation in their applicability for probabilistic forecasting under regime shifts. Time series prediction in economic, financial, or energy domains frequently demands not just a point forecast but also an understanding of confidence and risk. In such contexts, deterministic predictions are brittle, particularly when the data distribution is nonstationary or the regime identity itself is latent and dynamic. The R-NP model, with its capacity to encode contextual information and yield full predictive distributions through latent-variable inference, provides both epistemic and aleatoric uncertainty estimates. This enables not only better forecasting but also risk-aware decision making, scenario generation, and anomaly detection—none of which are readily supported by standard DNN or LEAR models.

In sumary, the R-NP model demonstrates a holistic superiority across predictive performance, statistical calibration, and model interpretability under structural uncertainty. Its design allows adaptation to regime-dependent dynamics while preserving probabilistic coherence through Bayesian latent variable modeling. While challenges remain in transitional regimes, as seen in the 2022 degradation, the model's ability to outperform traditional and black-box baselines in volatile conditions (2023) and to retain statistical significance in low-variance regimes (2021) underlines its robustness. The comprehensive statistical evidence presented here justifies the adoption of R-NP as a reliable forecasting framework in applications where both accuracy and uncertainty are critical.

\subsection{Analysis of Operation Strategies}
 
%We also assume that we start with the empty battery and at the end of the day the battery must be empty so that it can be re-used for the next day. For the comparison of the effect of predicted price and its estimated uncertainty, the comparison of each strategy with the perfect foresight has been done. The numerical results in the Table \ref{mae_error_profit_ Case_I} to \ref{mae_error_profit_Case_IV} shows that in each strategy the the profit using the price calculated by the LEAR Model is performing better, however, one should note that the point prediction accuracy for the price via LEAR model is significantly poor [c.f. Table \ref{mae_error} to \ref{tab:picp_mpiw}].
The operational strategies formulated in the different cases as the optimization problems \ref{no_uncer_no_exogebous_var} to \ref{with_uncer_no_exogebous_var_fix_load} are tested using the price predicted via R-NP and the benchmark models. In this section we show the graphical results only for the year 2023 for all of the cases while numerical results for 2021, 2022 and 2023 have been shown. The graphical results for the year 2021 and 2022 are given in the Appendix \ref{ap:opertational_strategies}. We choose the capacity of battery $C_{\max} = 10 KWh$ with charging and discharging capacity, $\eta_{c}, \eta_{d} = 95 \%$. The maximum charge in and out, $x_{\min}, x_{\max} = 5 KWh, -5 KWh$.

The comprehensive evaluation of strategies, guided by various electricity price prediction models (R-NP, DNN, LEAR) across diverse cases and years, reveals an inherently complex decision landscape. The evaluation is done using two metrics: the mean realized profit (or cost for Case IV) and the Mean Absolute Error (MAE) of this realized profit/cost relative to a Perfect Foresight (PF) benchmark. The Table \ref{profit_ Case_I} to \ref{mae_error_profit_Case_IV} provide insights into operational effectiveness, which is a composite outcome influenced by the underlying price forecasts and the robustness of the optimization model. 

\begin{table}[ht!]
\centering
\caption{Profit Comparison (MWh/Euros) for Case I}
\begin{tabular}{l c c c}
    \toprule
    Model & {Mean Profit in 2021} & {Mean Profit in 2022} & {Mean Profit in 2023} \\
    \midrule
    R-NP     & 0.4459 & 1.0792 & 0.6160 \\
    DNN    & 0.4616  & 1.0867  & 0.6227 \\
    LEAR     & 0.4654 & 1.1250  & 0.6267 \\
    Perfect Foresight & 0.5345 & 1.2712 & 0.6995 \\
    
    \bottomrule
\end{tabular}
\label{profit_ Case_I}
\end{table}

\begin{table}[ht!]
\centering
\caption{Profit Comparison (MWh/Euros) for Case II}
\begin{tabular}{l c c c}
    \toprule
    Model & {Mean Profit in 2021} & {Mean Profit in 2022} & {Mean Profit in 2023} \\
    \midrule
    R-NP     & 0.2712 & 0.7203 & 0.1150 \\
    DNN    & 0.2307  & 0.5834  & 0.0210 \\
    LEAR     & 0.2720 & 0.8001  & 0.1254 \\
    Perfect Foresight & 0.5345 & 1.2712 & 0.6995 \\
    
    \bottomrule
\end{tabular}
\label{profit_ Case_II}
\end{table}

\begin{table}[ht!]
\centering
\caption{Profit Comparison (MWh/Euros) for Case III}
\begin{tabular}{l c c c}
    \toprule
    Model & {Mean Profit in 2021} & {Mean Profit in 2022} & {Mean Profit in 2023} \\
    \midrule
    R-NP     & 0.2806 & 0.7444 & 0.1310 \\
    DNN    & 0.2416  & 0.6099  & 0.0252 \\
    LEAR     & 0.2856 & 0.8286  & 0.1401 \\
    Perfect Foresight & 0.5617 & 1.3015 & 0.7330 \\
    
    \bottomrule
\end{tabular}
\label{profit_Case_III}
\end{table}

\begin{table}[ht!]
\centering
\caption{Cost Comparison (MWh/Euros) for Case IV}
\begin{tabular}{l c c c}
    \toprule
    Model & {Mean Cost in 2021} & {Mean Cost in 2022} & {Mean Cost in 2023} \\
    \midrule
    R-NP     & 0.7842 & 1.9818 & 0.8197 \\
    DNN    & 0.7951  & 1.9905  & 0.8459 \\
    LEAR     & 0.7815 & 1.9718  & 0.8071 \\
    Perfect Foresight & 0.7307 & 1.8876 & 0.7026 \\
    
    \bottomrule
\end{tabular}
\label{profit_Case_IV}
\end{table}

The LEAR consistently achieves the highest mean realized profit in profit-maximization scenarios (e.g., Table \ref{profit_ Case_I}, Table \ref{profit_ Case_II}, Table \ref{profit_Case_III}) and the lowest mean realized cost in cost-minimization (Table \ref{profit_Case_IV}), its MAE performance is not always universally superior. Specifically, in Case IV for 2022, DNN exhibits a significantly lower MAE (Table \ref{mae_error_profit_Case_IV}: 0.1028) compared to LEAR (0.8420), implying that DNN's operational decisions in that specific instance were far more optimally aligned with the ideal, despite LEAR achieving a slightly lower absolute cost for that year. Beyond these operational trade-offs, a crucial and often overlooked conflict lies between a model's direct price prediction accuracy and its ultimate operational performance. It is mathematically possible for a model like LEAR, despite demonstrating superior operational outcomes (higher profits, lower operational MAEs), to exhibit higher Mean Absolute Errors or Root Mean Squared Errors in its raw price forecasts compared to other models. This phenomenon can occur if LEAR's forecasts, while perhaps globally less accurate, are precisely accurate during critical periods that supports the trading strategies (e.g., extreme price spikes or dips), or if the optimization algorithm is robust enough to tolerate certain types of prediction inaccuracies. Such a scenario creates a fundamental conflict: should the preferred model be the one that provides the most accurate raw predictions, or the one that, regardless of its raw prediction quality, consistently translates its forecasts into the most effective real-world operational decisions? Given these inherent conflicts across multiple criteria (absolute economic value, optimality gap, raw forecasting accuracy), diverse operational contexts (Case I to IV), and varying performance across years, a simple univariate ranking is insufficient. The Technique for Order Preference by Similarity to Ideal Solution (TOPSIS) provides a rigorous, quantitative framework to aggregate these diverse and often conflicting criteria. TOPSIS allows for the explicit weighting of each criterion according to the decision-maker's priorities, thus enabling a transparent and defensible identification of the ``best compromise" solution that is closest to the ideal and farthest from the anti-ideal solution.

\begin{table}[ht!]
\centering
\caption{Mean Absolute Error for Profit from Real Data and Predicted Data for Case I}
\begin{tabular}{l c c c}
    \toprule
    Model & {Error 2021} & {Error 2022} & {Error 2023} \\
    \midrule
    R-NP     & 0.0886 & 0.1919 & 0.0835 \\
    DNN    & 0.0765  & 0.1844  & 0.0768 \\
    LEAR     & 0.0690 & 0.1461  & 0.0728 \\
    
    \bottomrule
\end{tabular}
\label{mae_error_profit_ Case_I}
\end{table}

\begin{table}[ht!]
\centering
\caption{Mean Absolute Error for Profit from Real Data and Predicted Data for Case II}
\begin{tabular}{l c c c}
    \toprule
    Model & {Error 2021} & {Error 2022} & {Error 2023} \\
    \midrule
    R-NP     & 0.2633 & 0.5508 & 0.5844 \\
    DNN    & 0.3074  & 0.6877  & 0.6785 \\
    LEAR     & 0.2625 & 0.4711  & 0.5742 \\
    
    \bottomrule
\end{tabular}
\label{mae_error_profit_ Case_II}
\end{table}

\begin{table}[ht!]
\centering
\caption{Mean Absolute Error for Profit from Real Data and Predicted Data for Case III}
\begin{tabular}{l c c c}
    \toprule
    Model & {Error 2021} & {Error 2022} & {Error 2023} \\
    \midrule
    R-NP     & 0.2810 & 0.5571 & 0.6019 \\
    DNN    & 0.3238  & 0.6916  & 0.7078 \\
    LEAR     & 0.2760 & 0.4729  & 0.5929 \\
    
    \bottomrule
\end{tabular}
\label{mae_error_profit_Case_III}
\end{table}

\begin{table}[ht!]
\centering
\caption{Mean Absolute Error for Profit from Real Data and Predicted Data for Case IV}
\begin{tabular}{l c c c}
    \toprule
    Model & {Error 2021} & {Error 2022} & {Error 2023} \\
    \midrule
    R-NP     & 0.5350 & 0.9420 & 0.1171 \\
    DNN    & 0.6176  & 0.1028  & 0.1433 \\
    LEAR     & 0.5077 & 0.8420  & 0.1044 \\
    
    \bottomrule
\end{tabular}
\label{mae_error_profit_Case_IV}
\end{table}

\begin{figure}[ht!]
    \centering
    \begin{subfigure}{0.45\linewidth}
        \includegraphics[height = 5.5cm, width = 7cm]{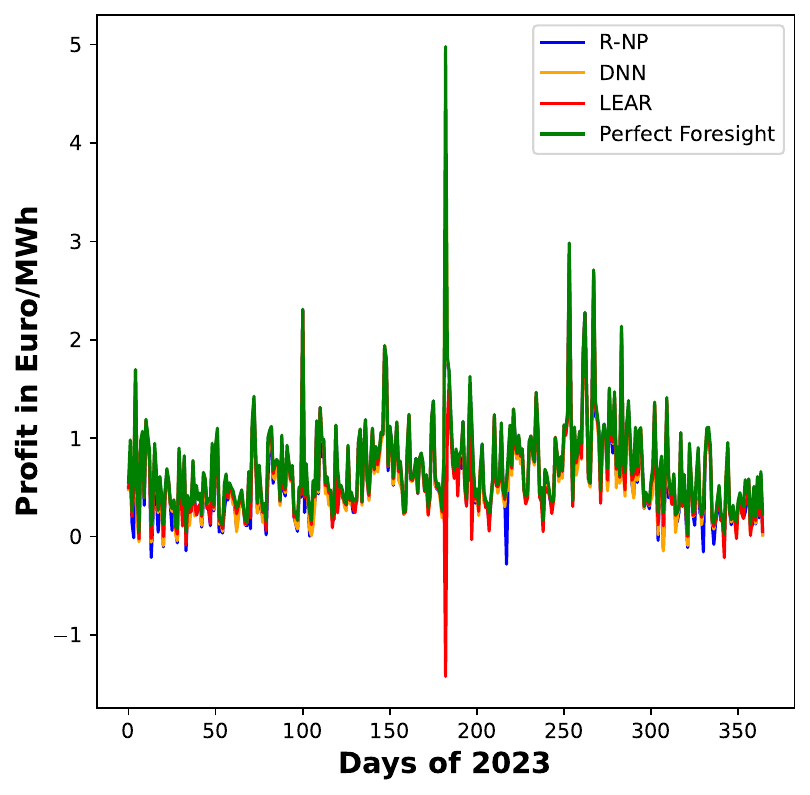}
        \caption{}
        \label{subfig:Profit_2023_case_I}
    \end{subfigure}
    \begin{subfigure}{0.45\linewidth}
        \includegraphics[height = 5.5cm, width = 7cm]{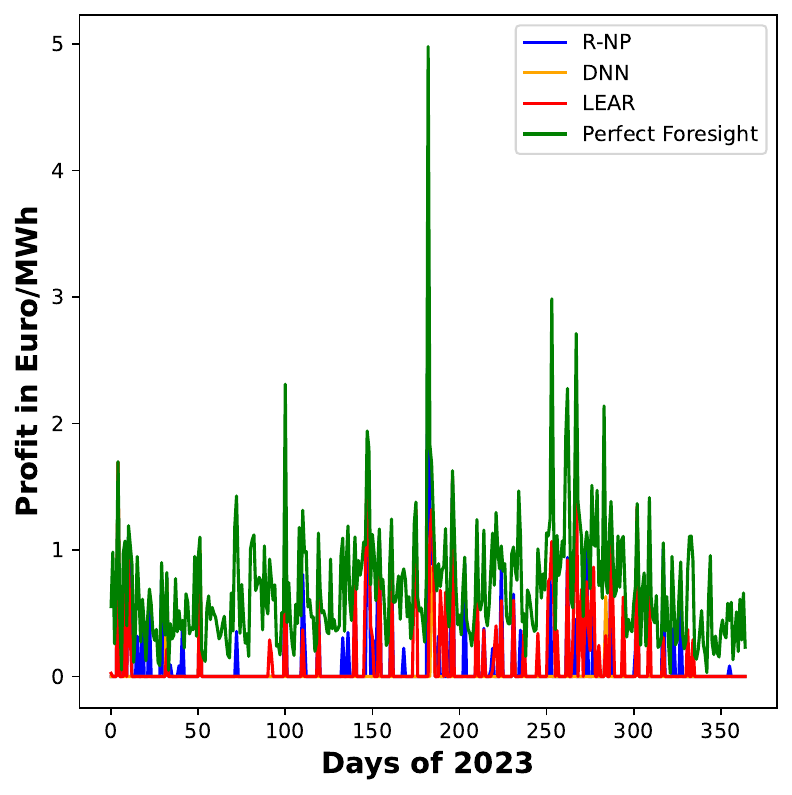}
        \caption{}
        \label{subfig:Profit_2023_case_II}
    \end{subfigure}
    \begin{subfigure}{0.45\linewidth}
        \includegraphics[height = 5.5cm, width = 7cm]{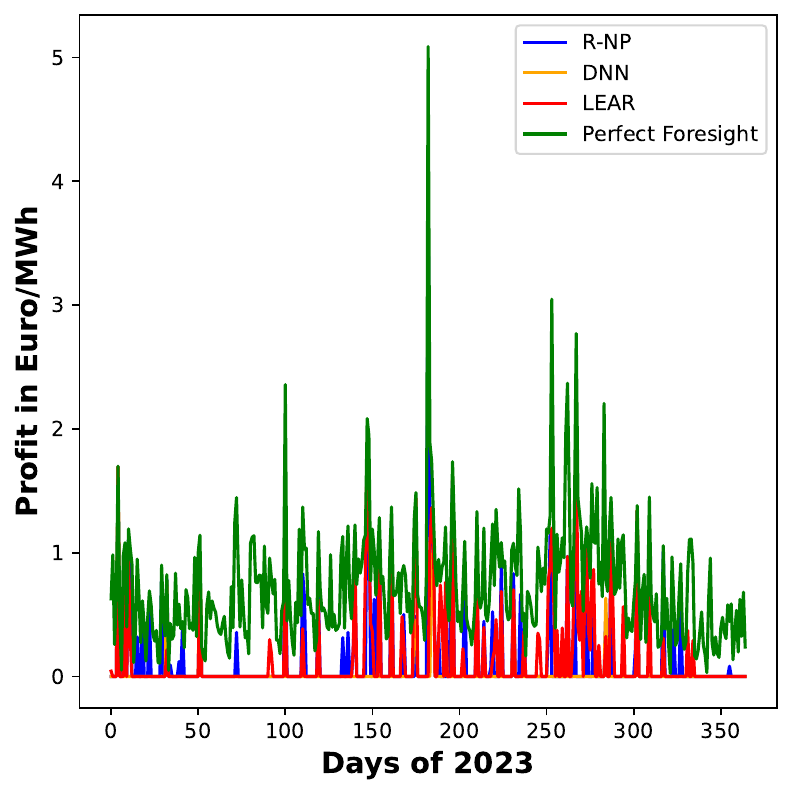}
        \caption{}
        \label{subfig:Profit_2023_case_III}
    \end{subfigure}
    \begin{subfigure}{0.45\linewidth}
        \includegraphics[height = 5.5cm, width = 7cm]{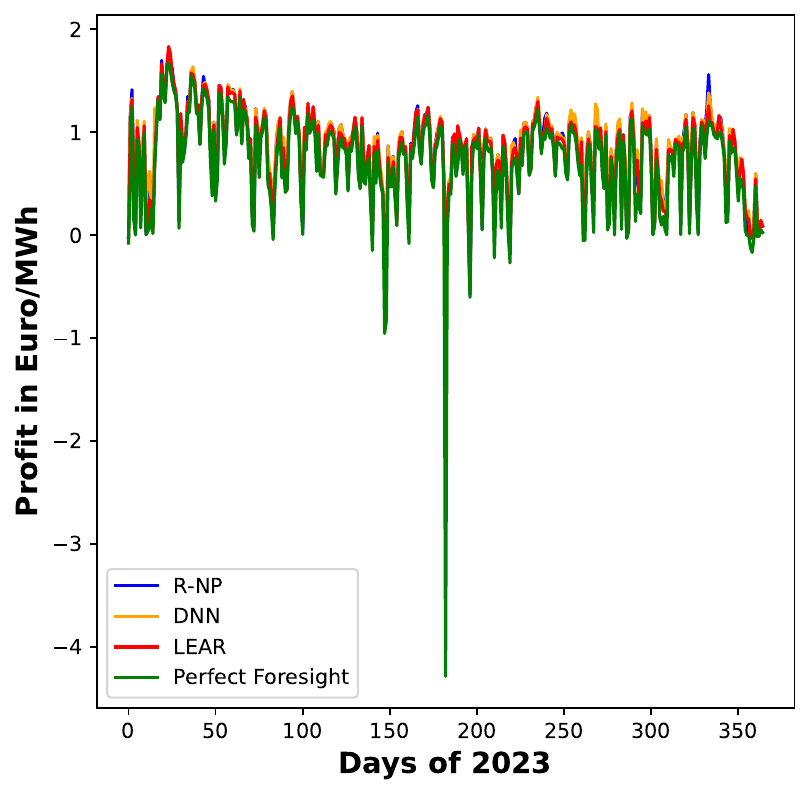}
        \caption{}
        \label{subfig:Profit_2023_case_IV}
    \end{subfigure}
    \caption{\centering{Comparison of daily profit through the predicted price and real price for year 2023 (a) Case I, (b) Case II, (c) Case III and (d) Case IV}}
    \label{fig:Profit_Comparison}
\end{figure}

\begin{figure}[ht!]
    \centering
    \begin{subfigure}{0.45\linewidth}
        \includegraphics[height = 5.5cm, width = 7cm]{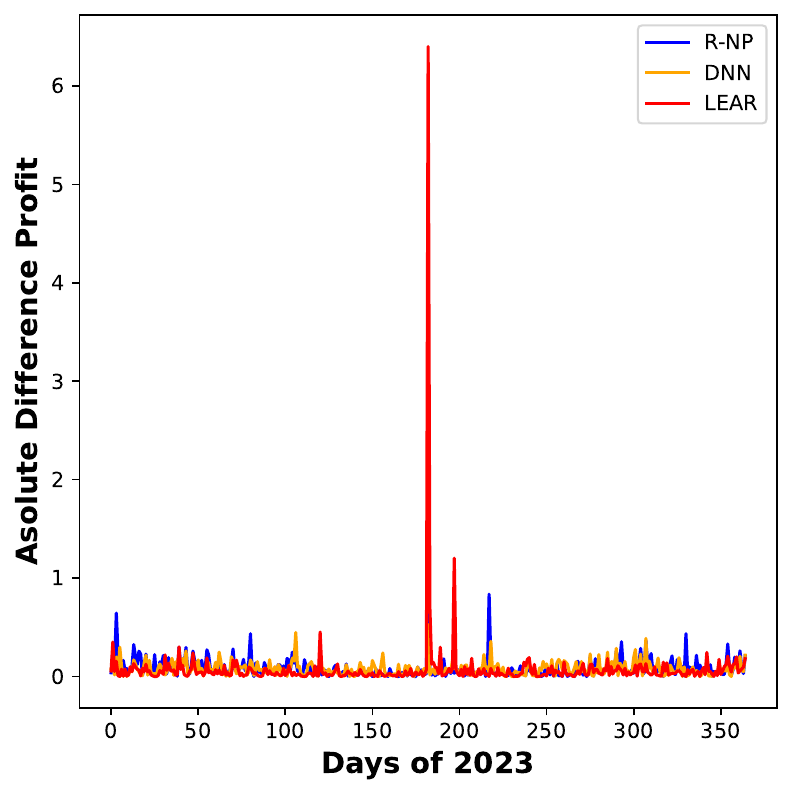}
        \caption{}
        \label{subfig:Profit_Diff_2023_case_I}
    \end{subfigure}
    \begin{subfigure}{0.45\linewidth}
        \includegraphics[height = 5.5cm, width = 7cm]{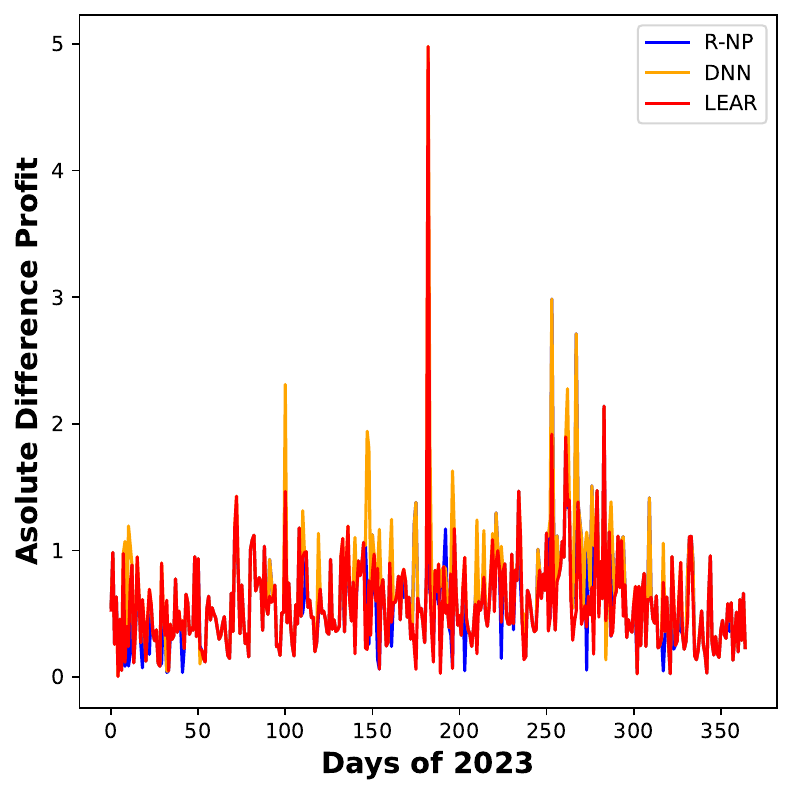}
        \caption{}
        \label{subfig:Profit_Diff_2023_case_II}
    \end{subfigure}
    \begin{subfigure}{0.45\linewidth}
        \includegraphics[height = 5.5cm, width = 7cm]{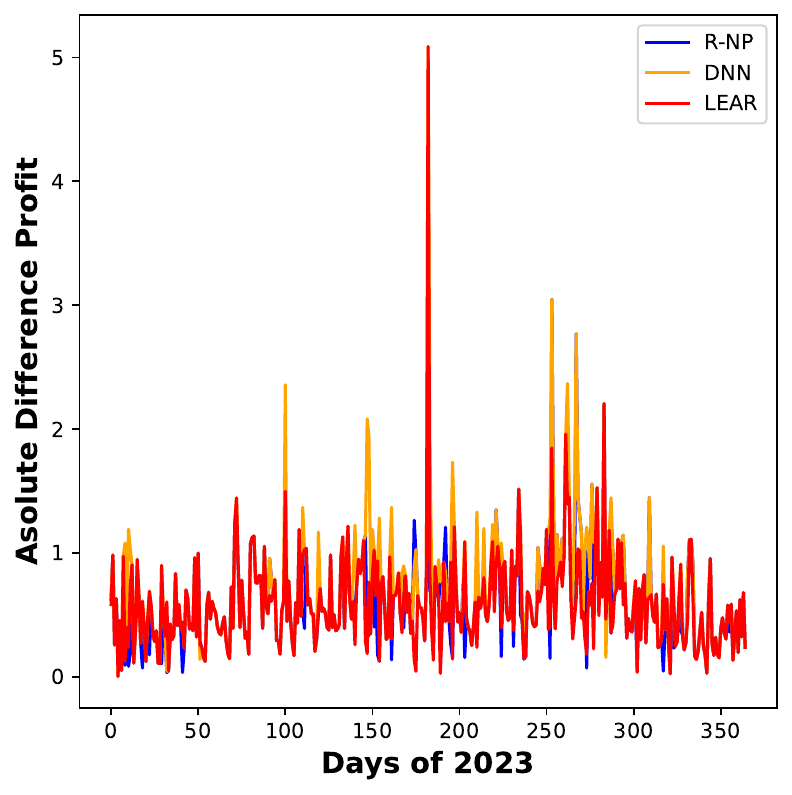}
        \caption{}
        \label{subfig:Profit_Diff_2023_case_III}
    \end{subfigure}
    \begin{subfigure}{0.45\linewidth}
        \includegraphics[height = 5.5cm, width = 7cm]{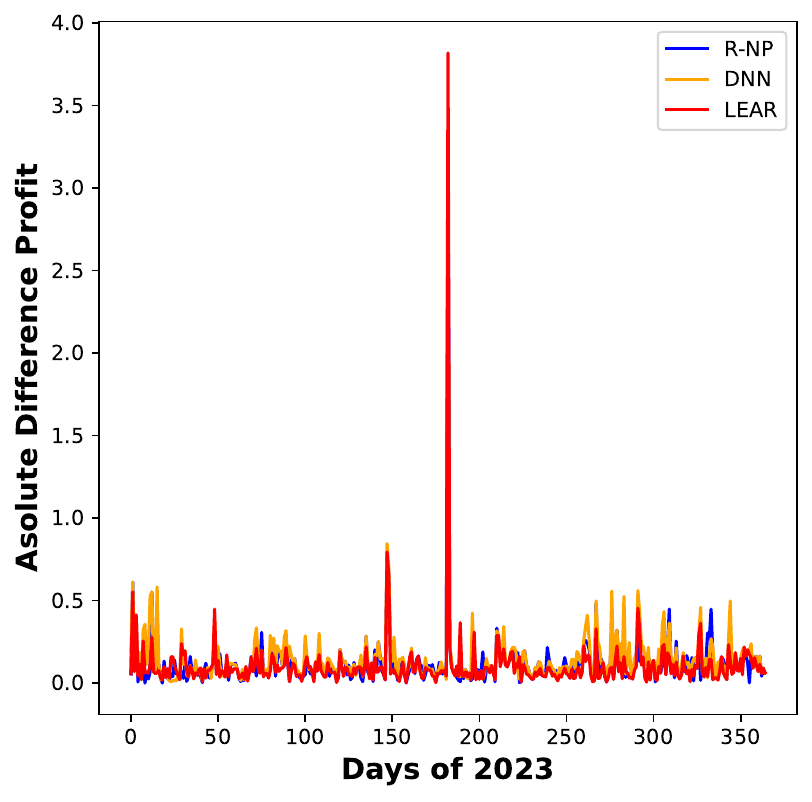}
        \caption{}
        \label{subfig:Profit_Diff_2023_case_IV}
    \end{subfigure}
    \caption{\centering{Comparison of daily profit difference through the predicted price and real price for year 2023 (a) Case I, (b) Case II, (c) Case III and (d) Case IV}}
    \label{fig:Profit_Diff_Comparison}
\end{figure}
\newpage
\subsection{Computation of MCDM Scores using TOPSIS}
TOPSIS is a widely used MCDM method that ranks alternatives based on their Euclidean distance from an ideal best solution and a negative-ideal worst solution. In our case, the alternatives are the three predictive models—RNP, DNN, and LEAR—evaluated annually using the predictive accuracy metrics: MAE, RMSE, SMAPE, MAPE (all to be minimized) , profit difference (i.e., regret) for Cases I–IV (all to be minimized) and the , mean profit for each year (to be maximized).

Mathematically, let $A = \{a_1, a_2, a_3\}$ denote the set of models (RNP, DNN, LEAR), and $C = \{c_1, \ldots, c_8\}$ the set of criteria. The decision matrix $X = [x_{ij}]$ is formed where $x_{ij}$ represents the performance of model $a_i$ on criterion $c_j$. The steps for TOPSIS are as follows:

\begin{enumerate}
    \item \textbf{Normalize the decision matrix}:
    \[
    r_{ij} = \frac{x_{ij}}{\sqrt{\sum_{i=1}^m x_{ij}^2}}, \quad \forall i,j
    \]
    \item \textbf{Construct the weighted normalized matrix}:
    \[
    v_{ij} = w_j \cdot r_{ij}, \quad \text{where } w_j \text{ is the weight of criterion } c_j, \text{ and } \sum_{j=1}^{n} w_j = 1
    \]
    In our case we have assigned equal weights $w_j = \frac{1}{12}$ to the criteria.

    \item \textbf{Determine the positive ideal solution (PIS) and negative ideal solution (NIS)}:
    \[
    v_j^+ = \min_i v_{ij}, \quad v_j^- = \max_i v_{ij} \quad \text{(for minimization problems)}
    \]

    \item \textbf{Compute the Euclidean distance to the ideal and anti-ideal solutions}:
    \[
    S_i^+ = \sqrt{\sum_{j=1}^n (v_{ij} - v_j^+)^2}, \quad S_i^- = \sqrt{\sum_{j=1}^n (v_{ij} - v_j^-)^2}
    \]

    \item \textbf{Calculate the relative closeness to the ideal solution}:
    \[
    C_i = \frac{S_i^-}{S_i^+ + S_i^-}
    \]

    \item \textbf{Rank the alternatives}: A higher $C_i$ indicates better overall performance.
\end{enumerate}

This approach enables a quantitative aggregation of both prediction fidelity and operational robustness into a single composite score for each model, per year. By applying TOPSIS annually, we systematically identify the most balanced model that not only forecasts well but also leads to effective economic decisions under realistic market constraints. The results in Table \ref{TPOSIS} shows that in all three years, R-NP model achieves the highest score which further supports the use of R-NP model.

\begin{table}[ht!]
\centering
\caption{TOPSIS Score}
\begin{tabular}{l c c c }
    \toprule
    Model & {TPOSIS Score 2021} & {TPOSIS Score 2022} & {TPOSIS Score 2023} \\
    \midrule
    R-NP     & 0.5312 & 0.7748 & 0.8282 \\
    DNN    & 0.5001  & 0.3511  & 0.2095 \\
    LEAR     & 0.5831 & 0.6216  & 0.7822 \\ 
    
    \bottomrule
\end{tabular}
\label{TPOSIS}
\end{table}

\section{Conclusion}\label{sec:conslusion}
Through this work we have introduced and validated a regime-aware conditional neural processes for operational forecasting in volatile electricity markets leveraging the Disentangled Sticky Hierarchical Dirichlet Process Hidden Markov Model (DS-HDP-HMM) for unsupervised regime discovery. This model explicitly accounts for inherent non-stationarities introduced by renewable energy integration. Our comprehensive evaluation, which integrated these forecasts into diverse opertaional strategies in presence of battery storage system optimization frameworks, empirically confirmed a critical finding: raw prediction accuracy alone often fails to translate directly into optimal operational outcomes, such as maximized profit or minimized regret. By employing the multi criteria decision making framework of TOPSIS, we rigorously quantified the performance of R-NP against leading benchmarks (DNN and LEAR). The findings reveal that while LEAR demonstrated superior overall performance in 2021, our proposed R-NP model consistently achieved the highest composite scores in 2022 and 2023, establishing it as the most balanced and preferred solution for navigating complex operational strategies across varying market conditions. These results underscore the indispensable value of multi-criteria evaluation for selecting forecasting models that genuinely optimize economic objectives in dynamic energy environments.

The future work aims to enhance the practical utility and robustness of such forecasting and evaluation frameworks. Methodologically, investigating adaptive weighting schemes within the TOPSIS framework could allow for dynamic prioritization of criteria based on evolving market conditions or decision-makers' risk appetites, moving beyond fixed equal weights. Furthermore, extending the R-NP model to incorporate a broader spectrum of real-time market data, including grid congestion and service prices, alongside explicit stochastic optimization for battey storage under uncertainty, could enhance its responsiveness and applicability in short-term dispatch. Exploring the generalization capabilities of the regime-aware architecture to other forms of energy storage (e.g., hydrogen, thermal) or diverse international energy markets with distinct regulatory structures would also be valuable.

\bibliographystyle{plain}
\bibliography{bibfile}

\newpage
\appendixpage

\appendix

\section{Theoretical Motivation for Transitioning from HDP-HMM to Sticky and Disentangled Sticky HDP-HMM}

In this section, we sequentially formalize the theoretical justification for moving from the Hierarchical Dirichlet Process Hidden Markov Model (HDP-HMM) to the Sticky HDP-HMM (sHDP-HMM), and finally to the Disentangled Sticky HDP-HMM (DS-HDP-HMM), using a series of theorems and remarks.

\subsection{Exchangeability of Transitions in HDP-HMM}\label{thm:hdp-exchangeable}
\begin{theorem}[Exchangeability in HDP-HMM]
Let $\boldsymbol{\pi}_j^{\text{HDP}} \sim \mathrm{DP}(\alpha, \beta)$, where $\beta = (\beta_1, \beta_2, \dots)$ is a base measure over a countable state space $\mathbb{N}$. Then, the expected transition probability from state $j$ to any state $k \in \mathbb{N}$ is
\[
\mathbb{E}[\pi_{jk}^{\text{HDP}}] = \beta_k.
\]
In particular,
\[
\mathbb{E}[\pi_{jj}^{\text{HDP}}] = \beta_j.
\]
This implies that transitions are a priori exchangeable: self-transitions are not preferred over other transitions.
\end{theorem}

\begin{proof}
From the properties of the Dirichlet Process, if $\boldsymbol{\pi}_j \sim \mathrm{DP}(\alpha, \beta)$, then for any measurable subset $A \subseteq \mathbb{N}$,
\[
\mathbb{E}[\pi_j(A)] = \beta(A).
\]
Setting $A = \{k\}$, we obtain $\mathbb{E}[\pi_{jk}] = \beta_k$ for all $k$, including $k = j$.
\end{proof}

\subsection{Biasing Self-Transitions with Sticky HDP-HMM}\label{rem:shdp-self}
\begin{remark}[Self-transition in sHDP-HMM]
In the sHDP-HMM, the prior on transitions is modified:
\[
\boldsymbol{\pi}_j^{\text{sHDP}} \sim \mathrm{DP}(\alpha \beta + \kappa \delta_j),
\]
where $\kappa > 0$ is the self-transition ("stickiness") parameter. Then the expected self-transition becomes:
\[
\mathbb{E}[\pi_{jj}^{\text{sHDP}}] = \frac{\alpha \beta_j + \kappa}{\alpha + \kappa} > \beta_j.
\]
This breaks the exchangeability of transitions and promotes state persistence.
\end{remark}

\begin{remark}[Why \textbf{sHDP-HMM} is Preferable to HDP-HMM]
Consider the difference in expected self-transition probabilities:
\[
\mathbb{E}[\pi_{jj}^{\text{sHDP}}] - \mathbb{E}[\pi_{jj}^{\text{HDP}}] = \frac{\alpha \beta_j + \kappa}{\alpha + \kappa} - \beta_j.
\]
After simplification,
\[
= \frac{\kappa (1 - \beta_j)}{\alpha + \kappa} > 0 \quad \text{for } \kappa > 0.
\]
Thus, sHDP-HMM \emph{increases} the expected self-transition probability relative to HDP-HMM, offering better temporal consistency in latent regimes, which is desirable in electricity price modeling where regimes are persistent.
\end{remark}

\subsection{Limitations of sHDP-HMM: Entangled Parameters}\label{rem::limitsHDPHMM}\
\begin{remark}[Parameter Coupling in sHDP-HMM]
In the sHDP-HMM, both $\alpha$ and $\kappa$ influence $\boldsymbol{\pi}_j^{\text{sHDP}}$:
\begin{itemize}
    \item $\alpha$ affects how closely $\boldsymbol{\pi}_j$ follows $\beta$
    \item $\kappa$ increases mass on the self-transition
\end{itemize}
This entanglement makes it difficult to independently tune self-persistence and inter-state similarity.
\end{remark}

\subsection{Disentangled Self-Persistence via DS-HDP-HMM}\label{thm:dshdp-expected}
\begin{theorem}[Expected Transitions in DS-HDP-HMM]
Let the transition distribution be defined as:
\begin{align*}
    \kappa_j &\sim \text{Beta}(\rho_1, \rho_2), \\
    \overline{\pi}_j &\sim \text{DP}(\alpha \beta), \\
    \pi_{jk}^{\text{DS}} &= \kappa_j \delta_j(k) + (1 - \kappa_j) \overline{\pi}_{jk}.
\end{align*}
Then:
\begin{enumerate}
    \item Expected self-transition:
    \[
    \mathbb{E}[\pi_{jj}^{\text{DS}}] = \mathbb{E}[\kappa_j] + (1 - \mathbb{E}[\kappa_j])\beta_j
    = \frac{\rho_1}{\rho_1 + \rho_2} + \left(1 - \frac{\rho_1}{\rho_1 + \rho_2}\right)\beta_j.
    \]
    \item Expected transition to $k \neq j$:
    \[
    \mathbb{E}[\pi_{jk}^{\text{DS}}] = (1 - \mathbb{E}[\kappa_j])\beta_k.
    \]
\end{enumerate}
\end{theorem}

\begin{proof}
The transition is a convex mixture. Taking expectations:
\begin{align*}
    \mathbb{E}[\pi_{jj}^{\text{DS}}] &= \mathbb{E}[\kappa_j] + (1 - \mathbb{E}[\kappa_j]) \mathbb{E}[\overline{\pi}_{jj}] = \mathbb{E}[\kappa_j] + (1 - \mathbb{E}[\kappa_j]) \beta_j, \\
    \mathbb{E}[\pi_{jk}^{\text{DS}}] &= (1 - \mathbb{E}[\kappa_j]) \beta_k \quad (k \neq j).
\end{align*}
\end{proof}

\begin{remark}[Why DS-HDP-HMM is Preferable to sHDP-HMM]
In DS-HDP-HMM, the persistence probability $\kappa_j$ is modeled separately via a Beta distribution:
\begin{align*}
    \mathbb{E}[\kappa_j] &= \frac{\rho_1}{\rho_1 + \rho_2}, \\
    \mathrm{Var}(\kappa_j) &= \frac{\rho_1 \rho_2}{(\rho_1 + \rho_2)^2 (\rho_1 + \rho_2 + 1)}.
\end{align*}
This achieves full disentanglement:
\begin{itemize}
    \item $\alpha$ controls the similarity of $\overline{\pi}_j$ across states (inter-regime behavior).
    \item $\rho_1, \rho_2$ independently control regime persistence.
\end{itemize}
Such flexibility is essential for heterogeneous markets (like electricity) where some regimes are sticky (e.g., base load), and others are fleeting (e.g., price spikes).
\end{remark}
\newpage
\section{Analysis of Operational Strategies}\label{ap:opertational_strategies}
\begin{figure}[ht!]
    \centering
    \begin{subfigure}{0.45\linewidth}
        \includegraphics[height = 5.5cm, width = 7cm]{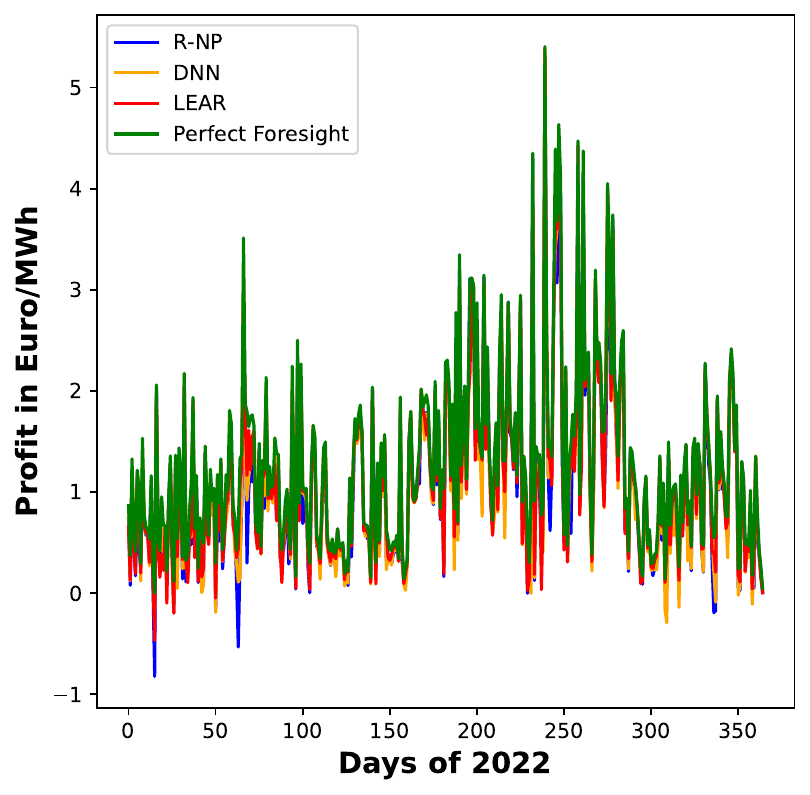}
        \caption{}
        \label{subfig:Profit_2022_case_I}
    \end{subfigure}
    \begin{subfigure}{0.45\linewidth}
        \includegraphics[height = 5.5cm, width = 7cm]{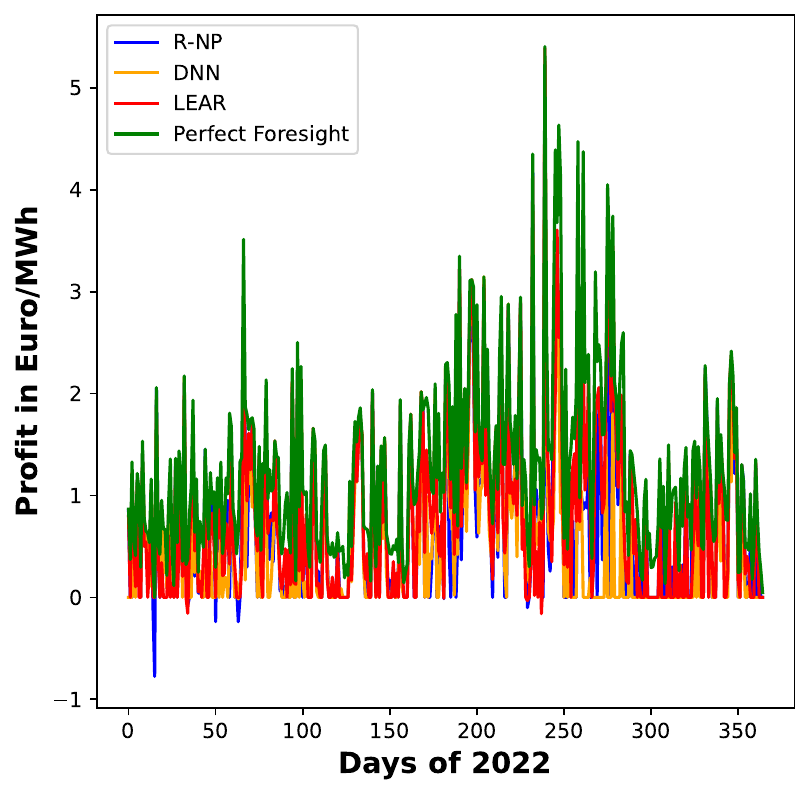}
        \caption{}
        \label{subfig:Profit_2022_case_II}
    \end{subfigure}
    \begin{subfigure}{0.45\linewidth}
        \includegraphics[height = 5.5cm, width = 7cm]{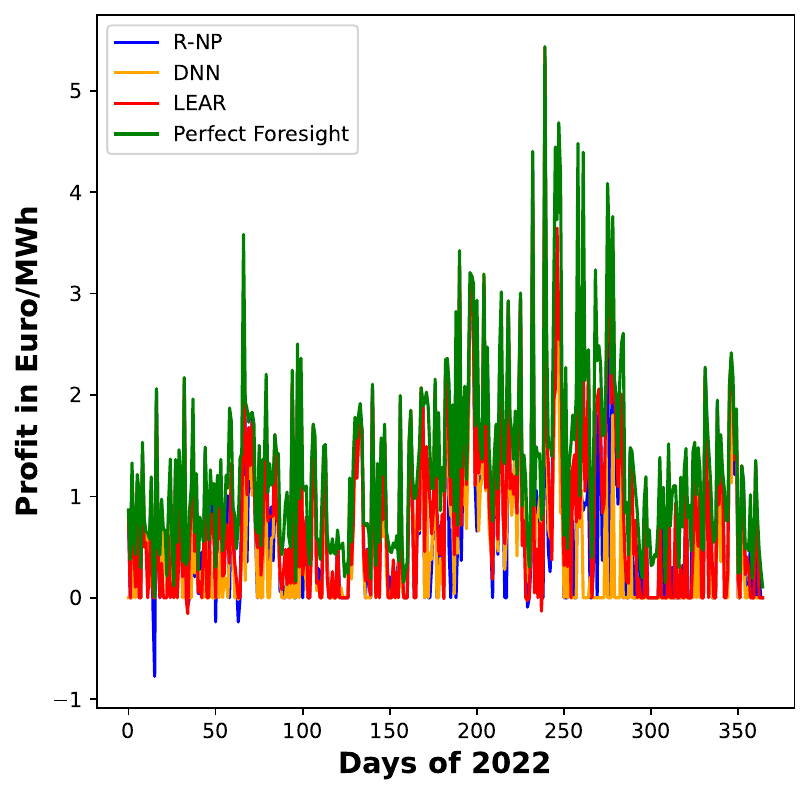}
        \caption{}
        \label{subfig:Profit_2022_case_III}
    \end{subfigure}
    \begin{subfigure}{0.45\linewidth}
        \includegraphics[height = 5.5cm, width = 7cm]{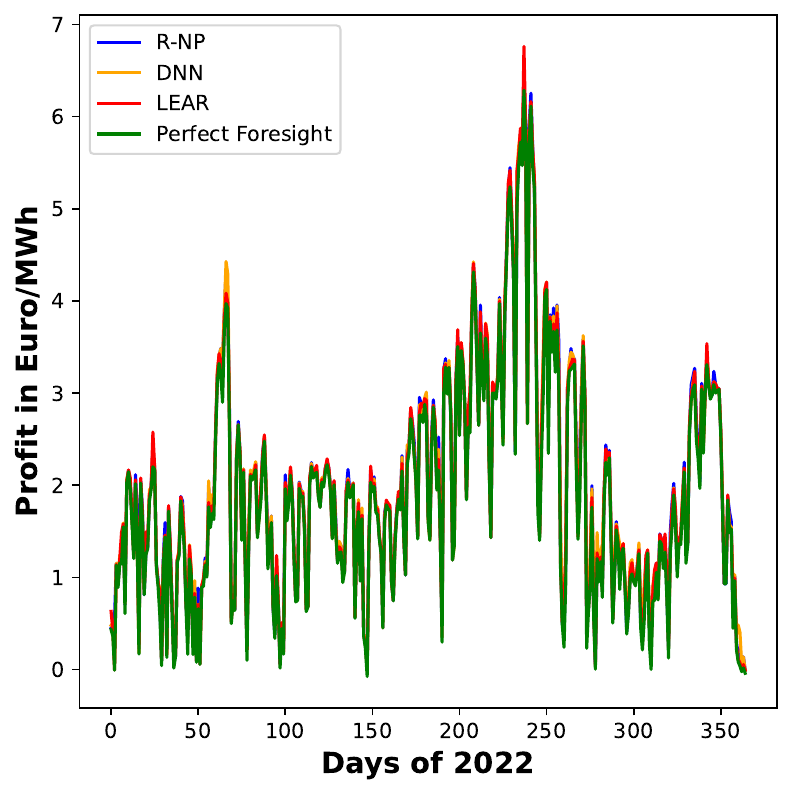}
        \caption{}
        \label{subfig:Profit_2022_case_IV}
    \end{subfigure}
    \caption{\centering{Comparison of daily profit through the predicted price and real price for year 2022 (a) Case I, (b) Case II, (c) Case III and (d) Case IV}}
    \label{fig:Profit_Comparison_2022}
\end{figure}

\begin{figure}[ht!]
    \centering
    \begin{subfigure}{0.45\linewidth}
        \includegraphics[height = 5.5cm, width = 7cm]{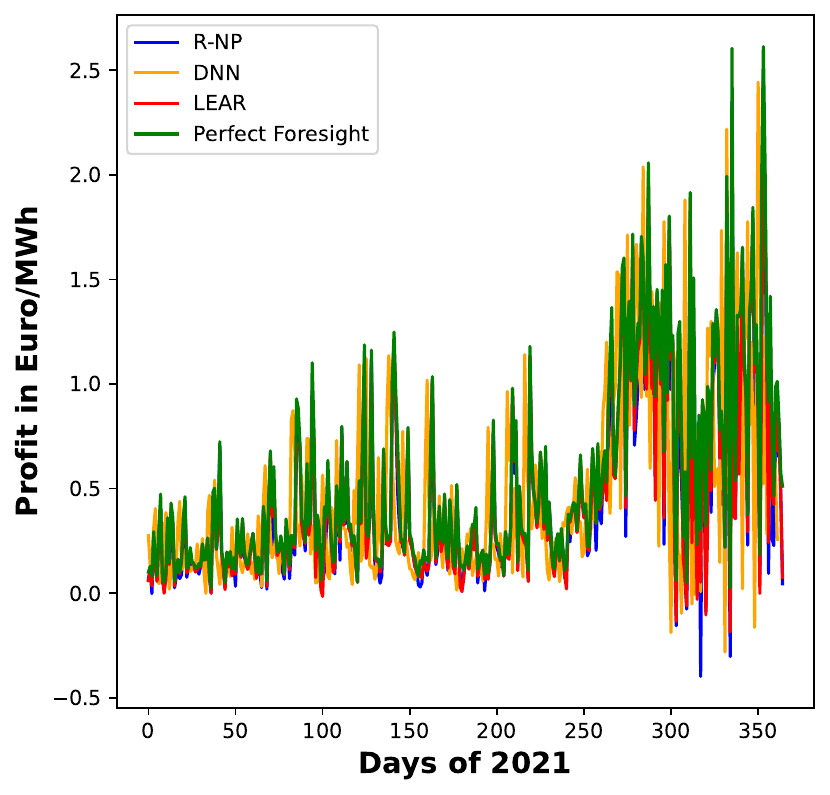}
        \caption{}
        \label{subfig:Profit_2021_case_I}
    \end{subfigure}
    \begin{subfigure}{0.45\linewidth}
        \includegraphics[height = 5.5cm, width = 7cm]{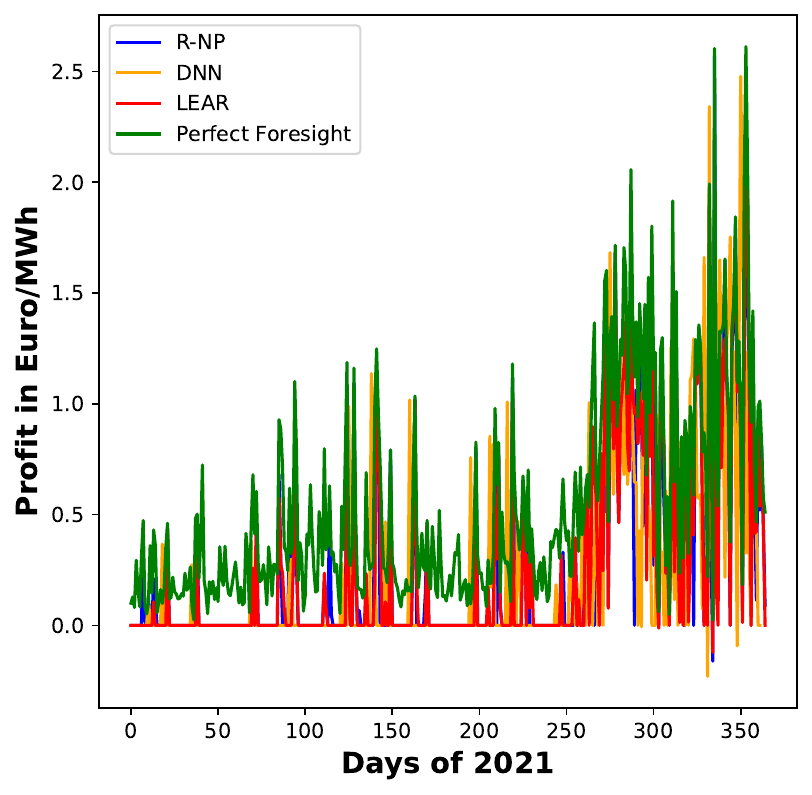}
        \caption{}
        \label{subfig:Profit_2021_case_II}
    \end{subfigure}
    \begin{subfigure}{0.45\linewidth}
        \includegraphics[height = 5.5cm, width = 7cm]{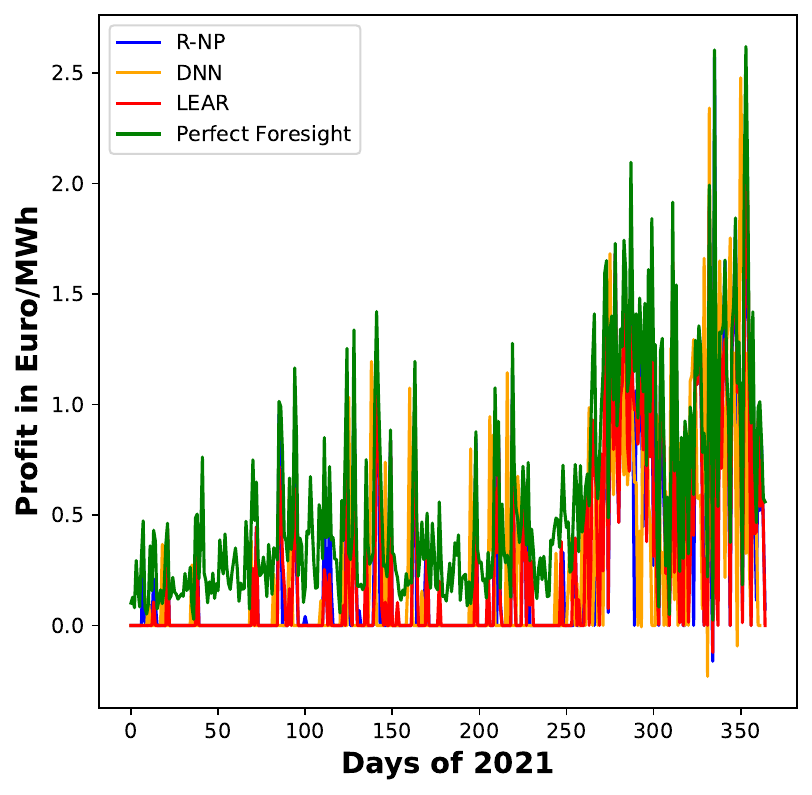}
        \caption{}
        \label{subfig:Profit_2021_case_III}
    \end{subfigure}
    \begin{subfigure}{0.45\linewidth}
        \includegraphics[height = 5.5cm, width = 7cm]{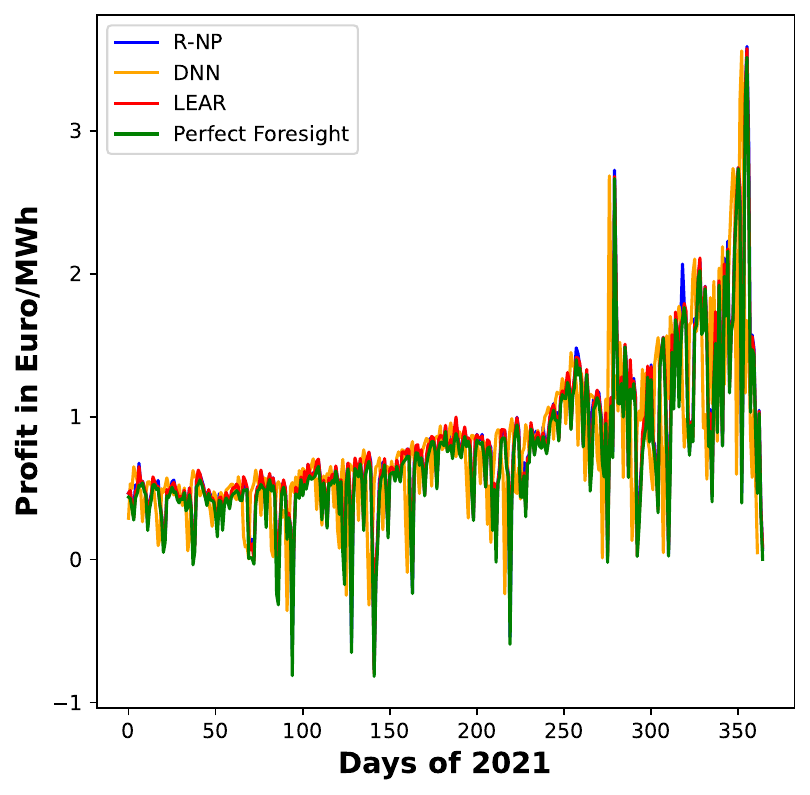}
        \caption{}
        \label{subfig:Profit_2021_case_IV}
    \end{subfigure}
    \caption{\centering{Comparison of daily profit through the predicted price and real price for year 2021 (a) Case I, (b) Case II, (c) Case III and (d) Case IV}}
    \label{fig:Profit_Comparison_2021}
\end{figure}

\section{Benchmark Models}\label{a:benchmark}
The research in \cite{LAGO2021116983} suggest that the LASSO estimated autoregressive (LEAR) model and the  deep neural network (DNN) model are the better performing model when compared to different class and categories of the models for electricity price forecasting. In this study we compare the performance of our hybrid model which predicts via GPR and SVR with both of these benchmark models. We also compare the individual predictions by GPR and SVR with both of these benchmarks. In this section we briefly introduce the construction of LEAR and DNN Models and for detail study of these we request the readers to refer to \cite{LAGO2021116983}:

\subsection{LEAR Model}\label{ass:lear_model}
Let us assume that we have a data set with price of electricity , forecast residual load and forecast total renewable energy production in hourly resolution. With the assumption that the current price of the electricity is dependent on the past prices, forecast residual load and forecast total renewable energy production, the LEAR model to predict the electricity price on $h^{th}$ hour of day $i$, is as follows:
\begin{equation}\label{lear}
    \begin{split}
        P^{(i)}_{h} = & f(P^{(i-1)}, P^{(i-2)}, P^{(i-3)}, P^{(i-7)}, L^{(i)}, L^{(i-1)}, L^{(i-7)}, R^{(i)}, R^{(i-1)}, R^{(i-7)}, \boldsymbol{\theta}_{h})+\epsilon^{(i)}_{h}\\ 
         = & \sum_{j=1}^{24}\theta_{h,j}P^{(i-1)}_{j} + \sum_{j=1}^{24}\theta_{h,24+j}P^{(i-2)}_{j} + \sum_{j=1}^{24}\theta_{h,48+j}P^{(i-3)}_{j}\\
         & + \sum_{j=1}^{24}\theta_{h,72+j}P^{(i-7)}_{j} + \sum_{j=1}^{24}\theta_{h,96+j}L^{(i)}_{j} + \sum_{j=1}^{24}\theta_{h,120+j}L^{(i-1)}_{j} \\
         & + \sum_{j=1}^{24}\theta_{h,144+j}L^{(i-7)}_{j} + \sum_{j=1}^{24}\theta_{h,168+j}R^{(i)}_{j} + \sum_{j=1}^{24}\theta_{h,192+j}R^{(i-1)}_{j}\\
         &  + \sum_{j=1}^{24}\theta_{h,216+j}R^{(i-7)}_{j} + + \sum_{j=1}^{24}\theta_{h,240+j}z^{(i)}_{j} +\epsilon^{(i)}_{h}
    \end{split}
\end{equation}
where $\boldsymbol{\theta}_{h} = [\theta_{h,1},\cdots, \theta_{h,247}]$  are the parameters for LEAR. These parameters are estimated using LASSO as follows:
$$\boldsymbol{\hat{\theta}}_{h} = \operatorname*{argmin}_{\boldsymbol{\theta}_{h}} \sum_{i=8}^{N}(P^{(i)}_{h} - \hat{P}^{(i)}_{h})^{2} + \lambda \sum_{k = 1}^{247}\vert\theta_{h,k}\vert$$
where $\lambda \geq 0$ is the regularization parameter controlling the sparsity of the solution.

\subsection{DNN}\label{ass:dnn}
A deep neural network with 4 hidden layers, trained under a multivariate framework was considered. The parameters $\theta$ are optimized using the Adam algorithm, while hyperparameters and input features were selected via Tree-structured Parzen Estimator (TPE) which is a Bayesian optimization method. This DNN model can also be accessed as python library$^2$.\footnote{$^2$\href{https://epftoolbox.readthedocs.io/en/latest/modules/started.html}{Link to Python Toolbox}}

\end{document}